\documentclass[letterpaper]{article}
\pdfoutput=1

    \PassOptionsToPackage{numbers, compress}{natbib}

\usepackage[final]{neurips_2022}

\usepackage[utf8]{inputenc} 
\usepackage[T1]{fontenc}    
\usepackage{hyperref}       
\usepackage{url}            
\usepackage{booktabs}       
\usepackage{amsfonts}       
\usepackage{nicefrac}       
\usepackage{microtype}      
\usepackage{xcolor}         

\usepackage{wrapfig}
\usepackage{amsmath,amssymb,amsthm}
\usepackage{mathtools}
\usepackage{algorithm}  
\usepackage[noend]{algpseudocode}
\usepackage{subcaption}
\usepackage[capitalize,noabbrev]{cleveref}
\crefformat{footnote}{#2\footnotemark[#1]#3}
\usepackage{multirow}
\usepackage{tablefootnote}
\usepackage{caption}
\theoremstyle{plain}
\newtheorem{theorem}{Theorem}[section]
\newtheorem{proposition}[theorem]{Proposition}
\newtheorem{lemma}[theorem]{Lemma}
\newtheorem{corollary}[theorem]{Corollary}
\theoremstyle{definition}
\newtheorem{definition}[theorem]{Definition}

\theoremstyle{remark}

\newcommand{\paragraphbe}[1]{\noindent{\bf \em #1}. }
\usepackage{minitoc}

\usepackage{float}
\def\compileFigures{0}

\usepackage{tikz}
\if\compileFigures1
\usetikzlibrary{external}
\fi

\usepackage{tikz-cd}
\usetikzlibrary{calc}
\usepackage{pgfplots}
\usetikzlibrary{math}
\usepgfplotslibrary{groupplots}
\pgfplotsset{compat=1.3}
\usepackage{csvsimple}

\newcommand{\D}{D}
\newcommand{\X}{\mathcal{X}} 
\newcommand{\n}{n} 
\newcommand{\q}{q} 
\newcommand{\dimension}{d} 
\newcommand{\eps}{\varepsilon} 
\newcommand{\x}{\mathbf{x}} 
\newcommand{\y}{\mathbf{y}} 
\newcommand{\z}{\mathbf{z}} 

\newcommand{\loss}{\ell}                  
\newcommand{\size}{n}                                 	     

\newcommand{\K}{K}					     
\newcommand{\kk}{k}

\newcommand{\E}{\mathbb{E}}

\newcommand{\Id}{\mathbb{I}_d}

\newcommand{\algo}{\mathcal{A}}

\newcommand{\step}{\eta}

\newcommand{\sig}{\sigma}

\newcommand{\lam}{\lambda}
\newcommand{\be}{\beta}

\newcommand{\Lip}{L}

\newcommand{\Gauss}[2]{\mathcal{N}\left(#1,#2\right)}

\newcommand{\grad}{\nabla}

\newcommand{\vtheta}{\theta}

\newcommand{\B}{B}

\newcommand{\ii}{i}

\newcommand{\g}[2]{g\left(#1;#2\right)}

\makeatletter
\newtheorem*{rep@theorem}{\rep@title}
\newcommand{\newreptheorem}[2]{%
\newenvironment{rep#1}[1]{%
 \def\rep@title{#2 \ref{##1}}%
 \begin{rep@theorem}}%
 {\end{rep@theorem}}}
\makeatother

\newreptheorem{theorem}{Theorem}
\newreptheorem{lemma}{Lemma}
\newreptheorem{proposition}{Proposition}
\newreptheorem{corollary}{Corollary}

\title{Differentially Private Learning Needs Hidden State\\ (Or Much Faster Convergence)}

\author{%
  Jiayuan Ye, Reza Shokri\\
  Department of Computer Science\\
  National University of Singapore\\
  \texttt{\{jiayuan, reza\}@comp.nus.edu.sg} \\
}

\begin{document}

\maketitle

\renewcommand*{\proofname}{Proof Sketch}

\doparttoc 
\faketableofcontents 

\begin{abstract}
    Prior work on differential privacy analysis of randomized SGD algorithms relies on composition theorems, where the implicit (unrealistic) assumption is that the internal state of the iterative algorithm is revealed to the adversary. As a result, the R\'enyi DP bounds derived by such composition-based analyses linearly grow with the number of training epochs. When the internal state of the algorithm is hidden, we prove a converging privacy bound for noisy stochastic gradient descent (on strongly convex smooth loss functions). We show how to take advantage of privacy amplification by sub-sampling and randomized post-processing, and prove the dynamics of privacy bound for ``shuffle and partition'' and ``sample without replacement'' stochastic mini-batch gradient descent schemes. We prove that, in these settings, our privacy bound converges exponentially fast and is substantially smaller than the composition bounds, notably after a few number of training epochs. Thus, unless the DP algorithm converges fast, our privacy analysis shows that hidden state analysis can significantly amplify differential privacy. 
\end{abstract}


\pgfplotsset{
    /pgf/declare function={
        naive_dp_dynamics(\k,\a,\l,\e,\g,\s,\n,\b) = \a*\g^2/(\l*\s^2*\b^2)*(1-exp(-\l*\e*\k/2));
    }
}

\pgfplotsset{
    /pgf/declare function={
        priv_mixing_diffusion(\k,\a,\l,\e,\g,\s,\n,\b) = \a*\e*\g^2/(2 * (\n/\b - 1)*\b^2 * 2 * \s^2)*(1- 2 * \e * 4 * \l / (4 + \l) )^(\n/(2*\b)) * \k;
    }
}

\pgfplotsset{
    /pgf/declare function={
        priv_mixing_diffusion_last_batch(\k,\a,\l,\e,\g,\s,\n,\b) = min(
            \a*\e*\g^2/(2 * (\n/\b - 1)*\b^2 * 2 * \s^2)*(1- 2 * \e * 4 * \l / (4 + \l) )^(\n/(2*\b)) * (\k - 1) + \a*\e*\g^2/(2 *\b^2 * 2 * \s^2),
            2 * \k * \a*\e*\g^2/(2 *\b^2 * 2 * \s^2)
        );
    }
}

\pgfplotsset{
    /pgf/declare function={
        improved_dp_dynamics_first_batch(\k,\a,\l,\e,\g,\s,\n,\b) = max(0, 
        \a*\e*\g^2/(4 * \s^2 * \b^2) * (1-\e*\l)^(2 * (floor(\n/(2*\b)) - 1)) 
        * (1 - (1-\e*\l)^2) / (1- (1-\e*\l)^( 2 * floor(\n/(2*\b)) ) )
        * (1- (1-\e*\l)^(
            2 * (\k - 1) * (\n/\b - floor(\n/(2*\b))) 
            ) 
            ) / ( 1 - (1-\e*\l)^(
                2 * ( \n/\b - floor(\n/(2*\b)) )
                ) 
                )
        + \a*\e*\g^2/(4 * \s^2 * \b^2) * (1-\e*\l)^(2 * (\n/\b - 0 - 1)) 
        * (1 - (1-\e*\l)^2) / (1- (1-\e*\l)^( 2 * (\n/\b - 0) ) )
        )
        ;
    }
}

\pgfplotsset{
    /pgf/declare function={
        improved_dp_dynamics_last_batch(\k,\a,\l,\e,\g,\s,\n,\b) = min( 2 *\k * \a*\e*\g^2/(4 * \s^2 * \b^2), 
        \a*\e*\g^2/(4 * \s^2 * \b^2) * (1-\e*\l)^(2 * (floor(\n/(2*\b)) - 1)) 
        * (1 - (1-\e*\l)^2) / (1- (1-\e*\l)^( 2 * floor(\n/(2*\b)) ) )
        * (1- (1-\e*\l)^(
            2 * (\k - 1) * (\n/\b - floor(\n/(2*\b))) 
            ) 
            ) / ( 1 - (1-\e*\l)^(
                2 * ( \n/\b - floor(\n/(2*\b)) )
                ) 
                )
        + \a*\e*\g^2/(4 * \s^2 * \b^2)
        )
        ;
    }
}


\pgfplotsset{
    /pgf/declare function={
        amp_dp_dynamics_worse(\k,\a,\l,\e,\g,\s,\n,\b) = 
        1/(\a-1) *
        ln(
            1 - \b/\n + \b/\n * 
            exp(
                (\a-1) * \a * \g^2/ (\l * \s^2 * \b^2) 
                * (1 - exp( -\l*\e/2 ))
                )
            ) * (1 - exp(-\l*\e*\k*\n/\b/2)) /(1-exp(-\l*\e/2));
    }
}

\pgfplotsset{
    /pgf/declare function={
        amp_dp_dynamics(\k,\a,\l,\e,\g,\s,\n,\b) = 
        1/(\a-1) *
        ln(
            1 - \b/\n + \b/\n * 
            exp(
                (\a-1) * \a * \g^2/ (\l * \s^2 * \b^2) 
                * (1 - exp( -\l*\e/2 ))
                )
            ) * (1 - exp(-\l*\e*\k*\n/\b/2)) /(1-exp(-\l*\e/2));
    }
}

\pgfplotsset{
    /pgf/declare function={
        amp_dp_dynamics_avg_batch(\k,\a,\l,\e,\g,\s,\n,\b) = \a*\g^2/(\l*\s^2*\b^2)
        *(1-exp(-\l*\e*\n/\b))
        /(1 - exp(- \l*\e*(\n/\b-1/2))) 
        *  (1-exp{-\k*\l*\e* (\n/\b - 1/2)})
        * \b/\n * 1/(1+exp(-\l*\e/2));
    }
}

\pgfplotsset{
    /pgf/declare function={
        amp_dp_dynamics_shuffle(\k,\a,\l,\e,\g,\s,\n,\b) = \a*\e*\g^2/(4*\s^2*\b^2)
        *(1-exp(-\l*\e*\k*(\n/\b-1)))
        /(1 - exp(- \l*\e*(\n/\b-1)))
        ;
    }
}

\section{Introduction}
\label{sec:intro}

Machine learning models leak sensitive information about their training data~\cite{ye2022enhanced, carlini2021extracting}. To protect user privacy, the widely-used differentially private training algorithm, DP-SGD~\cite{abadi2016deep}, adds carefully calibrated noise in each step of updating model parameters. This randomness guarantees that the models trained on any two neighboring datasets are indistinguishable in probability distributions. To quantify this indistinguishability, the DP analysis \textit{bounds} the (moment of) likelihood ratio between a pair of models trained on any two neighboring datasets. Improving this privacy analysis is crucial for obtaining a higher (train and test) accuracy of the output model, under a constrained privacy budget.

The main-stream analysis of privacy loss in DP-SGD is based on composition theorems, which quantify the total privacy loss of the training process across all its iterations.  Given that privacy-preserving learning via DP-SGD usually suffers from a slow convergence in empirical prediction accuracy~\cite{bu2021convergence}, the final bound could be significantly large (and loose). This analysis worsens the privacy-accuracy trade-off by overestimating the magnitude of required noise for DP training.

To alleviate this problem, one popular direction is to design new variants of the DP-SGD algorithm, that converge faster, i.e. require a smaller number of iterations to reach a stable training accuracy. This includes works that derive privacy preserving variants of fast optimization algorithms, such as performing DP-SGD with momentum~\cite{kairouz2021practical}, adaptive gradient clipping~\cite{andrew2021differentially, pichapati2019adaclip}, adaptive selection of step-size and noise scale~\cite{asi2021private}, and using pre-trained or hand-crafted features~\cite{tramer2020differentially}.  All these approaches aim for a \textit{faster training convergence} for the DP learning algorithm, such that the composition analysis is applied to a smaller number of iterations and the total bound remains small. 

Another line of work focuses on directly improving the privacy analysis of the DP-SGD algorithm. To this end, prior work~\cite{feldman2018privacy, balle2019privacy, chourasia2021differential} suggest that hiding the internal state of noisy (S)GD, therefore analyzing the privacy bound for releasing only the last iterate results in a more accurate estimation of privacy loss. By applying the last-iterate privacy analysis~\cite{feldman2018privacy}, \citet{feldman2020private} design new differentially private algorithms that achieve theoretically optimal excess risk with linear runtime for convex optimization, which is better than the prior privacy utility trade-off for the noisy SGD algorithm~\cite{bassily2014private}.  However, the analysis applies only to one \textit{single} epoch of training, thus it is unclear how such training performs in practical learning settings (which require multiple epochs of training). \citet{chourasia2021differential} derive a strong privacy guarantee for hidden-state noisy GD over many training epochs, under strongly convex loss function. However, the analysis is limited to computationally expensive GD~\cite{geiping2021stochastic}, and extension of it to stochastic mini-batch training would fail to model the privacy amplification by post-processing~\cite{feldman2018privacy, balle2019privacy} and mini-batch sub-sampling~\cite{chaudhuri2006random, li2012sampling, bassily2014private, abadi2016deep, balle2019privacy}. 

Thus, under the hidden-state assumption over many epochs of training, the challenge is to compute a small privacy bound for differentially private stochastic gradient descent algorithms by taking advantage of privacy amplifications due to sub-sampling and randomized updates over mini-batches.

\setlength{\intextsep}{1pt}
\begin{wrapfigure}{r}{0.45\textwidth}
    \resizebox{0.45\textwidth}{!}{
        \begin{tikzpicture}
            \begin{axis}[
            no markers,
            samples=50,
            xmin=0,
            ymin=0,
            axis line style = thick,
            ymax=0.15,
            axis lines = left,
            ytick={0.0,0.04,0.08,0.12}, yticklabels={0,0.05,0.10,0.15},
            xlabel={Training Epochs},
            ylabel={$\eps$ in $(\alpha=15, \eps)$-R\'enyi Differential Privacy},
            xmax=42,clip = true,
            clip mode=individual,axis y line*=left,axis x line*=bottom,
            legend style={at={(0.05,1.1)},anchor=west}]
                \addplot[thick,purple] table [x=k, y=eps, col sep=comma] {figures/data_shuffle_sample/k=40_a=15_l=1_e=0.02_g=4_s=2_n=50_b=2.csv};
                \addplot[thick,purple,dotted,domain=0:40] {improved_dp_dynamics_last_batch(x,15,1,0.02,4,2,50,2)};
                \addplot[thin,purple,dashed] table [x=k, y=eps, col sep=comma] {figures/data_sgm_sample/k=40_a=15_l=1_e=0.02_g=4_s=2_n=50_b=2.csv};
                \node[anchor=west] at (axis cs: 5,{improved_dp_dynamics_last_batch(40,15,1,0.02,4,2,50,2) + 0.015}) {Our Privacy Bound};
                \node[anchor=west] at (axis cs: 1,{improved_dp_dynamics_last_batch(40,15,1,0.02,4,2,50,2) + 0.005}) {(fixed-order last mini-batch)};
                \node[anchor=west] at (axis cs: 25,{0.04}) {Our Privacy Bound};
                \node[anchor=west] at (axis cs: 23,{0.03}) {(shuffle and partition)};
                \node[anchor=west] at (axis cs: 20,{0.135}) {DP-SGD~\cite{abadi2016deep} + SGM~\cite{mironov2019r}};
                \node[anchor=west] at (axis cs: 28,{0.125}) {(composition)};
            \end{axis}
        \end{tikzpicture}
    }
    \vspace*{-2mm}
    \label{fig:comp_bound}
\end{wrapfigure}
\paragraphbe{Contributions}
In this paper, we tackle this challenge, and bridge the gaps in the prior work on hidden-state (last-iterate) privacy analysis. We model the privacy dynamics of noisy SGD, and show the privacy amplification for private (hidden state) machine learning due to stochastic mini-batch selection.  \textbf{(i)} As we also show in this figure, for multi-epoch noisy \emph{stochastic mini-batch} gradient descent under \textit{``shuffle and partition''} and \textit{iterative ``sampling without replacement''}, we prove new converging privacy bounds (for strongly convex smooth loss function) that significantly improves over the prior bounds~\cite{balle2019privacy, mironov2019r, chourasia2021differential}. Our proof relies on our new bounds for the privacy amplification by post-processing (\cref{lem:recursive}) and sub-sampling (\cref{thm:shuffle,thm:amp_samp_wo_replacement}). \textbf{(ii)} For the special case of full gradient descent, our new approach (via proving better bounds for the privacy amplification by randomized post-processing) results in a strictly tighter bound than the prior work~\cite{chourasia2021differential} with a different proof (\cref{append:simple_proof_gd}). The key insight of our new proof is that we can break down one noisy GD update into two consecutive steps: a noisy GD update with smaller noise scale, followed by pure additive Gaussian noise (randomized post-processing).

Our results show that, to obtain a tighter privacy bound, it is crucial to apply our hidden-state privacy dynamics analysis for learning tasks with slow to moderate convergence. Alternatively, it is not so costly to use composition-based privacy analysis when the training process converges very quickly.

\section{An overview of the problem and our approach}
\label{sec:overview}

We analyze differential privacy loss of the noisy (stochastic) mini-batch gradient descent algorithm, when only the last iterate parameters~$\theta_\K^0$ are visible~(see \cref{alg:noisymBGD}).\footnote{See \cref{sec:prelim} for the preliminaries about differential privacy and necessary tools for our analysis.} We consider two mini-batch generation variants: \textbf{(1)} ``shuffle and partition'' (widely implemented in privacy libraries~\cite{yousefpour2021opacus, tensorflowprivacy}); \textbf{(2)} ``sample without replacement'' (analyzed extensively in prior results~\cite{chaudhuri2006random, li2012sampling, balle2019privacy, wang2019subsampled} despite its computational cost). 
Our ultimate goal is to prove a \emph{worst-case} upper bound of the R\'enyi divergence between distributions of \emph{last-iterate} parameters $\theta_\K^0$ and ${\theta'}_\K^0$ trained on \emph{any} two neighboring datasets. 

The prior work~\citet{chourasia2021differential} has shown a tight converging hidden-state privacy dynamics analysis for full-batch noisy gradient descent. However, the sub-sampling steps that are unique to the stochastic mini-batch gradient descent, and their privacy benefits, are not modeled nor quantified in such bounds and the follow-up work. Would a naive extension of the prior GD analysis to the stochastic mini-batch setting result in a tight bound? To this end, we can view the updates that \emph{involve} the (sensitive) differing data record (between neighboring datasets) as gradient descent on a \emph{smaller}, iteratively \emph{changing} dataset of size~$b$ (i.e., the size of a mini-batch). This extension results in the following naive privacy baseline.

\begin{theorem}[Naive Extension of~\citet{chourasia2021differential} bound to SGD]
    \label{thm:naive_RDP_baseline}
    If the loss function $\ell(\theta;\x)$ is $\lambda$-strongly convex and $\beta$-smooth, and its gradient has $\ell_2$-sensitivity $S_g$, then \cref{alg:noisymBGD} under "shuffle and partition" and step-size $\eta<\frac{2}{\lambda + \beta}$ satisfies $(\alpha, \eps)$-R\'enyi DP with $\eps \leq \frac{\alpha S_g^2}{\lambda\sig^2 b^2}(1-e^{-\lam \eta \K/2})$.
\end{theorem}
However, as we also show in~\cref{fig:improved_dynamics}, this naive privacy dynamics baseline \textit{slowly} converges to a \textit{huge} constant, which is significantly worse than the bounds derived by composition~\cite{mironov2019r, abadi2016deep}.~\footnote{A recent concurrent work~\cite{ryffel2022differential} also follows this approach to extend the GD analysis~\cite{chourasia2021differential} to SGLD setting, and proves a similar bound to our naive privacy dynamics baseline~\cref{thm:naive_RDP_baseline}. However, there is a slight difference between \cite[Corollary 3.3]{ryffel2022differential} and our \cref{thm:naive_RDP_baseline}, due to the flawed assumption in~\cite[Lemma 3.4]{ryffel2022differential} that the LSI constant proved in~\citet{chourasia2021differential} (that \emph{only} holds for GD process) also holds for SGLD process (that takes the form of a more complicated mixture distribution). See more details in~\cref{app:privacy_dynamics_baseline}. }  This is because this naive baseline \textit{fails to} capture the privacy amplification due to the stochasticity of mini-batch sub-sampling, and the iterative noisy updates which amplify the privacy of preceding mini-batches (referred to as amplification by post-processing). In this paper, we show how to compute a much tighter differential privacy bound for noisy SGD, under the hidden-state assumption, while taking advantage of these privacy amplifications. Our methodology is novel and quantifies the hidden-state privacy amplification due to iterative data (re)sampling throughout the training process (over multiple epochs of training).  This solves the limitation of the prior work on privacy amplification (by iteration) which focuses on a single epoch, and uses composition theorems across epochs (thus, not modeling the privacy of hidden-state iterative resampling). Finally, our methodology is generic and applies to the special case of full-batch noisy gradient descent, which enables a strictly tighter privacy bound than the prior work~\cite{chourasia2021differential} under the same assumptions (\cref{append:simple_proof_gd}).

\begin{algorithm}[t!]
    \small
	\caption{$\algo_{\text{Noisy-mBGD}}$: Noisy (Stochastic) mini-batch Gradient Descent}
	\label{alg:noisymBGD}
	\begin{algorithmic}[1]
		\State {\bfseries Input:} Dataset $\D=(\x_1, \x_2, \cdots, \x_\size)$. Parameter space $\theta\in\mathbb{R}^\dimension$. Loss function $\loss(\vtheta;\x)$. Stepsize~$\step$. Noise standard deviation $\sig$. mini-batch size $b$. Initial parameters $\theta_0^0$ sampled from an arbitrary distribution $p_0(\theta)$. 
		\State{\bfseries Batch Generation:} obtain mini-batches $B^j_k$ of size $b$, for $j=0,\cdots, n/b-1$; for epochs $k=0,\cdots,\K-1$. \\
        \Comment{\textbf{If ``Shuffle and partition''}: randomly partition $n$ data indices into $B^0,\cdots, B^{n/b-1}$, let $B^j_k=B^j$. \footnotemark}
        \label{step:batch_generation_shuffle}\\
        \Comment{\textbf{If Sample without replacement}: resample $b$ different indices from $\{1,\cdots,n\}$ to obtain every $B_k^j$.\ }
        \label{step:batch_generation_samp_wo}
		\For { epoch $\kk = 0, 1, \cdots, \K - 1$}
		\For { iteration $j = 0, 1, \cdots, n/b - 1$}
		\State {$\vtheta_{\kk}^{j + 1} = \vtheta_{\kk}^{j} - \step\cdot \g{\vtheta_{\kk}^{j}}{\B_k^j} + \sqrt{2\step\sig^2}\cdot \Gauss{0}{\Id} \quad \text{where}\quad \g{\vtheta_{\kk}^{j}}{\B_k^{j}}= \frac{1}{b}\sum_{\ii\in \B_k^{j}} \grad\loss(\vtheta_{\kk}^{j};\x_\ii)$\label{step:mbgd}}
		\EndFor
		\State {$\vtheta_{\kk + 1}^{0} = \vtheta_{\kk}^{n/b}$}
		\EndFor
		\State {{\bfseries Output:} $\vtheta_{\K}^{0}$}
	\end{algorithmic}
\end{algorithm}

\footnotetext{For simplicity of presentation, we assume $b$ divides $n$. If $n/b$ is not an integer, and if the algorithm ignores the last $n - \lfloor n/b\rfloor \cdot b$ data points, then our privacy dynamics bound holds by replacing $n/b$ with $\lfloor n/b\rfloor$.}

\underline{Our approach is as follows.} We decompose the distribution of last iterate parameters $\theta_\K^0$ in \cref{alg:noisymBGD} as a mixture of conditional distributions $p(\theta_\K^0)=\sum_{B}p(B)\cdot p(\theta_\K^0|B)$, given any possible mini-batch sequence~$B$. We first analyze the R\'enyi privacy loss for a \textit{fixed mini-batch sequence} while modeling privacy amplification by randomized post-processing (i.e., how much the randomized gradient update improve the privacy due to updates on preceding mini-batches). We then quantify the privacy amplification by subsampling under \textit{stochastic mini-batches}.

\paragraphbe{Privacy amplification by randomized post-processing} To start, we consider Algorithm~\ref{alg:noisymBGD} without the effect of stochasticity in the mini-batch sampling process. That is, we first assume that the mini-batch sequence used in the algorithm is fixed (by an arbitrary order). 
When a mini-batch contains only the indices of shared data records between two neighboring datasets, then the deterministic mini-batch gradient descent mapping does not cause any additional privacy loss. Addition of the Gaussian noise that follows this deterministic update, however, serves as a randomized post-processing which \textit{decreases} the R\'enyi divergence between the two processes (associated to neighboring datasets). To compute this privacy amplification, we precisely model the change of parameter distributions with the Fokker-Planck equation (for diffusion process with zero drift).

\paragraphbe{Hidden-state privacy amplification by sub-sampling} 
Although privacy amplification by a \textit{single} sub-sampling operation is well-studied~\cite{chaudhuri2006random, li2012sampling, bassily2014private, abadi2016deep, balle2019privacy,wang2019subsampled,mironov2019r,feldman2022hiding,feldman2022stronger}, to the best of our knowledge, no prior bounds are applicable to multiple epochs of \emph{hidden-state} iterative resampling. This is because under the hidden-state assumption, the number of mixing components in the last-iterate parameter distribution grows exponentially with the number of epochs $K$, i.e., [number of possible values for one mini-batch]$^{K\cdot n/b}$, which makes the mixture distribution very difficult to analyze. In this paper, to study the exponentially many mixture components, we derive recursions for the divergence between mixture distributions after one epoch. We use the joint convexity of exponentiated R\'enyi divergence to bound how much smaller the R\'enyi divergence between mixture distributions (for model parameters at one epoch) is compare to the worst case R\'enyi divergence across any pair of their mixture components (representing the preceding epoch). This recursion quantifies the hidden-state privacy amplification by sub-sampling, and enables a significantly smaller R\'enyi DP bound for noisy stochastic mini-batch gradient descent, than the composition of sub-sampled mechanisms over multiple epochs (Figure~\ref{fig:comp_shuffle}).

\section{Privacy dynamics for fixed-ordering noisy mini-batch gradient descent}

\label{sec:fix_sgd}

In this section, we quantify the privacy amplification by randomized post-processing in~\cref{alg:noisymBGD}, during iterations that do not access the sensitive differing data between neighboring datasets. We then combine it with RDP composition for the remaining iterations and prove a converging privacy~bound.

\subsection{Privacy amplification by randomized post-processing (additive Gaussian noise)}

\label{ssec:amp_post_proc}

We first explain the key lemma that proves exponentially decaying R\'enyi privacy loss under additive Gaussian noise post-processing, when the parameter distributions satisfy the log-Sobolev inequality. This is a new bound for the well-known "privacy amplification by iteration" phenomenon~\cite{feldman2018privacy,balle2019privacy}. 
\begin{lemma}
    \label{lem:post_proc}
    Let $\mu, \nu$ be two distributions on $\mathbb{R}^d$. Let $f:\mathbb{R}^d\rightarrow\mathbb{R}^d$ be a measurable mapping on $\mathbb{R}^d$. We denote $\mathcal{N}(0,2t\sig^2\cdot\mathbb{I}_d)$ to be the standard Gaussian distribution on $\mathbb{R}^d$ with covariance matrix~$2t\sigma^2\cdot \mathbb{I}_d$. We denote $p_t(\theta)$ and $p_t'(\theta)$ to be the probability density functions for the distributions $f_{\#}(\mu) * \mathcal{N}(0,2t\sig^2\mathbb{I}_d)$ and $f_{\#}(\nu) * \mathcal{N}(0,2t\sig^2\mathbb{I}_d)$ respectively, where $f_\#(\mu), f_\#(\nu)$ denote the push forward distributions of $\mu, \nu$ under mapping $f$. Then if $\mu$ and $\nu$ satisfy log-Sobolev inequality with constant $c$, and if the mapping $f$ is $L$-Lipschitz, then for any order $\alpha> 1$,
    \begin{align}
        \label{eqn:post_proc}
       \frac{\partial}{\partial t}R_\alpha \left( p_t(\theta)\lVert p_t'(\theta)\right) & \leq - c_t\cdot 2\sig^2\cdot \left(\frac{R_\alpha(p_t(\theta)\lVert p_t'(\theta))}{\alpha} + (\alpha-1)\cdot \frac{\partial}{\partial\alpha} R_\alpha(p_t(\theta)\lVert p_t'(\theta)) \right),
    \end{align} 
    where $c_t = \left(\frac{L^2}{c} + 2t\sig^2\right)^{-1}$ is the log-Sobolev inequality constant for distributions $p_t(\theta)$ and $p_t'(\theta)$.
\end{lemma} 
\textit{Proof Sketch.}
The proof starts by modelling Gaussian noise as a diffusion process with \textit{zero} drift. We then bound the rate of R\'enyi divergence with the LSI constant for the process (following prior works~\cite{vempala2019rapid,chourasia2021differential}). However, instead of assuming a fixed LSI constant $c$ for all intermediate distributions, we prove a more precise LSI constant $c_t$ that changes with $t$. Complete proof is in \cref{appsub:proof_post_proc}.~\qedsymbol{}

The partial differential inequality in \cref{lem:post_proc} quantifies the amplification of R\'enyi privacy loss under additive Gaussian noise. By solving \cref{eqn:post_proc} on $t\in[0,\eta]$, we prove the following lemma that models recursive privacy dynamics during one step of noisy mini-batch gradient descent.
\begin{lemma}
    \label{lem:recursive}
    Let $D,D'$ be an arbitrary pair of neighboring datasets that differ in the $i_0$-th data point (i.e. $x_{i_0}\neq x'_{i_0}$). Let $B_k^j$ be a fixed mini-batch used (in iteration $j$ of epoch $k$) in~\cref{alg:noisymBGD}, which contains $b$ indices sampled from $\{1,\cdots,n\}$. We denote $\theta_\kk^j$ and ${\theta'}_\kk^{j}$ as the intermediate parameters in \cref{alg:noisymBGD} on input datasets $D$ and $D'$, respectively. If the distributions of $\theta_k^j$ and ${\theta'}_k^j$ satisfy log-Sobolev inequality with a constant $c$, and if the mini-batch GD mapping $f(\theta) = \theta - \eta \cdot \frac{1}{b}\cdot \sum_{i\in B_k^j}\ell(\theta;\x_i)$ is $L$-Lipschitz for parameters $\theta$, then the following recursive bound for R\'enyi divergence holds.
    \begin{align}
        \label{eqn:recursive}
        \frac{R_{\alpha}(\theta_\kk^{j+1}\lVert {\theta'}_\kk^{j+1})}{\alpha} & \leq
        \begin{cases}
            \frac{R_{\alpha'}(\theta_\kk^{j}\lVert {\theta'}_\kk^{j})}{\alpha'}\cdot  \left(1 + \frac{ c \cdot 2 \eta  \sig^2}{L^2}\right)^{-1} & \text{if $i_0\notin B_k^j$}\\
            \frac{R_{\alpha}(\theta_\kk^{j}\lVert {\theta'}_\kk^{j})}{\alpha} + \frac{\eta S_g^2}{4\sig^2b^2}& \text{if $i_0\in B_k^j$}
        \end{cases}\text{ with }\alpha' = \frac{\alpha-1}{1 + \frac{c\cdot 2 \eta  \sig^2}{L^2}} + 1.
    \end{align}
\end{lemma}
\textit{Proof Sketch.}
When $i_0\notin B_k^j$, we apply \cref{lem:post_proc} on the deterministic mini-batch gradient descent update $f$, and solve \cref{eqn:post_proc} to obtain recursive R\'enyi privacy bound.
This proof is similar to that of \cite[Theorem 2]{chourasia2021differential}, however, dealing with the new LSI constant $c_t$ that changes with $t$ introduces an additional technical difficulty. 
When $i_0\in B_k^j$, we use composition theorem to prove the additive recursion in the second row of \cref{eqn:recursive}. The complete proof is in \cref{appsub:proof_recursive}.~\qedsymbol{}

\paragraphbe{Comparison with prior amplification by post-processing bound} Prior to this work, for convex and smooth loss functions, \citet{feldman2018privacy} derive a tight bound for the amplification by post-processing in noisy stochastic gradient descent. In~\cref{appsub:proof_convexsmooth}, we show that our new recursive amplification bound could recover this known \emph{tight} bound for convex smooth loss functions, while using a different proof. For (a more restrictive setting of) \emph{one} epoch of noisy SGD on \emph{strongly convex} loss function, \citet{balle2019privacy} further improve the bound in~\citet{feldman2018privacy} via a careful coupling-based approach. However, we do not see an easy way to extend the bound in~\citet{balle2019privacy} to \emph{multiple} epochs except by using R\'enyi DP composition (which would give a linearly worsening R\'enyi DP bound with the number of epochs). By contrast, our recursive amplification bound easily applies to multiple epochs (under hidden state assumption, and without requiring composition over the epochs), and enables converging privacy dynamics for~\cref{alg:noisymBGD} on strongly convex smooth loss functions (as discussed in \cref{ssec:improve_priv_dynamics}). 

\subsection{Improved privacy dynamics for fixed-ordering noisy mini-batch gradient descent}
\label{ssec:improve_priv_dynamics}

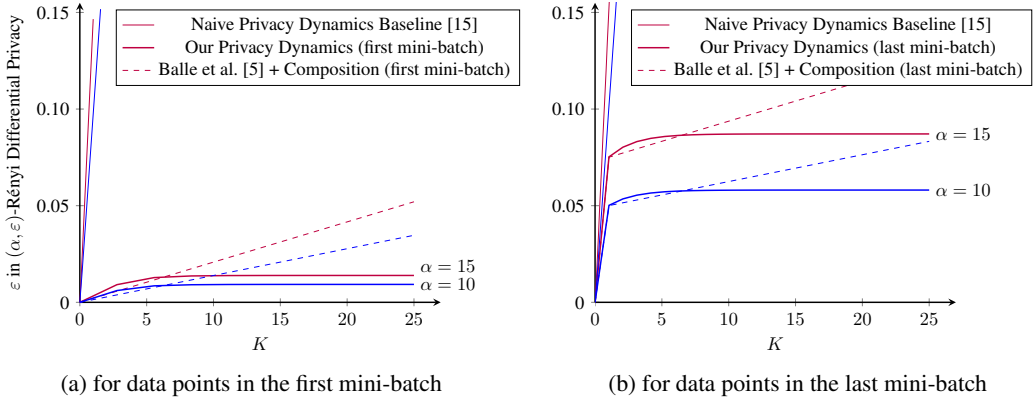
\begin{figure}[t!]
    \centering 
    \begin{subfigure}[b]{0.48\textwidth}
        \scalebox{0.7}
        {\begin{tikzpicture}
            \begin{axis}[
            no markers,
            samples=10,
            xmin=0,
            ymin=0,
            axis line style = thick,
            ymax=0.155,
            axis lines = left,
            ytick={0.0,0.05,0.10,0.15}, yticklabels={0,0.05,0.10,0.15},
            xlabel=$K$,
            ylabel={$\eps$ in $(\alpha, \eps)$-R\'enyi Differential Privacy},
            xmax=27,
            clip = true,
            clip mode=individual,
            axis y line*=left,axis x line*=bottom,
            legend style={at={(0.1,0.85)},anchor=west}]
                \addplot[thin,purple,domain=0:1] {naive_dp_dynamics(x,15,1,0.02,4,2,50,2)};
                \addlegendentry{Naive Privacy Dynamics Baseline~\cite{chourasia2021differential}}
                %
                %
                \addplot[thick,purple,domain=0:25] {improved_dp_dynamics_first_batch(x,15,1,0.02,4,2,50,2)};
                \addlegendentry{Our Privacy Dynamics (first mini-batch)}
                %
                \addplot[thin,purple,dashed,domain=0:25] {priv_mixing_diffusion(x,15,1,0.02,4,2,50,2)};
                \addlegendentry{\citet{balle2019privacy} + Composition (first mini-batch)}
                \addplot[thin,blue,domain=0:1.55] {naive_dp_dynamics(x,10,1,0.02,4,2,50,2)};
                \addplot[thick,blue,domain=0:25] {improved_dp_dynamics_first_batch(x,10,1,0.02,4,2,50,2)};
                %
                %
                \addplot[thin,blue,dashed,domain=0:25] {priv_mixing_diffusion(x,10,1,0.02,4,2,50,2)};
                %
                %
                \node[anchor=west] at (axis cs: 25,{improved_dp_dynamics_first_batch(25,20,1,0.02,4,2,50,2)}) {$\alpha=15$};
                \node[anchor=west] at (axis cs: 25,{improved_dp_dynamics_first_batch(25,10,1,0.02,4,2,50,2)}) {$\alpha=10$};
            \end{axis}
        \end{tikzpicture}
        }
    \caption{for data points in the first mini-batch}
    \end{subfigure}%
    \hfill
    \begin{subfigure}[b]{0.48\textwidth}
        \scalebox{0.7}
        {\begin{tikzpicture}
            \begin{axis}[
            no markers,
            samples=25,
            xmin=0,
            ymin=0,
            axis line style = thick,
            ymax=0.155,
            axis lines = left,
            ytick={0.0,0.05,0.10,0.15}, yticklabels={0,0.05,0.10,0.15},
            xlabel=$K$,
            xmax=27,
            clip = true,
            clip mode=individual,
            axis y line*=left,axis x line*=bottom,
            legend style={at={(0.1,0.85)},anchor=west}]
                \addplot[thin,purple,domain=0:2] {naive_dp_dynamics(x,15,1,0.02,4,2,50,2)};
                \addlegendentry{Naive Privacy Dynamics Baseline~\cite{chourasia2021differential}}
                %
                \addplot[thick,purple,domain=0:25] {improved_dp_dynamics_last_batch(x,15,1,0.02,4,2,50,2)};
                %
                \addlegendentry{Our Privacy Dynamics (last mini-batch)}
                \addplot[thin,purple,dashed,domain=0:25] {priv_mixing_diffusion_last_batch(x,15,1,0.02,4,2,50,2)};
                \addlegendentry{\citet{balle2019privacy} + Composition (last mini-batch)}
                \addplot[thin,blue,domain=0:4] {naive_dp_dynamics(x,10,1,0.02,4,2,50,2)};
                \addplot[thick,blue,domain=0:25] {improved_dp_dynamics_last_batch(x,10,1,0.02,4,2,50,2)};
                %
                %
                \addplot[thin,blue,dashed,domain=0:25] {priv_mixing_diffusion_last_batch(x,10,1,0.02,4,2,50,2)};
                %
                %
                \node[anchor=west] at (axis cs: 25,{improved_dp_dynamics_last_batch(25,15,1,0.02,4,2,50,2)}) {$\alpha=15$};
                \node[anchor=west] at (axis cs: 25,{improved_dp_dynamics_last_batch(25,10,1,0.02,4,2,50,2)}) {$\alpha=10$};
            \end{axis}
        \end{tikzpicture}
        }
    \caption{for data points in the last mini-batch}
    \end{subfigure}%
    \caption{R\'enyi privacy loss of fixed-ordering noisy mini-batch gradient descent over $K$ epochs, which repeatedly goes over a fixed sequence of mini-batch partition in each epoch. We show the privacy loss for data points in the first mini-batch $B^0$ and the last mini-batch $B^{n/b-1}$ of each pass. We show the $\eps$ in the $(\q,\eps)$-RDP guarantee derived by our privacy dynamics analysis (bold lines), the naive privacy dynamics baseline (thin lines), and the privacy amplification by mixing and diffusion analysis~\cite{balle2019privacy} combined with composition (dashed line).
    We evaluate under the following setting: RDP order ${\alpha\in\{10,15\}}$;\ \  $\lam$-strongly convex loss function with $\lam =1 $; $\be$-smooth loss function with $\be=4$; gradient sensitivity $S_g=4$; size of the data set $\size=50$; step-size $\step=0.02$; noise variance $\sig^2=4$, mini-batch size $b=2$. We use~\cref{thm:strconvexsmooth} for our privacy dynamics; ~\cref{thm:naive_RDP_baseline} for naive privacy dynamics baseline; and~\cite[Theorem 5]{balle2019privacy} for~\citet{balle2019privacy} + Composition (details are in~\cref{append:figure1}). }
    \label{fig:improved_dynamics}
\end{figure}

By using the recursive privacy bound~\cref{lem:recursive} and the non-overlapping property of mini-batch partitions of the dataset, we prove position-dependent R\'enyi DP bounds for \cref{alg:noisymBGD} as follows.

\begin{theorem}[Privacy dynamics under strongly convex smooth loss]
    \label{thm:strconvexsmooth}
    Conditioned on a fixed sequence of partitioned mini-batches $B^0,\cdots,B^{n/b-1}$ in \cref{step:batch_generation_shuffle}, 
    if the loss function is $\lambda$-strongly convex, ${\beta\text{-smooth}}$ and its gradient has $\ell_2$-sensitivity $S_g$, then running \cref{alg:noisymBGD} for $K\geq 1$ epochs with step-size $\eta <\frac{2}{\lambda + \beta}$, satisfies $(\alpha,\eps)$-R\'enyi DP for data points in the batch $B^{j_0}$, with
    \begin{align}
        \label{eqn:strconvexsmooth}
        \eps\leq
        \eps_0^{\lfloor \frac{n}{2b}\rfloor}(\alpha) \cdot \frac{1 - (1-\eta\lambda)^{2\cdot (\K - 1) \cdot (n/b - \lfloor \frac{n}{2b}\rfloor)}}{1 - (1-\eta\lambda)^{2\cdot (n/b - \lfloor \frac{n}{2b}\rfloor)} } + \eps_{0}^{n/b - j_0}(\alpha)
    \end{align}
    where $\eps_0^{j}(\alpha)= \frac{\alpha \eta S_g^2}{4\sig^2b^2}\cdot \left(1 - \eta \lambda \right)^{2\cdot (j - 1)} \cdot \frac{1}{\sum_{s=0}^{ j - 1 } (1-\eta\lambda)^{2s}}$ for any $j = 1, \cdots, \frac{n}{b}$ (we assume $\frac{n}{b}\geq 2$).
\end{theorem}
\textit{Proof Sketch.}
    We first prove that the distribution of parameters $\theta_k^j$ satisfies LSI with a constant $c_k^j$ that depends on $k,j$. Then we plug the LSI constants into \cref{lem:recursive} and prove a recursive privacy bound for data points in each batch $B^{j_0}$. Finally, by carefully choosing which recursion to use and solve, we obtain the privacy bound in the theorem statement. The complete proof is in \cref{appsub:proof_strconvexsmooth}.~\qedsymbol{}
The above new privacy bound~\cref{thm:strconvexsmooth} quantifies the privacy amplification during iterations that only access the shared data points between neighboring datasets, and it is significantly smaller than the naive privacy dynamics baseline derived from~\citet{chourasia2021differential} (that does not capture this additional privacy amplification). In \cref{fig:improved_dynamics}, we observe that for data points in the first batch $B^0$, our new privacy dynamics bound is smaller by a multiplicative factor of approximately $n/b$ (where $b/n$ is the batch sampling ratio), and the improvement for data points in the last-batch is also significant.

\paragraphbe{Bound improvement compared with \citet{balle2019privacy}} In the first epoch (when $K=1$), our privacy dynamics bound (based on recursive scheme) in \cref{thm:strconvexsmooth} is of same order as the privacy amplification bound~\citet[Theorem 5]{balle2019privacy}, as shown in \cref{fig:improved_dynamics}. However, for training with multiple epochs, \citet{balle2019privacy} use a coupling-based approach, and we do not see an easy way to extend their analysis except by using R\'enyi DP composition~\cite{mironov2017renyi} over the epochs, which gives a linearly accumulating R\'enyi~DP bound for multiple epochs (as $\K$ increases). On the contrary, as shown in~\cref{fig:improved_dynamics}, our improved privacy dynamics bound converges to a constant, thus significantly improves over the R\'enyi~DP composition of~\citet[Theorem 5]{balle2019privacy}.
\section{Privacy dynamics for noisy stochastic mini-batch gradient descent}
\label{sec:sgd_subsampled}

In this section, we further improve over the position-dependent privacy dynamics analysis in~\cref{sec:fix_sgd}, by incorporating the effect of amplification of privacy loss due to the stochastic mini-batch sampling. We first prove the following lemma by the convexity of $f$-divergence with $f(x) = x^\alpha$.

\begin{lemma}[Joint convexity of scaled exponentiation of R\'enyi divergence]
    \label{lem:rootconvexity1}
    Let $\mu_1,\cdots,\mu_m$ and $\nu_1, \cdots, \nu_m$ be distributions over $R^d$. Then for any RDP order $\alpha\geq 1$, and any coefficients $p_1,\cdots,p_m\geq 0$ that satisfy $p_1+\cdots + p_m=1$, the following inequality holds.
    \begin{align}
        \label{eqn:rootconvexity1}
        e^{(\alpha-1) \cdot R_\alpha(\sum_{j=1}^mp_j\mu_j\lVert \sum_{j=1}^mp_j\nu_j)} & \leq \sum_{j=1}^mp_j\cdot e^{(\alpha-1)\cdot R_\alpha(\mu_j\lVert \nu_j)}
    \end{align}
\end{lemma}
We provide a detailed proof for~\cref{lem:rootconvexity1} in~\cref{append:rootconvexity}. This Lemma is our main tool for quantifying the additional privacy amplification under stochastic mini-batches in the rest of the~section.

\subsection{Privacy dynamics under shuffle and partition}
\label{ssec:shuffle}
The ``shuffle and partition'' batch generation scheme is commonly used for practical DP-SGD implementations in privacy libraries~\cite{tensorflowprivacy,yousefpour2021opacus}. By using~\cref{thm:strconvexsmooth} and new bounds for the privacy amplification by shuffling in~\cref{alg:noisymBGD}, we prove the following R\'enyi DP guarantee.

\begin{theorem}[Privacy dynamics under ``shuffle and partition'']
    If the loss function $\ell(\theta;x)$ is ${\lambda\text{-strongly convex}}$, $\beta$-smooth, and if its gradient has finite $\ell_2$-sensitivity $S_g$, then for $\frac{n}{b}\geq 2$, running \cref{alg:noisymBGD} for $K\geq 1$ epochs with step-size $\eta<\frac{2}{\lambda + \beta}$, under ``shuffle and partition'' batch generation scheme, satisfies $(\alpha,\eps)$-R\'enyi DP with
    \begin{align}
    \label{eqn:shuffle_dynamics}
    \eps & \leq \eps_0^{\lfloor \frac{n}{2b}\rfloor}(\alpha) \cdot \frac{1 - (1-\eta\lambda)^{2\cdot (\K - 1) \cdot (n/b - \lfloor \frac{n}{2b}\rfloor)}}{1 - (1-\eta\lambda)^{2\cdot (n/b - \lfloor \frac{n}{2b}\rfloor)} } +  \frac{1}{\alpha-1}\cdot \log\left(\underset{0\leq j_0<n/b}{Avg} e^{(\alpha - 1)\eps_{0}^{n/b - j_0}(\alpha)}\right)
    \end{align}
    where $\eps_0^{j}(\alpha) = \frac{\alpha \eta S_g^2}{4\sig^2b^2}\cdot \left(1 - \eta \lambda \right)^{2\cdot (j - 1)} \cdot \frac{1}{\sum_{s=0}^{ j - 1 } (1-\eta\lambda)^{2s}}$ for any $j = 1, \cdots, \frac{n}{b}$ (we assume $\frac{n}{b}\geq 2$).
    \label{thm:shuffle}
\end{theorem}

\begin{figure}[t!]
    \centering 
    \begin{subfigure}[b]{0.48\textwidth}
        \scalebox{0.7}{
        \begin{tikzpicture}
            \begin{axis}[
            no markers,
            samples=50,
            xmin=0,
            ymin=0,
            ymax=0.155,
            axis line style = thick,
            axis lines = left,
            ytick={0.0,0.05,0.10,0.15}, yticklabels={0,0.05,0.10,0.15},
            xlabel=$K$,
            ylabel={$\eps$ in $(\alpha, \eps)$-R\'enyi Differential Privacy},
            xmax=42,clip = true,
            clip mode=individual,axis y line*=left,axis x line*=bottom,
            legend style={at={(0.05,1.1)},anchor=west}]
                %
                %
                \addplot[thick,purple] table [x=k, y=eps, col sep=comma] {figures/data_shuffle_sample/k=40_a=15_l=1_e=0.02_g=4_s=2_n=50_b=2.csv};
                \addlegendentry{Our Privacy Dynamics (shuffle and partition)}
                %
                %
                \addplot[thick,purple,dotted,domain=0:40] {improved_dp_dynamics_last_batch(x,15,1,0.02,4,2,50,2)};
                \addlegendentry{Our Privacy Dynamics (fixed-order last mini-batch)}
                %
                %
                \addplot[thin,purple,dashed] table [x=k, y=eps, col sep=comma] {figures/data_sgm_sample/k=40_a=15_l=1_e=0.02_g=4_s=2_n=50_b=2.csv};
                \addlegendentry{DP-SGD~\cite{abadi2016deep} + SGM~\cite{mironov2019r} (composition)}
                \addplot[thin,blue,dashed] table [x=k, y=eps, col sep=comma] {figures/data_sgm_sample/k=40_a=10_l=1_e=0.02_g=4_s=2_n=50_b=2.csv};
                \addplot[thick,blue,dotted,domain=0:40] {improved_dp_dynamics_last_batch(x,10,1,0.02,4,2,50,2)};
                \addplot[thick,blue] table [x=k, y=eps, col sep=comma] {figures/data_shuffle_sample/k=40_a=10_l=1_e=0.02_g=4_s=2_n=50_b=2.csv};
                %
                %
                %
                %
                \node[anchor=west] at (axis cs: 40,{improved_dp_dynamics_last_batch(40,15,1,0.02,4,2,50,2)}) {$\alpha=15$};
                \node[anchor=west] at (axis cs: 40,{improved_dp_dynamics_last_batch(40,10,1,0.02,4,2,50,2)}) {$\alpha=10$};
            \end{axis}
        \end{tikzpicture}
        }
    \caption{shuffle and partition}
    \end{subfigure}%
    \hfill
    \begin{subfigure}[b]{0.48\textwidth}
        \scalebox{0.7}{
        \begin{tikzpicture}
            \begin{axis}[
            no markers,
            samples=50,
            xmin=0,
            ymin=0,
            axis line style = thick,
            ymax=0.155,
            axis lines = left,
            ytick={0.0,0.05,0.10,0.15}, yticklabels={0,0.05,0.10,0.15},
            xlabel=$K$,
            xmax=42,clip = true,
            clip mode=individual,axis y line*=left,axis x line*=bottom,
            legend style={at={(0.03,1.15)},anchor=west}]
                %
                \addplot[thick,purple] table [x=k, y=eps, col sep=comma] {figures/data_rec_sample/k=40_a=15_l=1_e=0.02_g=4_s=2_n=50_b=2.csv};
                \addlegendentry{Our Privacy Dynamics (samp. w.o. replacement)}
                %
                \addplot[thin,purple,dashed] table [x=k, y=eps, col sep=comma] {figures/data_sgm_sample/k=40_a=15_l=1_e=0.02_g=4_s=2_n=50_b=2.csv};
                \addlegendentry{DP-SGD~\cite{abadi2016deep} + SGM~\cite{mironov2019r} (composition)}
                %
                %
                \addplot[thin,blue,dashed] table [x=k, y=eps, col sep=comma] {figures/data_sgm_sample/k=40_a=10_l=1_e=0.02_g=4_s=2_n=50_b=2.csv};
                %
                %
                \addplot[thick,blue] table [x=k, y=eps, col sep=comma] {figures/data_rec_sample/k=40_a=10_l=1_e=0.02_g=4_s=2_n=50_b=2.csv};
                %
                %
                %
                %
                \node[anchor=west] at (axis cs: 40,{0.1453135217648525}) {$\alpha=15$};
                \node[anchor=west] at (axis cs: 40,{0.06724058347919833}) {$\alpha=10$};
            \end{axis}
        \end{tikzpicture}
        }  %
    \caption{(re)sample without replacement}
    \end{subfigure}%
    \caption[Caption without FN]{Rényi privacy loss of noisy (stochastic) mini-batch gradient descent over $K$ epochs. We show $\eps$ in the $(\q,\eps)$-RDP guarantee derived by our privacy dynamics bound under "shuffle and partition" (bold lines, left plot), our privacy dynamics bound under sampling without replacement (bold lines, right plot), our privacy dynamics bound for data points in the last batch $B^{n/b - 1}$ (thin dashed lines), and the baseline composition-based bound for DP-SGD~\cite{abadi2016deep,mironov2019r} (thin lines). We evaluate under the following setting: RDP order ${\alpha\in\{10,15\}}$;\ \  $\lam$-strongly convex loss function with $\lam =1 $; $\be$-smooth loss function with $\be=4$; finite total gradient sensitivity $S_g=4$; size of the data set $\size=50$; step-size $\step=0.02$; noise variance $\sig^2=4$, batch size $b=2$. The expressions for computing the privacy bounds are: Privacy Dynamics (shuffle and partition): \cref{thm:shuffle}~\footnotemark; Privacy Dynamics (samp. w.o. replacement): \cref{thm:amp_samp_wo_replacement}; Privacy Dynamics (last batch): \cref{thm:strconvexsmooth} under $j_0=n/b-1$; Composition: derived from Section 3.3. of~\citet{mironov2019r} that approximately equals $\frac{b}{n}\cdot \frac{\alpha \eta S_g^2}{4\sig^2b^2}\cdot \K$. }
    \label{fig:comp_shuffle}
\end{figure}
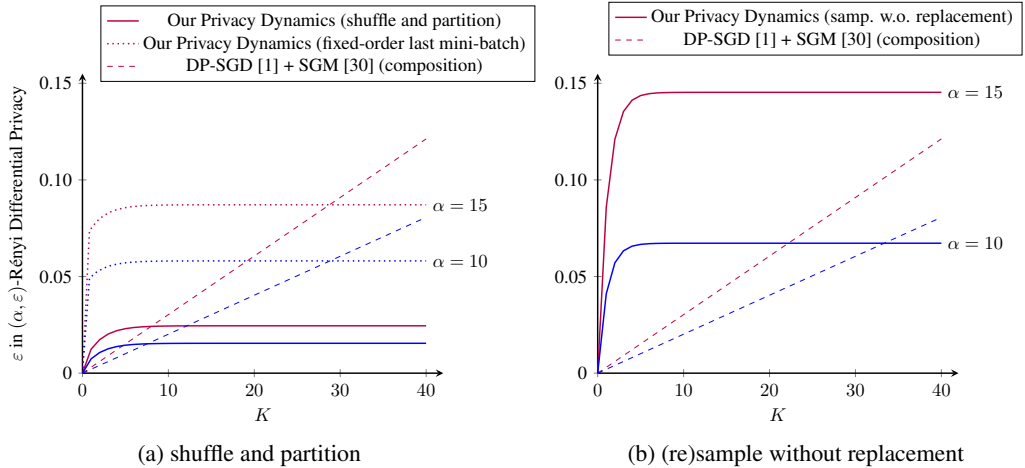

\footnotetext{When $\alpha-1$ and {\scriptsize$\eps_0^{n/b-j_0}(\alpha)$} are large, the second term in \cref{eqn:shuffle_dynamics} might overflow. In the experiments, we compute this log-sum-exp term with the shifted approximation {\scriptsize$\eps_0^{1} + \log\underset{0\leq j_0<n/b}{Avg} e^{(\alpha - 1)\left(\eps_{0}^{n/b - j_0}(\alpha) - \eps_0^1(\alpha)\right)}$}.}
The above R\'enyi DP bound~\cref{thm:shuffle} is always smaller than the R\'enyi DP bound in~\cref{thm:strconvexsmooth} for the worst-case fixed mini-batch sequence (i.e., when the differing data point is in the last mini-batch of each pass). Therefore, \cref{thm:shuffle} quantifies the privacy amplification due to shuffling, when compared to the worst-case R\'enyi DP bound among all possible mini-batch sequences. In~\cref{fig:comp_shuffle} (a), we illustrate this amplification in more details.

\paragraphbe{Comparison with prior amplification by sub-sampling bounds~\cite{mironov2019r,feldman2022hiding,feldman2022stronger}} We now investigate how our R\'enyi DP bound \cref{thm:shuffle} compares with prior bounds for privacy amplification by sub-sampled mechanisms, in terms of its amplification rate (i.e. the ratio between RDP bound for noisy mini-batch GD on sub-sampled mini-batch with size $b$, and RDP bound for full-batch noisy GD on a small dataset with size $b$). Observe that only the first term of the R\'enyi DP bound~\cref{eqn:shuffle_dynamics} increases with the number of epochs $K$, and its growth rate is strictly smaller than $\eps_0^{\lfloor\frac{n}{b}\rfloor}< \frac{b}{n}\frac{\alpha \eta S_g^2}{4\sigma^2b^2}$. Here, the term $\frac{\alpha \eta S_g^2}{4\sigma^2b^2}$ is exactly the R\'enyi DP bound for a \emph{full-batch} noisy GD update on dataset with small size $b$. Therefore,~\cref{thm:shuffle} achieves an amplification rate of $O(\frac{b}{n})$ for \emph{one epoch} of noisy mini-batch gradient descent under ``shuffle and partition''. This $O(\frac{b}{n})$ amplification rate for one epoch matches the $O(\frac{b^2}{n^2})$ amplification rate in prior bounds for one epoch of sub-sampled mechanisms~\cite{mironov2019r,feldman2022hiding,feldman2022stronger}, because one epoch consists of $\frac{n}{b}$ iterations. 

However, prior bounds for privacy amplification by sub-sampling (such as for sub-sampled Gaussian mechanism~\cite{mironov2019r} and for noisy SGD under ``shuffle and partition''~\cite{feldman2022hiding,feldman2022stronger}) only study a single update or a single epoch, and then rely on composition theorem for computing amplified privacy bounds for multiple epochs. Consequently, the R\'enyi DP bounds derived by such analyses linearly grow with the number of epochs $K$. On the contrary, our R\'enyi DP bound in \cref{thm:shuffle} applies to multiple epochs with hidden intermediate state, and thus enables a converging R\'enyi DP bound that never exceeds a maximum value. Hence, our R\'enyi DP bound in \cref{thm:shuffle} is strictly smaller than composition-based privacy bound for DP-SGD~\cite{abadi2016deep,mironov2019r} after $\frac{1}{\lambda\eta} + \frac{n}{b}$ epochs. This is because the first term in \cref{eqn:shuffle_dynamics} is strictly smaller than the composition-based privacy bound after $\frac{1}{\lambda \eta}$ epochs, and the second term in \cref{eqn:shuffle_dynamics} is strictly smaller than the composition-based privacy bound after $n/b$ epochs (as we explain with more details in~\cref{append:proofshuffle}). 

\subsection{Privacy dynamics under (re)sampling mini-batch of fixed size without replacement} 

\label{ssec:samp_wo_replacement}

We now similarly analyze privacy dynamics for another variant of noisy mini-batch gradient descent (\cref{alg:noisymBGD}), where in each iteration of each epoch, we freshly (re)sample $b$ indices without replacement from $\{1,\cdots,n\}$ to obtain a fixed-size mini-batch. That is, the mini-batches used in different iterations (of one epoch) may~overlap. This ``sampling a mini-batch of fixed size without replacement'' scheme is widely studied in the amplification by sub-sampling literature~\cite{balle2019privacy,wang2019subsampled,koskela2020computing}, as an attractive alternative that ensures fixed mini-batch size (when compared to Poisson sub-sampling). 

\begin{theorem}[Recursive amplification by sampling without replacement]
    \label{thm:amp_samp_wo_replacement}
    If the loss function $\ell(\theta;x)$ is $\lambda$-strongly convex, $\beta$-smooth, and if its gradient has finite $\ell_2$-sensitivity $S_g$, then \cref{alg:noisymBGD} under sampling without replacement and stepsize $\eta<\frac{2}{\lambda + \beta}$ satisfies $(\alpha, \eps)$-R\'enyi DP guarantee with
    \begin{align}
        \eps \leq \frac{1}{\alpha - 1} \log\left(S_\K^0(\alpha)\right)
        \label{eqn:amp_samp_wo_replacement}
    \end{align}
    where the terms $S_k^j(\alpha)$ for $k = 0,\cdots,\K-1$ and $j=0, \cdots, n/b - 1$ are recursively computed by $S_0^0(\alpha) = 1$; $S_k^{j+1}(\alpha) = \frac{b}{n}\cdot e^{\frac{(\alpha - 1) \alpha \eta S_g^2}{4\sig^2 b^2}}\cdot S_k^j(\alpha) + (1-\frac{b}{n})\cdot S_k^j(\alpha)^{(1-\eta\lambda)^2}$; and $S_{k+1}^0(\alpha) = S_k^{n/b}(\alpha)$.
\end{theorem}
Observe that in the above recursion, the term $S_{k}^{j+1}(\alpha)$ is strictly smaller than $e^{\frac{(\alpha - 1) \alpha \eta S_g^2}{4\sig^2 b^2}}\cdot S_k^j(\alpha)$ (i.e., the R\'enyi DP bound for one step of noisy gradient descent on dataset with size $b$), thus quantifying the privacy amplification by sampling a mini-batch of fixed size without replacement. We illustrate this privacy amplification in Figure~\ref{fig:comp_shuffle} (b).
However, our privacy bound under ``sample without replacement'' in~\cref{fig:comp_shuffle} (b) is larger than our privacy bound under ``shuffle and partition'' in~\cref{fig:comp_shuffle} (a), which suggests room for future improvement of the amplification rate bound. Indeed, for the special case of a single update, prior privacy amplification by subsampling bounds~\cite{wang2019subsampled,koskela2020computing} achieve better amplification rates of $O(\frac{b^2}{n^2})$. 

However, prior bounds for privacy amplification by sampling generally \emph{only} apply to \emph{a single sub-sampled update}. Therefore, for analyzing multiple iterations of updates, prior works still rely on composition theorems, which results in linear growth of R\'enyi DP bound with regard to number of epochs. On the contrary, our privacy amplification bound \cref{thm:amp_samp_wo_replacement} applies to multiple iteration of hidden-state (re)sampling without replacement steps, which enables a converging R\'enyi DP bound even for training an infinite number of epochs. In \cref{fig:comp_shuffle} (b), we compare our R\'enyi DP bound~\cref{thm:amp_samp_wo_replacement} with the baseline composition-based privacy bound for DP-SGD~\cite{abadi2016deep,mironov2017renyi}. We observe that for a range of RDP orders $\alpha=10,15, 20$, our privacy dynamics bound significantly improves over the baseline composition-based bound (after $50$ epochs).

\section{Example: privacy dynamics for DP-SGD on regularized logistic regression}

\label{sec:trade-off}

In this section, we explain a practical setting that our improved privacy dynamics analysis is applicable: training regularized logistic regression with the DP-SGD algorithm~\cite{abadi2016deep} (under ``shuffle and partition'' mini-batch sampling scheme). Note that DP-SGD algorithm under ``shuffle and partition'' is equivalent to our analyzed \cref{alg:noisymBGD} after change of notations, as discussed in \cref{app:pseudocode}. For completeness, in~\cref{alg:noisymBGD_logistic}, we also provide the pseudocode for an equivalent of DP-SGD algorithm under notations in our paper.

\subsection{How to ensure strong convexity, smoothness and finite sensitivity}
Suppose that we want to train models for image classification tasks. 
The training dataset that we take as input is $D=(\z_1, \cdots, \z_n)$, where each data record $\z_i=(\x_i, \y_i)$ consists of the $d$-dimentional data input feature vector $\x_i\in\mathbb{R}^{d}$, and the label vector $\y_i\in \{0,1\}^c$ (in one-hot encoding). To satisfy the necessary conditions for our privacy dynamics bound, we use regularized logistic regression to ensure strong convexity, use feature clipping to ensure smoothness, and use gradient clipping to enforce finite gradient sensitivity, as follows. 
\label{app:proof_str_cvx}

\paragraphbe{Regularzied Logistic regression (for strong convexity)} The loss function for regularized logistic regression in the multi-class setting (with per-class bias) is as follows.
\begin{align}
    \label{eqn:reg_loss_logistic}
    \ell_\lambda(\theta;\x,\y)&= \ell_0(\theta;\x,\y) + \frac{\lambda}{2}\lVert \theta \rVert_2^2
\end{align}
where $\ell_0(\theta;\x,\y)$ is the following logistic regression loss function.
\begin{align}
    \ell_0(\theta;\x,\y) = - \y^1 \log\left(\frac{e^{\bar{x}^T\cdot \theta_1}}{e^{\bar{x}^T\cdot \theta_1} + \cdots + e^{\bar{x}^T\cdot \theta_c}}\right) - \cdots - \y^c \log\left(\frac{e^{\bar{x}^T\cdot \theta_c}}{e^{\bar{x}^T\cdot \theta_1} + \cdots + e^{\bar{x}^T\cdot \theta_c}}\right)
    \label{eqn:loss_logistic}
\end{align}
where $\bar{\x} = (\x, 1) \in \mathbb{R}^{d+1}$ denotes the concatenation of the data feature vector $\x$ and $1$, and $\y = (\y^1,\cdots,\y^c)$ is the label vector. The parameter vector is $\theta = (\theta_1,\cdots,\theta_c)\in \mathbb{R}^{(d+1)\cdot c}$ that represents the weight and the per-class bias of the linear model. The logistic regression loss function is convex, and therefore the regularized logistic regression loss function is $\lambda$-strongly convex.

\paragraphbe{Feature Clipping (for bounding the smoothness constant)} To ensure that the condition of loss function smoothness with regard to parameters $\theta$ is satisfied, we follow \citet{feldman2018privacy} and normalize the data feature vector in $\ell_2$ norm, such that $\lVert\x\rVert_2\leq \Lip$. Under this data feature clipping, we prove that the logistic regression loss function \eqref{eqn:loss_logistic} is $(\frac{\Lip^2+1}{2})$-smooth in the following \cref{prop:smooth}.
\begin{proposition}
    \label{prop:smooth}
    If the data feature vector $\x$ has bounded $\ell_2$ norm, such that $\lVert \x\rVert_2\leq L$, then the unregularized logistic regression loss function $\ell_0(\theta;\x,\y)$ \cref{eqn:loss_logistic} is convex
    , $L$-Lipschitz 
    and $\beta$-smooth with regard to parameters $\theta$, for
    \begin{align}
        L & = \sqrt{2 (L^2 + 1)} \label{eqn:sensitivity}\\
        \beta & = \frac{L^2 + 1}{2} \label{eqn:smoothness}
    \end{align}
\end{proposition}

By \cref{prop:smooth}, the regularized logistic regression loss function~\cref{eqn:reg_loss_logistic} is $(\beta + \lambda)$-smooth. The above feature clipping technique is different from the DP-SGD algorithm~\cite{abadi2016deep} that only requires per-example gradient clipping. The major reason that we use data feature clipping (besides per-example gradient clipping), is for ensuring smoothness of the logistic regression loss function (by \cref{prop:smooth}), which is a necessary condition for applying our privacy bound \cref{thm:strconvexsmooth}.

\paragraphbe{Per-example Clipping on Unregularized Gradient (for reducing gradient sensitivity without harming smoothness or strong convexity)} Although feature clipping already bounds the gradient sensitivity by $2\sqrt{2(L^2+1)}$ (by \cref{prop:smooth}), this bound grows with the feature clipping norm $L$. This in turn restricts the signal to noise ratio, and tends to give suboptimal privacy-utility trade-off in practical experiments. Therefore, we additionally perform per-example $\ell_2$-clipping on the unregularized gradient (detailed pseudocode in \cref{app:pseudocode}). Under per-example clipping on unregularized gradient, we prove in the following \cref{prop:smooth_strcvx_clip}, that each gradient update in \textit{regularized logistic regression} has finite gradient sensitivity, and preserves strong convexity and smoothness.

\begin{proposition}
    \label{prop:smooth_strcvx_clip}
    Let $\ell_0(\theta;\x,\y)$ be the logistic regression loss function defined in~\cref{eqn:loss_logistic}. Let $g_0(\theta;\x, \y) = \frac{\nabla \ell_0(\theta;\x,\y)}{\lVert \nabla \ell_0(\theta;\x,\y)\rVert_2}\cdot \min\{\lVert \nabla \ell_0(\theta;\x,\y)\rVert_2, \frac{S_g}{2}\}$ be the clipped gradient of (unregularized) loss function $\ell_0(\theta;\x,\y)$, under $\ell_2$ clipping norm $\frac{S_g}{2}$. If $g(\theta;\x,\y) = g_0(\theta;\x,\y) + \lambda\theta$, and if the data vector $\x$ has bounded $\ell_2$ norm, such that $\lVert \x\rVert_2\leq L$, then $g(\theta;\x,y)$ has finite $\ell_2$-sensitivity $S_g$,  is continuous, and is almost everywhere differentiable with
    \begin{align}
        \label{eqn:prop_smooth_strcvx_clip}
        \lambda\cdot \mathbb{I}_{(d+1)\cdot c}\preceq \nabla_\theta g(\theta;\x,\y)\preceq (\beta + \lambda) \cdot \mathbb{I}_{(d+1)\cdot c}
    \end{align}
    for any $\theta,\theta'\in\mathbb{R}^{(d+1)\cdot c}$ and $\beta = \frac{L^2+1}{2}$.
\end{proposition}

We provide complete proof for this proposition in~\cref{app:proof_gradient_sensitivity}. This construction of clipped unregularized gradient enables us to enjoy the benefits of gradient clipping (such as for speeding up convergence~\cite{zhang2019gradient,chen2020understanding}) while satsifying the necessary smoothness and strong convexity conditions for applying our privacy dynamics bound.

\subsection{Composition-based privacy bound and privacy dynamics analysis for DP-SGD}

\paragraphbe{Baseline composition-based privacy analysis} To the best of our knowledge, the moments accountant~\cite{abadi2016deep} combined with the privacy amplification bound for Subsampled Gaussian Mechanism~\cite{mironov2019r} gives the strongest baseline composition-based privacy bound for DP-SGD. The original DP-SGD bound only holds for mini-batch sampling from a Poisson distribution. However, currently many privacy libraries~\cite{tensorflowprivacy, yousefpour2021opacus} still apply the DP-SGD analysis while implementing the ``shuffle and partition'' sampling scheme. Recent works~\cite{feldman2022hiding,feldman2022stronger} further prove that the R\'enyi DP bound for DP-SGD under ``shuffle and partition'' matches that for DP-SGD under Poisson sampling in order. Therefore, we consider the moments accountant bound for subsampled Gaussian mechanisms~\cite{mironov2019r} as a reasonable comparison baseline for our privacy bound of DP-SGD \cref{alg:noisymBGD_logistic} under ``shuffle and partition''.

\paragraphbe{Improved privacy dynamics analysis} We now compute our privacy dynamics bound \cref{thm:shuffle} for the modified DP-SGD algorithm~\cref{alg:noisymBGD_logistic} under ``shuffle and partition''. By plugging the notation transformation in \cref{app:pseudocode} into \cref{thm:shuffle}, we prove the following privacy dynamics theorem for noisy mini-batch gradient descent on regularized logistic regression.

\begin{corollary}[Privacy dynamics for noisy mini-batch gradient descent on regularized logistic regression]
    For the regularized logistic regression loss~\cref{eqn:reg_loss_logistic} with regularization coefficient $\lambda$, if the data feature vector is clipped in $\ell_2$-norm by $L$, and the unregularized gradient is clipped in $\ell_2$ norm by $\frac{S_g}{2}$, then for $\K\geq 1$ and $\frac{n}{b}\geq 2$, \cref{alg:noisymBGD_logistic} with stepsize $\eta<\frac{2}{(L^2 + 1)/2 + 2 \lam}$ and noise multiplier $\sig_{mul}$
    satisfies $(\alpha, \eps_{norm} + \eps)$-R\'enyi Differential Privacy with
    \begin{align}
        \label{eqn:shuffle_logistic}
        \eps & \leq \eps_0^{\lfloor \frac{n}{2b}\rfloor}(\alpha) \cdot \frac{1 - (1-\eta\lambda)^{2\cdot (\K - 1) \cdot (n/b - \lfloor \frac{n}{2b}\rfloor)}}{1 - (1-\eta\lambda)^{2\cdot (n/b - \lfloor \frac{n}{2b}\rfloor)} } +  \frac{1}{\alpha-1}\cdot \log\left(\underset{0\leq j_0<n/b}{Avg} e^{(\alpha - 1)\eps_{0}^{n/b - j_0}(\alpha)}\right)
    \end{align}
    where the terms $\eps_0^{j}(\alpha)$ is upper-bounded for any $j = 1, \cdots, n/b$ as follows.
    \begin{align}
        \label{eqn:shuffle_logistic_special}
        \eps_0^{j}(\alpha) \leq \frac{2 \alpha}{\sig_{mul}^2}\cdot \left(1 - \eta \lambda \right)^{2\cdot (j - 1)} \cdot \frac{1}{\sum_{s=0}^{ j - 1 } (1-\eta\lambda)^{2s}}
    \end{align}
    \label{cor:logistic}
\end{corollary}

\begin{proof}
    The proof is by applying \cref{thm:shuffle} with gradient sensitivity $S_g$, smoothness constant $\frac{L^2 + 1}{2} + \lambda$, strong convexity constant $\lambda$ (derived by \cref{prop:smooth_strcvx_clip})) and noise standard deviation $\sig = \sqrt{\frac{\eta}{2}}\cdot \frac{1}{b}\cdot \sig_{mul} \cdot \frac{S_g}{2}$ (explained in \cref{app:pseudocode}). 
\end{proof}
\section{Conclusions and Discussion}

\label{sec:conclusion}

We prove a novel converging last-iterate privacy bound for noisy stochastic mini-batch gradient descent on strongly convex smooth loss functions. Our bound substantially improves the prior privacy dynamics bound for noisy GD~\cite{chourasia2021differential}, by proving novel bounds for the additional privacy amplification (by randomized post-processing and sub-sampling) during training with stochastic mini-batches.  Our results show that to obtain tighter privacy bound (thus achieving better privacy accuracy trade-off), differentially private learning algorithms needs to be evaluated by a last-iterate privacy bound, unless it has a very fast convergence (under which, due to the small number of epochs, the cost of composition bound is not significant). 

\paragraphbe{Future Work and Other Related Work}
For iteratively resampling a mini-batch of fixed size without replacement, there is room for improving our hidden-state privacy amplification bound~\cref{thm:amp_samp_wo_replacement} under smaller batch size $b$, because under \emph{one} sampling without replacement step, the best known amplification rate is as small as $O(\frac{b^2}{n^2})$~\cite{balle2019privacy,wang2019subsampled}. This suggests possibility for a better \emph{hidden-state} privacy amplification bound (without using composition) for multiple steps of~sub-sampling.

In our current analysis,  strong convexity and smoothness of the loss functions are \emph{necessary conditions} for obtaining a \emph{converging} hidden-state R\'enyi DP bound. More specifically, strong convexity and smoothness ensure that log-Sobolev inequality with constant $c$ (which is a condition required by Lemma~\ref{lem:recursive}) holds throughout the training process (i.e. after any number of epochs $k$ and iterations $j$). 
A recent work~\citet{altschuler2022privacy} consider constrained optimization problem where $\theta$ is optimized over a bounded set with finite $\ell_2$ diameter (rather than the unconstrained parameter space $\mathbb{R}^d$ considered in this paper). Under such setting, they prove that convexity (instead of strong convexity) and smoothness of the loss functions suffice to enable a time-independent R\'enyi DP bound for noisy SGD. However, it remains an important open problem to further relax the convexity and smoothness conditions for proving converging R\'enyi DP bound. 

One of the most important motivation for studying hidden-state R\'enyi DP bound in this paper, is that most State-of-the-art privacy accuracy trade-off for differentially private learning are achieved via the last-iterate model (instead of average of iterates). However, in terms of theoretical privacy utility trade-off, prior works~\cite{bassily2014private,bassily2019private} prove that the average of iterates achieves optimal privacy-utility trade-off in certain regimes. For convex stochastic optimization, by exploiting the privacy amplification of last-iterate, recent works~\cite{feldman2020private,chourasia2021differential,ganesh2022langevin} prove that last-iterate also achieves asymptotically optimal privacy-utility trade-off (for appropriate choice of step sizes or continuous-time algorithm). Therefore, it remains an interesting open problem as to whether last-iterate or average of iterate enables better (non-assymptotic) privacy accuracy trade-off.

\begin{ack}
    The authors would like to thank Yaodong Yu, Chuan Guo, Maziar Sanjabi and anonymous reviewers for helpful discussions on drafts of this paper. This research is supported by Google PDPO faculty research award,  Intel within the www.private-ai.org center, Meta faculty research award,  the NUS Early Career Research Award (NUS ECRA award number NUS ECRA FY19 P16), and the National Research Foundation, Singapore under its Strategic Capability Research Centres Funding Initiative. Any opinions, findings and conclusions or recommendations expressed in this material are those of the author(s) and do not reflect the views of National Research Foundation, Singapore.
\end{ack}

\bibliography{reference}
\bibliographystyle{plainnat}

\newpage
\renewcommand*{\proofname}{Proof}

\appendix
\noptcrule
\part{Appendix}
\parttoc

\newpage
\section{Symbols}

\centerline{\bf Algorithms and Definitions}
\bgroup
\def\arraystretch{1.5}
\begin{tabular}{p{1.25in}p{3in}p{0.92in}}
\toprule
Symbol & Meaning & Where\\
\midrule
$\displaystyle \x$ & A data vector & \cref{alg:noisymBGD}\\
$\displaystyle n$ & Size of a dataset & \cref{alg:noisymBGD}\\
$\displaystyle b$ & Size of a mini-batch & \cref{alg:noisymBGD}\\
$K$ & Total number of epochs in a learning algorithm & \cref{alg:noisymBGD}\\
$\displaystyle \mathbb{I}_d$ & Identity matrix with $d$ rows and $d$ columns & \cref{alg:noisymBGD}\\
$\displaystyle \theta$ & Model parameters & \cref{alg:noisymBGD}\\
$\displaystyle \ell(\theta,\x)$ & A loss function of $\theta$ parameterized by $\x$ & \cref{alg:noisymBGD}\\
$\displaystyle D,D'$ & Neighboring datasets that differ in at most one record & \cref{sec:prelim}\\
$\displaystyle S_g$ & $\ell_2$-sensitivity of total gradient ${g(\theta;D)=\sum\limits_{\x\in D}\nabla\ell(\theta;\x)}$ with regard to neighboring datasets. More specifically, $S_g=\max_{D,D',\theta}\lVert g(\theta,D)-g(\theta,D')\rVert_2$ & \cref{alg:noisymBGD}\\
$\displaystyle \mathcal{N} ( \mu , \sigma\cdot \mathbb{I}_d)$ & Gaussian distribution over $\mathbb{R}^d$ with mean $\mu$ and covariance matrix $\sigma\cdot \mathbb{I}_d$ & \cref{alg:noisymBGD}\\
$\displaystyle \eta$ & Stepsize for each iterative update in the learning algorithm & \cref{alg:noisymBGD}\\
$\displaystyle \{0, 1, \dots, n \}$ & The set of all integers between $0$ and $n$&\cref{alg:noisymBGD}\\
$\displaystyle \x_i$ & Data point $i$ of dataset $D$, with indexing starting at 1 & \cref{alg:noisymBGD}\\
$\displaystyle \theta_k^j, B_k^j$ & The parameters or generated mini-batch after $k$ epochs and $j$ iterations of a learning algorithm, with indexing starting at 0 & \cref{alg:noisymBGD} \\
$\displaystyle \alpha$ & R\'enyi DP order & \cref{sec:prelim}\\
$\displaystyle (\eps, \delta)$ & Differential Privacy parameters& \cref{sec:prelim}\\
$\displaystyle (\alpha, \eps)$ & R\'enyi Differential Privacy parameters& \cref{sec:prelim}\\
$c$ & A log-Sobolev inequality constant & \cref{sec:prelim}\\
$\lambda$ & The strong convexity parameter for a loss function & \cref{sec:fix_sgd}\\
$\beta$ & The smoothness parameter for a loss function & \cref{sec:fix_sgd}\\
\bottomrule
\end{tabular}
\egroup
\vspace{0.25cm}

\newpage

\centerline{\bf Probability and Information Theory}
\bgroup
\def\arraystretch{1.5}
\begin{tabular}{p{1.25in}p{3in}p{0.92in}}
    \toprule
    Symbol & Meaning & Where\\
    \midrule
$\displaystyle \mu$ & A distribution over $\mathbb{R}^d$ with density $\mu(\theta)$ & \cref{sec:fix_sgd}\\
$\displaystyle p(\theta)$ & A probability density function over $\theta\in\mathbb{R}^d$ & \cref{sec:overview}\\
$\displaystyle f_\#(\mu)$ & The push forward distribution of $\mu$ under mapping $f$ on the same domain & \cref{sec:fix_sgd}\\
$\displaystyle \mu*\nu$ & The convolution of two distributions $\mu$ and $\nu$ & \cref{sec:fix_sgd}\\
$\displaystyle R_\alpha(\mu\lVert\nu)$ & R\'enyi divergence between distributions $\mu$ and $\nu$ & \cref{sec:prelim}\\
$\displaystyle  \E_{\theta\sim \mu} [ f(\theta) ]$ & Expectation of $f(\theta)$ with respect to $\mu(\theta)$ & \cref{sec:prelim}\\
$\displaystyle KL ( \mu \Vert \nu ) $ & Kullback-Leibler divergence between distributions $\mu$ and $\nu$ & \cref{append:proof_fix_sgd}\\
\bottomrule
\end{tabular}
\egroup
\vspace{0.25cm}

\centerline{\bf Calculus and Linear Algebra}
\bgroup
\def\arraystretch{1.5}

\begin{tabular}{p{1.25in}p{4.0in}}
\toprule
Symbol & Meaning\\
\midrule
$t$& A real scalar\\
$\lfloor t\rfloor$ & The largest integer that is smaller than or equal than a real number $t$ \\
$\displaystyle A_{i,j}$ & Element $i, j$ of matrix $A$\\
$\displaystyle\frac{d y} {d x}$ & Derivative of $y$ with respect to $x$\\ [2ex]
$\displaystyle \frac{\partial y} {\partial x}, \nabla_xy $ & Partial derivative of $y$ with respect to $x$ \\
$\displaystyle \nabla_x^2y $ & Laplacian of $y$ with respect to $x$ \\
$\displaystyle \nabla_x y $ & Gradient of $y$ with respect to $x$ \\
$\displaystyle x \cdot y $ & Inner products between real vectors $x$ and $y$\\
$x$ & $\ell_2$ norm of a vector $x$\\
$\displaystyle A\otimes B$ & Tensor product between two real matrices $A$ and $B$\\
$\displaystyle x^T$ & Transpose of a real vector or matrix $x$\\
$\displaystyle f \circ g $ & Composition of the functions or mappings $f$ and $g$ \\
$A\preceq B$& The matrix $B-A$ is semi-positive definite.\\
\bottomrule
\end{tabular}
\egroup
\vspace{0.25cm}
\section{Preliminaries}

\label{sec:prelim}
\begin{definition}[Differential Privacy~\cite{dwork2006calibrating, dwork2014algorithmic}]
    A randomized algorithm $\mathcal{A}$ is $(\eps,\delta)$-differentially private if for any neighboring datasets $D, D'$, and for all possible event $S$ in the output space of $\mathcal{A}$, 
    \begin{align}
        P\left(\mathcal{A}(D)\in S\right) \leq e^\eps\cdot P\left(\mathcal{A}(D')\in S\right) + \delta
    \end{align}
    where we say $D, D'$ are \textit{neighboring} if they are of the \emph{same size} and differ in at most one data record.
\end{definition}

\begin{definition}[$(\alpha,\eps)$-R\'enyi DP~\cite{mironov2017renyi}]
    A randomized algorithm $\mathcal{A}$ is said to satisfy $(\alpha, \eps)$-R\'enyi differential privacy (or $(\alpha,\eps)$-R\'enyi DP for short), if for any neighboring datasets $D$ and $D'$, 
    \begin{align}
        R_\alpha(\mathcal{A}(D)\lVert \mathcal{A}(D')) \leq \eps,\text{ where }R_\alpha(\mu\lVert \nu) = \frac{1}{\alpha -1}\log \mathbb{E}_{\theta\sim\nu}\left[\left(\frac{\mu(\theta)}{\nu(\theta)}\right)^\alpha\right]
    \end{align} 
    where $\mathcal{A}(D)$ ($\mathcal{A}(D')$) denote the distribution of output given input dataset $D$ ($D'$), and $R_{\alpha}(\mu\lVert \nu)$ is the \textit{R\'enyi divergence}~\cite{renyi1961measures} of order $\alpha>1$ for two distributions with density $\mu(\theta)$ and~$\nu(\theta)$ on $\mathbb{R}^d$.
    \label{def:Renyidivergence}
\end{definition}

\begin{definition}[log-Sobolev Inequality~\cite{vempala2019rapid}]
    A distribution $\nu$ over $\mathbb{R}^d$ satisfies the log-Sobolev inequality (LSI) with constant $c$ if for all smooth function $g:\mathbb{R}^d\rightarrow \mathbb{R}$ with ${\mathbb{E}_{\theta\sim\nu}\left[g(\theta)^2\right]<\infty}$,
    \begin{align}
        \mathbb{E}_{\theta\sim\nu}\left[g(\theta)^2\log\left( g(\theta)^2\right)\right] - \mathbb{E}_{\theta\sim\nu}\left[g(\theta)^2\right]\cdot \log \mathbb{E}_{\theta\sim\nu}\left[g(\theta)^2\right]\leq \frac{2}{c}\mathbb{E}_{\theta\sim\nu}\left[\lVert \nabla g(\theta)\rVert^2\right].
    \end{align}
\end{definition}
\section{Discussion about the concurrent work~\cite[Corollary 3.3]{ryffel2022differential}}

\label{app:privacy_dynamics_baseline}

\citet{chourasia2021differential} prove privacy dynamics bound for noisy gradient descent, and \citet{ryffel2022differential} extend this bound to SGLD, by directly viewing each mini-batch update as gradient descent on a smaller dataset of size $b$ (which is the size of a mini-batch). This approach is similar to our approach for deriving the naive privacy dynamics baseline~\cref{thm:naive_RDP_baseline}, and the expression in \cite[Corollary 3.3]{ryffel2022differential} is very similar to the expression in our~\cref{thm:naive_RDP_baseline}, except for having $n^2$ (instead of $b^2$) in the bound denominator. 

However, an inspection of the proof for \cite[Lemma 3.4]{ryffel2022differential} shows that, this difference between~\cite[Corollary 3.3]{ryffel2022differential} and our~\cref{thm:naive_RDP_baseline} is caused because by a flawed assumption in~\cite[Lemma 3.3]{ryffel2022differential}. More specifically, \cite[Lemma 3.3]{ryffel2022differential} wrongly assume that the LSI constant proved in~\cite[Lemma 5]{chourasia2021differential} (which only holds for a GD process) would also similarly hold for a SGLD process that takes the form of a more complex mixture distribution. Here, each mixture component is the conditional distribution of last-iterate parameters given a fixed sequence of mini-batches.

This assumption is wrong because a given mixture distribution generally satisfies a different LSI constant than each of its component distributions. Moreover, bounding the LSI constant for a mixture distribution is largely an open problem~\cite{zimmermann2013logarithmic,zimmermann2016elementary,wang2016functional,bardet2018functional,chen2021dimension}. The current best bound for this problem, to the best of our knowledge, is \cite[Theorem 1]{chen2021dimension}, which says that the LSI constant for a mixture distribution, depends on the LSI constant for the distribution of each component, and the \emph{worst-case} $\chi^2$ distance between \emph{any} two components' distributions. Therefore, the actual LSI constant for SGLD process would be significantly smaller (related to the number of components in the mixture distribution) than the assumed LSI constant in~\cite[Lemma 3.4]{ryffel2022differential} (which only holds for the conditional parameter distribution given \emph{a fixed mini-batch sequence}). After replacing the wrongly assumed LSI constant in~\cite[Lemma 3.4]{ryffel2022differential} with a correct LSI constant for SGLD process (that takes the form of mixture distribution), the privacy dynamics bounds in~\cite[Corollary 3.3]{ryffel2022differential} would be significantly worse (larger). 

Due to this flawed assumption, in this paper, we do not compare our improved privacy dynamics theorem~\cref{thm:strconvexsmooth} with~\cite[Corollary 3.3]{ryffel2022differential}. Instead, we compare~\cref{thm:strconvexsmooth} with the naive privacy dynamics baseline~\cref{thm:naive_RDP_baseline} (which has similar expression as~\cite[Corollary 3.3]{ryffel2022differential}) in~\cref{fig:improved_dynamics}.

\section{Proof for \cref{sec:fix_sgd}}

\label{append:proof_fix_sgd}

We first establish a tool lemma for proving~\cref{lem:post_proc} and obtaining the partial differential inequality that bounds the growth of differential privacy loss.

\begin{lemma}[Lemma 5 in \citet{vempala2019rapid}]
    \label{lem:i_by_e_lower}
    Suppose $\nu$ is a distribution that satisfies log-Sobolev inequality with constant $c>0$, and that has smooth density $\mu(\theta)$. Let $\alpha\geq 1$. For all measure $\mu$ that has smooth density $\nu(\theta)$, 
    \begin{align}
        \frac{I_\alpha(\mu\lVert \nu)}{E_\alpha(\mu\lVert \nu)} \geq \frac{2c}{\alpha^2}\cdot R_{\alpha}(\mu\lVert \nu) + \frac{2c}{\alpha^2}\cdot \alpha (\alpha - 1)\frac{\partial R_\alpha(\mu\lVert\nu)}{\partial \alpha} 
    \end{align}
    where $I_\alpha(\mu\lVert\nu) = \mathbb{E}_{\theta\sim\nu}\left[\left(\frac{\mu(\theta)}{\nu(\theta)}\right)^\alpha \cdot \left\lVert\nabla \log\frac{\mu(\theta)}{\nu(\theta)}\right\rVert^2\right]$, and $E_\alpha(\mu\lVert \nu) = \mathbb{E}_{\theta\sim\nu}\left[\left(\frac{\mu(\theta)}{\nu(\theta)}\right)^\alpha\right]$
\end{lemma}

\begin{proof}
    This Lemma is initially proved in \citet{vempala2019rapid} Lemma 5. We give an alternative proof here, that only uses one step of inequality in~\cref{eqn:klfisher}. We hope this alternative proof helps understand whether there is still room for improving this Lemma (e.g. by improving the inequality in \cref{eqn:klfisher}).

    We denote $\rho$ to be another distribution with density  $\rho(\theta) = \frac{1}{E_\alpha(\mu\lVert \nu)}\cdot \left(\frac{\mu(\theta)}{\nu(\theta)}\right)^\alpha \cdot \nu(\theta)$. By simple integration, we verify that $\int \rho(\theta)d\theta = \frac{E_\alpha(\mu\lVert \nu)}{E_{\alpha}(\mu\lVert \nu)} = 1$. By definition, 

    \begin{align}
        \frac{I_\alpha(\mu\lVert \nu)}{E_\alpha(\mu\lVert \nu)} & = \frac{\mathbb{E}_{\theta\sim\nu}\left[\left(\frac{\mu(\theta)}{\nu(\theta)}\right)^\alpha \cdot \left\lVert\nabla \log\frac{\mu(\theta)}{\nu(\theta)}\right\rVert^2\right]}{E_\alpha(\mu\lVert \nu)}\\
        & = \mathbb{E}_{\theta\sim\nu}\left[ \frac{\rho(\theta)}{\nu(\theta)} \cdot \left\lVert\nabla \log\frac{\mu(\theta)}{\nu(\theta)}\right\rVert^2 \right]\\
        & = \frac{1}{\alpha^2}\mathbb{E}_{\theta\sim\nu}\left[ \frac{\rho(\theta)}{\nu(\theta)} \cdot \left\lVert\nabla \log\left(\frac{\mu(\theta)}{\nu(\theta)}\right)^\alpha\right\rVert^2 \right]\\
        \text{by $\nabla \log E_{\alpha}(\mu\lVert\nu) = 0$,}\quad & = \frac{1}{\alpha^2}\mathbb{E}_{\theta\sim\nu}\left[ \frac{\rho(\theta)}{\nu(\theta)} \cdot \left\lVert\nabla \log\left(\frac{\mu(\theta)}{\nu(\theta)}\right)^\alpha - \nabla \log E_\alpha(\mu\lVert \nu)\right\rVert^2 \right]\\
        & = \frac{1}{\alpha^2}\mathbb{E}_{\theta\sim\nu}\left[ \frac{\rho(\theta)}{\nu(\theta)} \cdot \left\lVert\nabla \log\frac{\rho(\theta)}{\nu(\theta)}\right\rVert^2 \right]
        \label{eqn:i_by_e_fisher}
    \end{align}
    By definition, $\mathbb{E}_{\theta\sim\nu}\left[ \frac{\rho(\theta)}{\nu(\theta)} \cdot \left\lVert\nabla \log\frac{\rho(\theta)}{\nu(\theta)}\right\rVert^2 \right]$ is the relative Fisher information $J(\rho\lVert \nu)$ of $\rho$ with respect to $\nu$. A celebrated equivalence result (e.g. Section 2.2 of \citet{vempala2019rapid}) says that, if and only if $\nu$ satsifies log-Sobolev inequality with constant $c$, the following relation between the KL divergence and relative Fisher information holds for all $\rho$:
    \begin{align}
        KL(\rho\lVert \nu)\leq \frac{1}{2c} J(\rho\lVert \nu)
        \label{eqn:klfisher}
    \end{align}
    By plugging \cref{eqn:klfisher} into \cref{eqn:i_by_e_fisher}, we prove 
    \begin{align}
        \frac{I_\alpha(\mu\lVert \nu)}{E_\alpha(\mu\lVert \nu)} & \geq \frac{2c}{\alpha^2}\cdot KL(\rho\lVert \nu)\\
        \text{by definition,}\quad& = \frac{2c}{\alpha^2}\cdot \int_{\mathbb{R}^n}\rho(\theta)\log\frac{\rho(\theta)}{\nu(\theta)}d\theta\\
        \text{by definition of $\rho$,}\quad& = \frac{2c}{\alpha^2}\cdot \mathbb{E}_{\theta\sim \nu}\left[\frac{1}{E_{\alpha}(\mu\lVert\nu)}\cdot \left(\frac{\mu(\theta)}{\nu(\theta)}\right)^{\alpha}\cdot \left(\alpha \log \frac{\mu(\theta)}{\nu(\theta)} - \log E_{\alpha}(\mu\lVert\nu)\right)\right]\\
        & = \frac{2c}{\alpha} \cdot \frac{1}{E_\alpha(\mu\lVert\nu)}\cdot \mathbb{E}_{\theta\sim \nu}\left[\frac{\partial}{\partial \alpha}\left(\frac{\mu(\theta)}{\nu(\theta)}\right)^{\alpha}\right] - \frac{2c}{\alpha^2}\cdot \frac{\log E_{\alpha}(\mu\lVert\nu)}{E_{\alpha}(\mu\lVert\nu)}\cdot \mathbb{E}_{\theta\sim\nu}\left[\left(\frac{\mu(\theta)}{\nu(\theta)}\right)^{\alpha}\right]\\
        & = \frac{2c}{\alpha} \cdot \frac{1}{E_\alpha(\mu\lVert\nu)}\cdot \mathbb{E}_{\theta\sim \nu}\left[\frac{\partial}{\partial \alpha}\left(\frac{\mu(\theta)}{\nu(\theta)}\right)^{\alpha}\right] - \frac{2c}{\alpha^2}\log E_\alpha(\mu\lVert\nu)
        \label{eqn:i_by_e_lower_above}
    \end{align}

    By exchanging the order of derivative and expectation (because $\mu(\theta)$ and $\nu(\theta)$ are smooth densities), we prove that \cref{eqn:i_by_e_lower_above} is equivalent to the following inequality.
    \begin{align}
        \frac{I_\alpha(\mu\lVert \nu)}{E_\alpha(\mu\lVert \nu)} & \geq \frac{2c}{\alpha}\cdot \frac{\partial}{\partial \alpha}\log \mathbb{E}_\alpha(\mu\lVert\nu) - \frac{2c}{\alpha^2}\log E_\alpha(\mu\lVert\nu)
    \end{align}

    By definition, we have $E_{\alpha}(\mu\lVert\nu) = e^{(\alpha-1)R_\alpha(\mu\lVert\nu)}$, therefore we prove
    \begin{align}
        \frac{I_\alpha(\mu\lVert \nu)}{E_\alpha(\mu\lVert \nu)} & \geq  \frac{2c}{\alpha}\cdot \left(R_{\alpha}(\mu\lVert \nu) + (\alpha - 1)\frac{\partial R_\alpha(\mu\lVert\nu)}{\partial \alpha}\right)  - \frac{2c}{\alpha^2}\cdot (\alpha - 1) R_{\alpha}(\mu\lVert\nu)\\
        & = \frac{2c}{\alpha^2}\cdot R_{\alpha}(\mu\lVert \nu) + \frac{2c}{\alpha^2}\cdot \alpha (\alpha - 1)\frac{\partial R_\alpha(\mu\lVert\nu)}{\partial \alpha} 
    \end{align}
\end{proof}

\subsection{Proof for \cref{lem:post_proc}}

\label{appsub:proof_post_proc}

\begin{replemma}{lem:post_proc}
    Let $\mu, \nu$ be two distributions on $\mathbb{R}^d$. Let $f:\mathbb{R}^d\rightarrow\mathbb{R}^d$ be a measurable mapping on $\mathbb{R}^d$. We denote $\mathcal{N}(0,2t\sig^2\cdot\mathbb{I}_d)$ to be the standard Gaussian distribution on $\mathbb{R}^d$ with covariance matrix~$2t\sigma^2\cdot \mathbb{I}_d$. We denote $p_t(\theta)$ and $p_t'(\theta)$ to be the probability density functions for the distributions $f_{\#}(\mu) * \mathcal{N}(0,2t\sig^2\mathbb{I}_d)$ and $f_{\#}(\nu) * \mathcal{N}(0,2t\sig^2\mathbb{I}_d)$ respectively, where $f_\#(\mu), f_\#(\nu)$ denote the push forward distributions of $\mu, \nu$ under mapping $f$. Then if $\mu$ and $\nu$ satisfy log-Sobolev inequality with constant $c$, and if the mapping $f$ is $L$-Lipschitz, then for any order $\alpha> 1$,
    \begin{align}
       \frac{\partial}{\partial t}R_\alpha \left( p_t(\theta)\lVert p_t'(\theta)\right) & \leq - c_t\cdot 2\sig^2\cdot \left(\frac{R_\alpha(p_t(\theta)\lVert p_t'(\theta))}{\alpha} + (\alpha-1)\cdot \frac{\partial}{\partial\alpha} R_\alpha(p_t(\theta)\lVert p_t'(\theta)) \right),
    \end{align} 
    where $c_t = \left(\frac{L^2}{c} + 2t\sig^2\right)^{-1}$ is the log-Sobolev inequality constant for distributions $p_t(\theta)$ and $p_t'(\theta)$.
\end{replemma}
\begin{proof}
    By definition, $p_t(\theta)$ and $p_t'(\theta)$ are probability density functions for distributions $f_{\#}(\mu) * \mathcal{N}(0,2t\sig^2\mathbb{I}_d)$ and $f_{\#}(\nu) * \mathcal{N}(0,2t\sig^2\mathbb{I}_d)$ respectively. Therefore $p_t(\theta)$ and $p_t'(\theta)$ satisfy the following Fokker-Planck equations.
    \begin{align}
        \frac{\partial p_t(\theta)}{\partial t} = \sig^2 \Delta p_t(\theta),\label{eqn:fokker1}\\
        \frac{\partial p_t'(\theta)}{\partial t} = \sig^2 \Delta p_t'(\theta)\label{eqn:fokker2}
    \end{align}
    We denote $E_{\alpha}(p_t(\theta)\lVert p'_t(\theta)) = \int \frac{p_t(\theta)^\alpha}{p_t'(\theta)^{\alpha-1}}d\theta$ to be the moment of the likelihood ratio function. Then by definition of R\'enyi divergence, we prove
    \begin{align}
        R_\alpha(p_t(\theta)\lVert p_t'(\theta)) = \frac{1}{\alpha - 1}\log E_\alpha(p_t(\theta)\lVert p'_t(\theta))
    \end{align}
    Therefore we compute the rate of R\'enyi divergence with regard to $t$ as follows.
    \begin{align}
        \frac{\partial }{\partial t}R_\alpha(p_t(\theta)\lVert p_t'(\theta)) &= \frac{1}{\alpha - 1} \frac{\partial }{\partial t}\log E_\alpha(p_t(\theta)\lVert p_t'(\theta))\\
        & = \frac{1}{(\alpha - 1)E_\alpha(p_t(\theta)\lVert p_t'(\theta))} \cdot \frac{\partial }{\partial t} E_\alpha(p_t(\theta)\lVert p_t'(\theta))\\
        \text{By definition of $E_\alpha(p_t(\theta)\lVert p'_t(\theta))$,}\quad & = \frac{1}{(\alpha - 1)E_\alpha(p_t(\theta)\lVert p_t'(\theta))} \cdot \frac{\partial}{\partial t}\left(\int \frac{p_t(\theta)^\alpha}{p_t'(\theta)^{\alpha-1}}d\theta\right)
    \end{align}
    %
    By exchanging the order of derivative and integration, we prove
    \begin{align}
        \frac{\partial }{\partial t} &R_\alpha(p_t(\theta)\lVert p_t'(\theta)) = \frac{1}{(\alpha - 1)E_\alpha(p_t(\theta)\lVert p_t'(\theta))}  \cdot \int \frac{\partial}{\partial t}\frac{p_t(\theta)^\alpha}{p_t'(\theta)^{\alpha-1}}d\theta \\
        & = \frac{1}{(\alpha - 1)E_\alpha(p_t(\theta)\lVert p_t'(\theta))}  \cdot \int \left(\alpha \cdot \frac{p_t(\theta)^{\alpha-1}}{p_t'(\theta)^{\alpha-1}}\cdot \frac{\partial p_t(\theta)}{\partial t} - (\alpha - 1)\cdot \frac{p_t(\theta)^\alpha}{p_t'(\theta)^{\alpha}}\cdot \frac{\partial p_t'(\theta)}{\partial t}\right)d\theta
    \end{align}
    By the Fokker-Planck equations \cref{eqn:fokker1} and \cref{eqn:fokker2}, we substitute the terms $\frac{\partial p_t(\theta)}{\partial t}$ and $\frac{\partial p_t'(\theta)}{\partial t}$ in the above equation as follows.
    \begin{align}
        & \frac{\partial }{\partial t} R_\alpha(p_t(\theta)\lVert p_t'(\theta)) \\
        & = \frac{1}{(\alpha - 1)E_\alpha(p_t(\theta)\lVert p_t'(\theta))}  \cdot \int \left(\alpha \sig^2 \cdot \frac{p_t(\theta)^{\alpha-1}}{p_t'(\theta)^{\alpha-1}}\cdot \Delta p_t(\theta) - (\alpha - 1)\sig^2\cdot \frac{p_t(\theta)^\alpha}{p_t'(\theta)^{\alpha}}\cdot \Delta p_t'(\theta)\right)d\theta 
        \label{eqn:substitute_fokker}
    \end{align}
    By applying Green's first identity in \cref{eqn:substitute_fokker}, the first intergration term in \cref{eqn:substitute_fokker} is changed to
    \begin{align}
        \int & \alpha \sig^2 \cdot \frac{p_t(\theta)^{\alpha-1}}{p_t'(\theta)^{\alpha-1}}\cdot \Delta p_t(\theta)d\theta = \lim_{r\rightarrow \infty}\int_{B_r} \alpha \sig^2 \cdot \frac{p_t(\theta)^{\alpha-1}}{p_t'(\theta)^{\alpha-1}}\cdot \Delta p_t(\theta)d\theta\\
        & = \lim_{r\rightarrow\infty}\int_{\partial B_r} \alpha \sig^2 \cdot \frac{p_t(\theta)^{\alpha-1}}{p_t'(\theta)^{\alpha-1}} \cdot \nabla p_t(\theta) \cdot d \textbf{S} - \int \nabla\left(\alpha \sig^2 \cdot \frac{p_t(\theta)^{\alpha-1}}{p_t'(\theta)^{\alpha-1}}\right) \cdot \nabla p_t(\theta) d\theta,\label{eqn:greenidlimit}
    \end{align}
    where $B_r$ is the unit ball centered around origin in $d$-dimensional Euclidean space with radius $r$. The limits in the first term of \cref{eqn:greenidlimit} becomes zero given the smoothness and fast decay properties of $p_t(\theta)$, and the Lebesgue integrability of $\frac{p_t(\theta)^{\alpha - 1}}{p_t'(\theta)^{\alpha - 1}}\cdot p_t(\theta)$. Therefore we prove that 
    \begin{align}
        \int & \alpha \sig^2 \cdot \frac{p_t(\theta)^{\alpha-1}}{p_t'(\theta)^{\alpha-1}}\cdot \Delta p_t(\theta)d\theta = - \int \nabla\left(\alpha \sig^2 \cdot \frac{p_t(\theta)^{\alpha-1}}{p_t'(\theta)^{\alpha-1}}\right) \cdot \nabla p_t(\theta) d\theta,\label{eqn:greenidfirst}
    \end{align}

    Similarly by applying Green's first identity in \cref{eqn:substitute_fokker}, the second intergration term in \cref{eqn:substitute_fokker} is changed to

    \begin{align}
        \int & - (\alpha-1) \sig^2 \cdot \frac{p_t(\theta)^{\alpha}}{p_t'(\theta)^{\alpha}}\cdot \Delta p_t'(\theta)d\theta = \int \nabla\left((\alpha-1) \sig^2 \cdot \frac{p_t(\theta)^{\alpha}}{p_t'(\theta)^{\alpha}}\right) \cdot \nabla p_t'(\theta) d\theta,\label{eqn:greenidsecond}
    \end{align}

    (This techinque for using Green's identity and bounding the rate of entropy change with relative Fisher information-like quantity, has been previously used for KL divergence in \citet{lyu2009interpretation}, and R\'enyi divergence in \citet{vempala2019rapid,chourasia2021differential}. The result is also closely related to the well known \emph{de Bruijn's indentity} in information theory literature. )

    By plugging in \cref{eqn:greenidfirst} and \cref{eqn:greenidsecond} into \cref{eqn:substitute_fokker}, we prove that
    \begin{align}
        & \frac{\partial }{\partial t}R_\alpha(p_t(\theta)\lVert p_t'(\theta)) \\
        & = \frac{\sig^2}{(\alpha - 1)E_\alpha(p_t(\theta)\lVert p_t'(\theta))}  \cdot \int - \alpha \cdot \left\langle\nabla\left(\frac{p_t(\theta)^{\alpha-1}}{p_t'(\theta)^{\alpha-1}}\right), \nabla p_t(\theta)\right\rangle + (\alpha - 1)\cdot \left\langle\nabla\left(\frac{p_t(\theta)^\alpha}{p_t'(\theta)^{\alpha}}\right), \nabla p_t'(\theta)\right\rangle d\theta\\
        & = \frac{\alpha(\alpha-1)\sig^2}{(\alpha - 1)E_\alpha(p_t(\theta)\lVert p_t'(\theta))}  \cdot \int \frac{p_t(\theta)^{\alpha-2}}{p_t'(\theta)^{\alpha-2}}\cdot \left\langle\nabla\left(\frac{p_t(\theta)}{p_t'(\theta)}\right),  - \nabla p_t(\theta) + \frac{p_t(\theta)}{p_t'(\theta)}\nabla p_t'(\theta)\right\rangle d\theta\\
        & = - \frac{\alpha\sig^2}{E_\alpha(p_t(\theta)\lVert p_t'(\theta))}  \cdot \int \frac{p_t(\theta)^{\alpha-2}}{p_t'(\theta)^{\alpha-2}}\cdot \left\langle\nabla\left(\frac{p_t(\theta)}{p_t'(\theta)}\right),  \nabla \left(\frac{p_t(\theta)}{p_t'(\theta)}\right) \right\rangle \cdot p_t'(\theta) d\theta\\
        & = - \alpha\sig^2\cdot \frac{I_\alpha(p_t(\theta)\lVert p_t'(\theta))}{E_\alpha(p_t(\theta)\lVert p_t'(\theta))},\label{eqn:rate_Renyi_I_t_by_E_t}
    \end{align}

    where $I_\alpha(p_t(\theta)\lVert p_t'(\theta)) = \int \frac{p_t(\theta)^{\alpha-2}}{p_t'(\theta)^{\alpha-2}}\cdot \left\langle\nabla\left(\frac{p_t(\theta)}{p_t'(\theta)}\right),  \nabla \left(\frac{p_t(\theta)}{p_t'(\theta)}\right) \right\rangle \cdot p_t'(\theta) d\theta$. Therefore, if $p_t(\theta)$ and $p_t'(\theta)$ satisfy log-Sobolev inequality with constant $c_t$, then by \cref{lem:i_by_e_lower}, we obtain the following inequality.
    \begin{align}
        \label{eqn:I_tbyE_t}
        \frac{I_\alpha(p_t(\theta)\lVert p_t'(\theta))}{E_\alpha(p_t(\theta)\lVert p_t'(\theta))}\geq \frac{2 c_t}{\alpha^2} R_\alpha(p_t(\theta)\lVert p_t'(\theta)) + \frac{2 c_t}{\alpha^2}\cdot \alpha (\alpha - 1)\frac{\partial}{\partial \alpha}R_\alpha(p_t(\theta) \lVert p_t'(\theta))
    \end{align}

    Meanwhile, by Lemma~16 in \citet{vempala2019rapid} and Lemma~17 in \citet{vempala2019rapid}, the distributions (with densities $p_t(\theta)$ and $p_t'(\theta)$) indeed satisfy $c_t$-log-Sobolev inequality with 
    \begin{align}
        \label{eqn:lsi_c_t}
        c_t = \left(\frac{L^2}{c} + 2t\sig^2\right)^{-1} 
    \end{align}

    By plugging \cref{eqn:lsi_c_t} and \cref{eqn:I_tbyE_t} into \cref{eqn:rate_Renyi_I_t_by_E_t}, we prove the following bound for the rate of R\'enyi divergence.
    
    \begin{align}
        \frac{\partial }{\partial t}R_\alpha(p_t(\theta)\lVert p_t'(\theta)) & \leq - \alpha\sig^2\cdot \left(\frac{2 c_t}{\alpha^2} R_\alpha(p_t(\theta)\lVert p_t'(\theta)) + \frac{2 c_t}{\alpha^2}\cdot\alpha(\alpha-1)\cdot\frac{\partial}{\partial \alpha}R_\alpha(p_t(\theta) \lVert p_t'(\theta))\right)\\
        & = - c_t\cdot 2\sig^2\cdot \left(\frac{R_\alpha(p_t(\theta)\lVert p_t'(\theta))}{\alpha} + (\alpha-1)\cdot \frac{\partial}{\partial\alpha} R_\alpha(p_t(\theta)\lVert p_t'(\theta)) \right)
    \end{align}

    where $c_t = \left(\frac{L^2}{c} + 2t\sig^2\right)^{-1}$ is the log-Sobolev inequality constant for distributions $p_t(\theta)$ and $p_t'(\theta)$.
\end{proof}

\subsection{Proof for \cref{lem:recursive}}

\label{appsub:proof_recursive}

\paragraphbe{Reduce analysis to point initialization:} Without loss of generality, in this paper, we only analyze recursive R\'enyi DP bounds for \cref{alg:noisymBGD} under an arbitrary point initialization for initial parameters $\theta_0^0$. This is because, under an arbitrary initialization distribution, the last-iterate parameters $\theta_K^0$ in~\cref{alg:noisymBGD} follow a mixture distribution, with each component being the conditional output distribution given fixed initial parameters $\theta_0^0$. Therefore, by the quasi-convexity of R\'enyi divergence, the largest (worst-case) R\'enyi DP bound for $R_{\alpha}(p(\theta_\K^0|\theta_0^0)\lVert p({\theta'}_\K^0|\theta_0^0))$ over all possible initial parameters $\theta_0^0$, is also an upper bound for the R\'enyi privacy loss $R_{\alpha}(\theta_\K^0\lVert{\theta'}_\K^0)$ between (mixture) last-iterate parameters distributions for running~\cref{alg:noisymBGD} on two neighboring datasets. 

We now proceed to prove the recursive privacy bound~\cref{lem:recursive}.

\begin{replemma}{lem:recursive}
    Let $D,D'$ be an arbitrary pair of neighboring datasets that differ in the $i_0$-th data point (i.e. $x_{i_0}\neq x'_{i_0}$). Let $B_k^j$ be a fixed mini-batch used (in iteration $j$ of epoch $k$) in~\cref{alg:noisymBGD}, which contains $b$ indices sampled from $\{1,\cdots,n\}$. We denote $\theta_\kk^j$ and ${\theta'}_\kk^{j}$ as the intermediate parameters in \cref{alg:noisymBGD} on input datasets $D$ and $D'$, respectively. If the distributions of $\theta_k^j$ and ${\theta'}_k^j$ satisfy log-Sobolev inequality with a constant $c$, and if the mini-batch GD mapping $f(\theta) = \theta - \eta \cdot \frac{1}{b}\cdot \sum_{i\in B_k^j}\ell(\theta;\x_i)$ is $L$-Lipschitz for parameters $\theta$, then the following recursive bound for R\'enyi divergence holds.
    \begin{align}
        \frac{R_{\alpha}(\theta_\kk^{j+1}\lVert {\theta'}_\kk^{j+1})}{\alpha} & \leq
        \begin{cases}
            \frac{R_{\alpha'}(\theta_\kk^{j}\lVert {\theta'}_\kk^{j})}{\alpha'}\cdot  \left(1 + \frac{ c \cdot 2 \eta  \sig^2}{L^2}\right)^{-1} & \text{if $i_0\notin B_k^j$}\\
            \frac{R_{\alpha}(\theta_\kk^{j}\lVert {\theta'}_\kk^{j})}{\alpha} + \frac{\eta S_g^2}{4\sig^2b^2}& \text{if $i_0\in B_k^j$}
        \end{cases}\text{ with }\alpha' = \frac{\alpha-1}{1 + \frac{c\cdot 2 \eta  \sig^2}{L^2}} + 1.
    \end{align}
\end{replemma}
\begin{proof}
    \begin{enumerate}
        \item When $i_0\notin B_k^j$, the noisy mini-batch gradient descent mapping under both dataset $D$ and $D'$ is written in the same way as $f(\theta) + \mathcal{N}(0,2\eta \sig^2\mathbb{I}_d)$, where $f(\theta) = \theta - \frac{\eta}{b}\cdot \sum_{i\in B_k^j}\nabla\ell(\theta;\x_i)$ (this is because $x_i = x'_i$ for $i\in B_k^j$, when $i_0\notin B_k^j$). Therefore, we could use \cref{lem:post_proc} and solve \cref{eqn:post_proc} on the interval $t\in[0,\eta]$ to obtain the recursive privacy bound, where the R\'enyi divergence at $t=0$ (i.e. before the update in iteration $j$ of epoch $k$ takes place) satisfies $R_\alpha(p_0(\theta)\lVert p_0'(\theta)) \leq \eps_{\kk}^j(\alpha)$. We denote function $R(\alpha, t) = R_\alpha(p_t(\theta)\lVert p'_t(\theta))$. Then \cref{eqn:post_proc} is equivalent to the following equation.
        \begin{align}
            \begin{cases}
                \frac{\partial}{\partial t}R(\alpha, t)  \leq - \left(\frac{L^2}{c} + 2 t \sig^2\right)^{-1}\cdot 2\sig^2\cdot \left(\frac{R(\alpha, t)}{\alpha} + (\alpha-1)\cdot \frac{\partial}{\partial\alpha} R(\alpha, t) \right),\\
                R(\alpha, 0)  \leq \eps_\kk^{j}(\alpha)
            \end{cases}
            \label{eqn:pde1}
        \end{align}
    
        By substituting $u (\alpha, t) = \frac{R(\alpha, t)}{\alpha}$ into \cref{eqn:pde1}, we prove that \cref{eqn:pde1} is equivalent to the following equation.
        \begin{align}
            \begin{cases}
                \frac{\partial}{\partial t}u(\alpha, t)  \leq - \left(\frac{L^2}{c} + 2 t \sig^2\right)^{-1}\cdot 2\sig^2\cdot \left( u(\alpha, t) + (\alpha-1)\cdot \frac{\partial}{\partial\alpha} u(\alpha, t) \right),\\
                u(\alpha, 0)  \leq \frac{\eps_\kk^{j}(\alpha)}{\alpha}
            \end{cases}
            \label{eqn:pde2}
        \end{align}
    
        By change of variable $y = \ln(\alpha -1)$, we prove that \cref{eqn:pde2} is equivalent to the following equation.
        \begin{align}
            \begin{cases}
                \frac{\partial}{\partial t}u(y, t)  \leq - \left(\frac{L^2}{c} + 2 t \sig^2\right)^{-1}\cdot 2\sig^2\cdot  \left(u(y, t) + \frac{\partial }{\partial y}u(y,t)\right),\\
                u(y, 0)  \leq \frac{\eps_\kk^{j}(e^{y} + 1)}{e^{y} + 1}
            \end{cases}
            \label{eqn:pde3}
        \end{align}
    
        Now we do change of variable $\begin{cases}
            \tau = t\\
            z = y - \int_{0}^t\left(\frac{L^2}{c} + 2 t' \sig^2\right)^{-1}\cdot 2\sig^2 dt' + \ln\left(1 + \frac{c\cdot 2 \eta \sig^2}{L^2}\right)
        \end{cases}$ then by chain rule, we prove the following expressions for the partial derivatives using new variables.
        \begin{align}
            \begin{cases}
                \frac{\partial u}{\partial t} = \frac{\partial u}{\partial \tau}\cdot \frac{\partial \tau}{\partial t} +  \frac{\partial u}{\partial z}\cdot \frac{\partial z}{\partial t} = \frac{\partial u}{\partial \tau} -  \frac{\partial u}{\partial z}\cdot \left(\frac{L^2}{c} + 2 t\sig^2\right)^{-1}\cdot 2\sig^2\\
            \frac{\partial u}{\partial y} = \frac{\partial u}{\partial \tau}\cdot \frac{\partial \tau}{\partial y} +  \frac{\partial u}{\partial z}\cdot \frac{\partial z}{\partial y} = \frac{\partial u}{\partial z}\\
            \end{cases}
            \label{eqn:change_var_last}
        \end{align}
    
        By plugging \cref{eqn:change_var_last} into \cref{eqn:pde3}, we prove that the partial differential inequality \ref{eqn:pde3} is equivalent to the following inequality under new variables $\tau$ and $z$.
    
        \begin{align}
            \begin{cases}
                \frac{\partial}{\partial \tau}u(z, \tau)  \leq - \left(\frac{L^2}{c} + 2 \tau \sig^2\right)^{-1}\cdot 2\sig^2\cdot  u(z, \tau),\\
                u(z, 0)  \leq \frac{\eps_\kk^{j}\left(e^{z -\ln\left(1 + \frac{c\cdot 2 \eta \sig^2}{L^2}\right)} + 1\right)}{e^{z - \ln\left(1 + \frac{c\cdot 2 \eta \sig^2}{L^2}\right)} + 1}
            \end{cases}
            \label{eqn:pde4}
        \end{align}
    
        Now we observe that given any fixed $z$, \cref{eqn:pde4} is an ordinary differential equation with regard to $\tau$, with a decay term that proportional to $- \left(\frac{L^2}{c} + 2 \tau \sig^2\right)^{-1}\cdot 2\sig^2\cdot  u(z, \tau)$. Therefore, we directly solve \cref{eqn:pde4}, and prove that
        \begin{align}
            \ln(u(z,\tau)) - \ln(u(z, 0))\leq  - \int_0^\tau \left(\frac{L^2}{c} + 2 \tau' \sig^2\right)^{-1}\cdot 2\sig^2 d\tau' = - \ln\left(1 + \frac{c\cdot 2 \tau \sig^2}{L^2}\right)
            \label{eqn:sol_uztau}
        \end{align}
        We take $\tau = \eta$ in \cref{eqn:sol_uztau}, then we prove
        \begin{align}
            \ln u(z, \eta) - \ln u(z, 0)\leq - \ln\left(1 + \frac{c\cdot 2 \eta \sig^2}{L^2}\right)
            \label{eqn:sol_uztau_eta}
        \end{align}

        By plugging the initial condition for $u(z,0)$ in \cref{eqn:pde4} into \cref{eqn:sol_uztau_eta}, we prove that the solution for $u$ at the end of a set $\tau = \eta$ satisfies the following inequality.
        \begin{align}
            u(z,\eta)\leq \frac{\eps_\kk^{j}\left(e^{z -\ln\left(1 + \frac{c\cdot 2 \eta \sig^2}{L^2}\right)} + 1\right)}{e^{z - \ln\left(1 + \frac{c\cdot 2 \eta \sig^2}{L^2}\right)} + 1} \cdot \left(1 + \frac{c\cdot 2 \eta \sig^2}{L^2}\right)^{-1}
        \end{align}

        Now we translate the variables $z=z, \tau = \eta$ back to the old variables $y$ and $t$ by definitions, and prove that $t=\tau = \eta$ and $z = y - \int_{0}^t\left(\frac{L^2}{c} + 2 t' \sig^2\right)^{-1}\cdot 2\sig^2 dt' + \ln\left(1 + \frac{c\cdot 2 \eta \sig^2}{L^2}\right) = y$. Therefore, under these variable substitutions, we prove that \cref{eqn:sol_uztau_eta} is equivalent to the following equation.
        \begin{align}
            u(y,\eta)\leq \frac{\eps_\kk^{j}\left(e^{y -\ln\left(1 + \frac{c\cdot 2 \eta \sig^2}{L^2}\right)} + 1\right)}{e^{y - \ln\left(1 + \frac{c\cdot 2 \eta \sig^2}{L^2}\right)} + 1} \cdot \left(1 + \frac{c\cdot 2 \eta \sig^2}{L^2}\right)^{-1}
            \label{eqn:sol_uyt}
        \end{align}
    
        Finally, we translate the variable $y$ back to $\alpha$, under the fixed variable $t=\eta$, by the definition $y=\log(\alpha-1)$. Therefore, we prove that \cref{eqn:sol_uyt} is equivalent to the following solution for $u(\alpha, \eta)$.
    
        \begin{align}
            u(\alpha, \eta)\leq  \frac{\eps_\kk^{j}\left((\alpha - 1)\cdot\left(1 + \frac{c\cdot 2 \eta \sig^2}{L^2}\right)^{-1} + 1\right)}{(\alpha - 1)\cdot\left(1 + \frac{c\cdot 2 \eta \sig^2}{L^2}\right)^{-1} + 1} \cdot \left(1 + \frac{c\cdot 2 \eta \sig^2}{L^2}\right)^{-1}
            \label{eqn:sol_ualphat}
        \end{align}
    
        By the definition $u(\alpha,t) = \frac{R(\alpha, t)}{\alpha}$, we prove that \cref{eqn:sol_ualphat} is equivalent to the following solution for $R(\alpha, t)$.
        \begin{align}
            \frac{R(\alpha, \eta)}{\alpha} & \leq \frac{\eps_\kk^{j}\left((\alpha - 1)\cdot\left(1 + \frac{c\cdot 2 \eta \sig^2}{L^2}\right)^{-1} + 1\right)}{(\alpha - 1)\cdot\left(1 + \frac{c\cdot 2 \eta \sig^2}{L^2}\right)^{-1} + 1} \cdot \left(1 + \frac{c\cdot 2 \eta \sig^2}{L^2}\right)^{-1}
        \end{align}
    
        Therefore, by using the definition $R(\alpha, \eta)=\eps_k^{j+1}(\alpha)$, we finish the proof for the Lemma statement, that when $i_0\neq B_k^j$,
        \begin{align}
            \frac{R_{\alpha}(\theta_k^{j+1}\lVert {\theta'}_k^{j+1})}{\alpha}\leq \frac{R_{\alpha'}(\theta_k^{j}\lVert {\theta'}_k^{j})}{\alpha'}\cdot \left(1 + \frac{c\cdot 2 \eta \sig^2}{L^2}\right)^{-1}
        \end{align}

        where $\alpha'=(\alpha - 1)\cdot\left(1 + \frac{c\cdot 2 \eta \sig^2}{L^2}\right)^{-1} + 1$.
        
        \item When $i_0\in B_k^j$, by composition theorem for R\'enyi differential privacy~\cite{mironov2017renyi}, and by the R\'enyi privacy bound for Gaussian mechanism~\cite{mironov2017renyi} under $\ell_2$-sensitivity $S_g/b$ (for batch averaged gradient) and noise $\mathcal{N}(0,2\eta\sig^2\mathbb{I}_d)$, we prove that
        \begin{align}
            R_{\alpha}(\theta_k^{j+1}\lVert {\theta'}_k^{j+1})& \leq R_{\alpha}(\theta_k^j,\theta_k^{j+1}\lVert {\theta'}_k^j, {\theta'}_k^{j+1})\\
            &\leq R_{\alpha}(\theta_k^{j}\lVert {\theta'}_k^{j}) + \frac{\eta^2 (S_g/b)^2}{2\cdot 2\eta \sig^2} = R_{\alpha}(\theta_k^{j}\lVert {\theta'}_k^{j}) + \frac{\alpha \eta S_g^2}{4\sig^2 b^2}.
        \end{align}
    \end{enumerate}
\end{proof}

\subsection{Proof for LSI sequence for noisy mini-batch gradient descent}

\label{appsub:proof_lsi}

To apply the recurisve privacy bound \cref{lem:recursive} and prove a converging privacy dynamics, we first need to prove that the distributions of parameters $\theta_k^j$ in \cref{alg:noisymBGD} satisfy log-Sobolev inequality with certain constant $c_k^j$, that depends on $k$ and $j$. In this section, we prove that noisy mini-batch gradient descent on convex smooth loss function, as well as strongly convex smooth loss functions,satusfies LSI with certain sequences of constants as follows.

\begin{lemma}[LSI constant sequence in \cref{alg:noisymBGD} for convex smooth loss]
\label{lem:LSIseq_convex}
Suppose that the loss function $\ell(\theta;\x)$ in \cref{alg:noisymBGD} is convex and $\beta$-smooth. If the step-size $\eta<\frac{2}{\beta}$, then for any $k=0,\cdots,\K-1$ and $j=0,\cdots,n/b-1$, the distribution of parameters $\theta_k^j$ in \cref{alg:noisymBGD} satisfies $c_{k}^j$-log-Sobolev inequality with
\begin{equation}
    \label{eqn:LSIseq_convex}
    c_k^j=\frac{1}{2 \eta \sig^2\cdot \left(k\cdot n/b + j\right)},
\end{equation}
and we define $c_0^0 = \frac{1}{0} = + \infty$.
\end{lemma}
\begin{proof}
    The mini-batch noisy gradient descent update could be written as $\theta_{\kk}^{j+1}=f(\theta_{\kk}^j) + \mathcal{N}(0,2\eta\sig^2\mathbb{I}_d)$, where $f$ is a deterministic mapping on $\mathbb{R}^d$ written as $f(\theta) = \theta - \eta \cdot \frac{\sum_{i\in B}\ell(\theta;\x_i)}{b}$, and $B$ is a mini-batch of size $b$ consisting of indices selected from $0,\cdots,n$.

    Because the initialization is point distribution around $\theta_0$, therefore $\theta_0^0$ satisfies log-Sobolev inequality with constant $c_0=\infty$.

    By the $\beta$-smoothness of $\ell(\theta;\x)$, and by $\eta<\frac{2}{\beta}$, we prove that the mini-batch gradient mapping $f(\theta) $ is $1$-Lipschitz. Further using LSI under Lipchitz mapping (Lemma 16 in \citet{vempala2019rapid}) and under Gaussian convolution (Lemma 17 in \citet{vempala2019rapid}), we prove that
    \begin{align}
        \frac{1}{c_k^{j}} & =\frac{1}{c_k^{j-1}}+2 \eta \sigma^2\\
        & = \frac{1}{c_0^0} + 2\eta \sig^2\cdot (k\cdot n/b + j)\\
        \text{by $c_0^0=+\infty$,}\quad& = 2\eta \sig^2\cdot (k\cdot n/b + j)\\
    \end{align}
    This suffices to prove the LSI sequence in \cref{eqn:LSIseq_convex}.
\end{proof}

Similarly, we prove another LSI sequence for noisy mini-batch gradient descent on strongly convex smooth loss function.

\begin{lemma}[LSI constant sequence in \cref{alg:noisymBGD} for strongly convex smooth loss]
    \label{lem:LSIseq_stronglyconvex}
Suppose the loss function $\ell(\theta;\x)$ in \cref{alg:noisymBGD} is $\lambda$-strongly convex and $\beta$-smooth. If the step-size $\eta<\frac{2}{\lambda + \beta}$, then for any $k=0,\cdots,\K-1$ and $j=0,\cdots,n/b-1$, the distribution of parameters $\theta_k^j$ in \cref{alg:noisymBGD} satisfies $c_{k}^j$-log-Sobolev inequality with
\begin{equation}
    \label{eqn:LSIseq_stronglyconvex}
    c_k^j=\frac{1}{2 \eta \sig^2} \cdot \frac{1}{\sum_{s = 0}^{k\cdot n/b + j - 1} (1-\eta\lambda)^{2s}}
\end{equation}
\end{lemma}

\begin{proof}
    The mini-batch noisy gradient descent update could be written as $\theta_{\kk}^{j+1}=f(\theta_{\kk}^j) + \mathcal{N}(0,2\eta\sig^2\mathbb{I}_d)$, where $f$ is a deterministic mapping on $\mathbb{R}^d$ written as $f(\theta) = \theta - \eta \cdot \frac{\sum_{i\in B}\ell(\theta;\x_i)}{b}$, and $B$ is a mini-batch of size $b$ consisting of indices selected from $0,\cdots,n$.

    Because the initialization is point distribution around $\theta_0$, therefore $\theta_0^0$ satisfies log-Sobolev inequality with constant $c_0=\infty$.

    By $\lambda$-strong convexity and $\beta$-smoothness of $\ell(\theta;\x)$, and by $\eta<\frac{2}{\lambda + \beta}$, we prove that the mini-batch gradient mapping $f(\theta) $ is $1 - \eta\lambda$-Lipschitz. Further using LSI under Lipchitz mapping (Lemma 16 in \citet{vempala2019rapid}) and under Gaussian convolution (Lemma 17 in \citet{vempala2019rapid}), we prove that
    \begin{align}
        \frac{1}{c_k^{j}} & =\frac{(1-\eta\lambda)^2}{c_k^{j-1}}+2 \eta \sigma^2\\
        & = \frac{(1-\eta\lambda)^{2(k\cdot n/b + j)}}{c_0^0} + 2\eta \sig^2\cdot \sum_{s = 0}^{k\cdot n/b + j - 1} (1-\eta\lambda)^{2s}\\
        \text{by $c_0^0=+\infty$,}\quad & = 2 \eta \sig^2 \cdot \sum_{s = 0}^{k\cdot n/b + j - 1} (1-\eta\lambda)^{2s}
    \end{align}
    This suffices to prove the LSI sequence in \cref{eqn:LSIseq_stronglyconvex}.
\end{proof}

\subsection{Equivalence Between~\cref{lem:recursive} (Ours) And The Bound in~\citet{feldman2018privacy}} 

\label{appsub:proof_convexsmooth}

We now plug the LSI constant sequence derived in \cref{lem:LSIseq_convex} into the recursive privacy bound~\cref{lem:recursive}, and prove the following privacy dynamics theorems for noisy mini-batch gradient descent under smooth convex loss functions.

\begin{theorem}[Privacy dynamics under convex smooth loss]
    \label{thm:convexsmooth}
    Under fixed mini-batches $B^0,\cdots,B^{n/b-1}$, if the loss function $\ell(\theta;x)$ is convex, $\beta$-smooth, and if its gradient has finite $\ell_2$-sensitivity $S_g$, then \cref{alg:noisymBGD} with step-size $\eta <\frac{2}{\beta}$ satisfies $(\alpha,\eps)$-R\'enyi DP for data points in the batch $B^{j_0}$, with
    \begin{align}
        \label{eqn:convexsmooth}
        \eps_j & \leq \frac{\alpha \eta S_g^2}{4 \sig^2b^2}\cdot \frac{b}{n} \cdot (\K - 1) + \frac{\alpha \eta S_g^2}{4 \sig^2b^2}\cdot \frac{1}{n/b-j_0}\\
    \end{align}
\end{theorem}

\begin{proof}
    We first observe that for any batch index $j_0 = 0,\cdots, n/b - 1$, the privacy bound for data points in batch $B^{j_0}$ in \cref{alg:noisymBGD} (that has $K$ epochs) is equivalent to the privacy bound for data points in the batch $B^0$ after $\K-1$ epochs and $n/b - j_0$ iterations in \cref{alg:noisymBGD}. That is, if $R_\alpha(\theta_k^{j}\lVert {\theta'}_k^{j})\leq \eps_\kk^{j}(\alpha)$ is an upper bound for the R\'enyi divergence between distributions of parameters $\theta_k^j$ and ${\theta'}_k^j$ in running \cref{alg:noisymBGD} on neighboring datasets $D$ and $D'$, when the differing data point between $D$ and $D'$ is contained in the mini-batch $B^0$, then the privacy bound $\eps$ for data points in the mini-batch $B^{j_0}$ of \cref{alg:noisymBGD} satisfies 
    \begin{align}
        \eps & \leq \eps_{\K - 1}^{n/b - j_0}(\alpha)  \leq \eps_{\K - 1}^{0}(\alpha) + \eps_0^{n/b - j_0}(\alpha)\\
        & \leq (\K-1)\cdot \eps^{n/b}_0(\alpha) + \eps_0^{n/b-j_0}(\alpha),
        \label{eqn:proofconvexsmooth_translate_j_0}
    \end{align} 

    where the last two inequalities are by composition of R\'enyi DP guarantees.

    Therefore, in the remaining proof, we only prove upper bounds for the terms $\eps_0^{j}(\alpha)$ for $j=1,\cdots,n/b$, that are required by \cref{eqn:proofconvexsmooth_translate_j_0} for bounding $\eps$.

    By \cref{lem:LSIseq_convex}, for any $k=0,\cdots, \K-1$ and $j = 0,\cdots, n/b-1$, the distribution of parameters $\theta_k^j$ satisfies log-Sobolev inequality with the following constant $c_k^j$.
    \begin{align}
        c_k^j=\frac{1}{2 \eta \sig^2\cdot \left(k\cdot n/b + j\right)}
    \end{align}

    By plugging the LSI constant sequence $\{c_k^j\}$ into \cref{lem:recursive}, and by the noisy mini-batch gradient descent mapping, under convex smooth loss with stepsize $\eta<\frac{2}{\beta}$, is $1$-Lipschitz, we prove the following recursive bound for $\eps_k^{j}(\alpha)$, i.e. the R\'enyi DP bound for data points in the mini-batch $B^0$ of \cref{alg:noisymBGD}. For any $k=0, \cdots, \K - 1$ and any $j=0, \cdots, n/b - 1$,

    \begin{align}
        &\eps_{0}^{0}(\alpha) = 0\\
        &\frac{\eps_k^{j+1}(\alpha)}{\alpha} \leq \begin{cases}
            \frac{\eps_{k}^{j}(\alpha)}{\alpha} + \frac{\eta S_g^2}{4\sig^2 b^2}& \text{ if $j=0$}\\
            \frac{\eps_{k}^{j}(\alpha')}{\alpha'}\cdot \frac{k\cdot n/b + j}{k\cdot n/b + j + 1} & \text{ if $j = 1, \cdots, n/b - 1$}\\
        \end{cases}
        \label{eqn:proof_convex}\\
        &\text{where $\alpha'=(\alpha - 1)\cdot\left(1 + \frac{c\cdot 2 \eta \sig^2}{L^2}\right)^{-1} + 1$}\\
        &\eps_{k+1}^{0}(\alpha) = \eps_{k}^{n/b}(\alpha)
    \end{align}

    By solving \cref{eqn:proof_convex} under $k=0$, we prove that for any $j = 1, \cdots, n/b$, and for any $\alpha>1$,
    \begin{align}
        \frac{\eps_0^{j}(\alpha)}{\alpha} & \leq \frac{\eta S_g^2}{4\sig^2b^2}\cdot \prod_{j' = 1}^{j - 1}\frac{j'}{j' + 1}\\
        & = \frac{\eta S_g^2}{4\sig^2b^2}\cdot \frac{1}{j}
        \label{eqn:proof_convex_first_epoch}
    \end{align}

    By plugging \cref{eqn:proof_convex_first_epoch} into \cref{eqn:proofconvexsmooth_translate_j_0}, we prove the privacy bound \cref{eqn:convexsmooth} in the theorem statement.
\end{proof}

\cref{thm:convexsmooth} is equivalent to Theorem 23 in \citet{feldman2018privacy} for  \cref{alg:noisymBGD} under single-epoch setting (where $K=1$) with batch-size $b=1$. For multi-epoch setting with $K=n$ and batch-size $b=1$, \cref{thm:convexsmooth} is equivalent to Theorem 35 in \citet{feldman2018privacy}.

\subsection{Proof for \cref{thm:strconvexsmooth}}

\label{appsub:proof_strconvexsmooth}

We now plug the LSI constant sequence derived in \cref{lem:LSIseq_stronglyconvex} into the recursive privacy bound~\cref{lem:recursive}, and prove the following privacy dynamics theorems for noisy mini-batch gradient descent under strongly convex smooth loss functions.

\begin{reptheorem}{thm:strconvexsmooth}
    Conditioned on a fixed sequence of partitioned mini-batches $B^0,\cdots,B^{n/b-1}$ in \cref{step:batch_generation_shuffle}, 
    if the loss function is $\lambda$-strongly convex, ${\beta\text{-smooth}}$ and its gradient has $\ell_2$-sensitivity $S_g$, then running \cref{alg:noisymBGD} for $K\geq 1$ epochs with step-size $\eta <\frac{2}{\lambda + \beta}$, satisfies $(\alpha,\eps)$-R\'enyi DP for data points in the batch $B^{j_0}$, with
    \begin{align}
        \eps\leq
        \eps_0^{\lfloor \frac{n}{2b}\rfloor}(\alpha) \cdot \frac{1 - (1-\eta\lambda)^{2\cdot (\K - 1) \cdot (n/b - \lfloor \frac{n}{2b}\rfloor)}}{1 - (1-\eta\lambda)^{2\cdot (n/b - \lfloor \frac{n}{2b}\rfloor)} } + \eps_{0}^{n/b - j_0}(\alpha)
    \end{align}
    where $\eps_0^{j}(\alpha)= \frac{\alpha \eta S_g^2}{4\sig^2b^2}\cdot \left(1 - \eta \lambda \right)^{2\cdot (j - 1)} \cdot \frac{1}{\sum_{s=0}^{ j - 1 } (1-\eta\lambda)^{2s}}$ for any $j = 1, \cdots, \frac{n}{b}$ (we assume $\frac{n}{b}\geq 2$).

\end{reptheorem}

\begin{proof}
    We first observe that for any batch index $j_0 = 0,\cdots, n/b - 1$, the privacy bound for data points in batch $B^{j_0}$ in \cref{alg:noisymBGD} (that has $K$ epochs) is equivalent to the privacy bound for data points in the batch $B^0$ after $\K-1$ epochs and $n/b - j_0$ iterations in \cref{alg:noisymBGD}. That is, if $R_\alpha(\theta_k^{j}\lVert {\theta'}_k^{j})\leq \eps_\kk^{j}(\alpha)$ is an upper bound for the R\'enyi divergence between distributions of parameters $\theta_k^j$ and ${\theta'}_k^j$ in running \cref{alg:noisymBGD} on neighboring datasets $D$ and $D'$, when the differing data point between $D$ and $D'$ is contained in the mini-batch $B^0$, then the privacy bound $\eps$, for running \cref{alg:noisymBGD} on neighboring datasets $D$ and $D'$ that differs in a point in the mini-batch $B^{j_0}$, satisfies 
    \begin{align}
        \eps & \leq \eps_{\K - 1}^{n/b - j_0}(\alpha)  \leq \eps_{\K - 1}^{0}(\alpha) + \eps_0^{n/b - j_0}(\alpha),
        \label{eqn:proofstrconvexsmooth_translate_j_0}
    \end{align} 

    where the last inequality is by composition of R\'enyi DP guarantees.

    Therefore, in the remaining proof, we only prove upper bounds for the terms $\eps_0^{n/b - 1 - j_0}(\alpha)$ and $\eps_{\K - 1}^{0}(\alpha)$, that are required by \cref{eqn:proofstrconvexsmooth_translate_j_0} for bounding $\eps$.

    By \cref{lem:LSIseq_stronglyconvex}, for any $k=0,\cdots, \K-1$ and $j = 0,\cdots, n/b-1$, the distribution of parameters $\theta_k^j$ satisfies log-Sobolev inequality with the following constant $c_k^j$.
    \begin{align}
        c_k^j& = \frac{1}{2\eta \sig^2}\cdot \frac{1}{\sum_{s=0}^{k\cdot n/b + j - 1} (1-\eta\lambda)^{2s} }
    \end{align}

    By plugging the LSI constant sequence $\{c_k^j\}$ proved in~\cref{lem:LSIseq_stronglyconvex} into the proved recursive privacy bound~\cref{lem:recursive}, we prove that for running~noisy mini-batch gradient descent on neighboring datasets that differ in a data point in the first mini-batch $B^0$, the following recursion for the R\'enyi divergence bound $\eps_k^{j}(\alpha)$ holds: if the loss function is $\lambda$-strongly convex and $\beta$-smooth, and if the stepsize $\eta<\frac{2}{\lambda + \beta}$, then for any $k=0, \cdots, \K - 1$, any $j=0, \cdots, n/b - 1$, and any $\alpha>1$, 

    \begin{align}
        &\eps_{0}^{0}(\alpha) = 0\\
        &\frac{\eps_k^{j+1}(\alpha)}{\alpha} \leq \begin{cases}
            \frac{\eps_{k}^{j}(\alpha)}{\alpha} + \frac{\eta S_g^2}{4\sig^2 b^2}& \text{ if $j=0$}\\
            \frac{\eps_{k}^{j}(\alpha')}{\alpha'}\cdot (1-\eta\lambda)^2 \cdot \frac{\sum_{s=0}^{k\cdot n/b + j - 1} (1-\eta\lambda)^{2s}}{\sum_{s=0}^{k\cdot n/b + j } (1-\eta\lambda)^{2s}} & \text{ if $j = 1, \cdots, n/b - 1$}\\
        \end{cases}
        \label{eqn:proof_strongly_convex}\\
        &\text{where $\alpha'=(\alpha - 1)\cdot\left(1 + \frac{c\cdot 2 \eta \sig^2}{L^2}\right)^{-1} + 1$}\\
        &\eps_{k+1}^{0}(\alpha) = \eps_{k}^{n/b}(\alpha)
    \end{align}

    We now solve the above recursion, thus bounding the privacy loss $\eps$ of \cref{alg:noisymBGD} in~\cref{eqn:proofstrconvexsmooth_translate_j_0}.

    \begin{enumerate}
        \item We first prove a bound for the term $\eps_0^{n/b - j_0}(\alpha)$ in \cref{eqn:proofstrconvexsmooth_translate_j_0}. By solving \cref{eqn:proof_strongly_convex} under $k=0$, we prove that for any $j = 1, \cdots, n/b$, and for any $\alpha>1$,
        \begin{align}
            \frac{\eps_0^{j}(\alpha)}{\alpha} & \leq \frac{\eta S_g^2}{4\sig^2b^2}\cdot \prod_{j' = 1}^{j - 1}\left( \left(1 - \eta \lambda \right)^{2} \cdot \frac{\sum_{s=0}^{j' - 1} (1-\eta\lambda)^{2s}}{\sum_{s=0}^{ j' } (1-\eta\lambda)^{2s}}\right)\\
            & = \frac{\eta S_g^2}{4\sig^2b^2}\cdot \left(1 - \eta \lambda \right)^{2\cdot (j - 1)} \cdot \frac{1}{\sum_{s=0}^{ j - 1 } (1-\eta\lambda)^{2s}}
            \label{eqn:proof_strongly_convex_first_epoch}
        \end{align}
        \item We now prove a bound for the term $\eps_{\K - 1}^{0}(\alpha)$ in \cref{eqn:proofstrconvexsmooth_translate_j_0}.
        
        The term $\eps_k^j(\alpha)$ corresponds to the privacy bound for the data points in the batch $B^0$ for running \cref{alg:noisymBGD} with $k$ epochs plus $j$ iteration. This is equivalent to the composition $\mathcal{A}_2\circ\mathcal{A}_1$ of two sub-mechanisms $\mathcal{A}_1$ and $\mathcal{A}_2$, where $\mathcal{A}_1$ corresponds to running \cref{alg:noisymBGD} with $\K=k$ epochs, and $\mathcal{A}_2$ corresponds to running \cref{alg:noisymBGD} with only $j$ iterations. Therefore, by the composition theorem for R\'enyi DP guarantees~\cite{mironov2017renyi}, we prove the following alternative recursive privacy bound for $\eps_k^j(\alpha)$. For any $k=0,\cdots,\K - 1$ and any $j = 0, \cdots, n/b - 1$,
        \begin{align}
            \eps_k^{j + 1}(\alpha) \leq \eps_k^0(\alpha) + \eps_0^{j + 1}(\alpha)
            \label{eqn:proof_strongly_convex_composition}
        \end{align}
        We now prove an upper bound for $\eps_k^j(\alpha)$, by carefully combining the original recursion~\cref{eqn:proof_strongly_convex} and the new alternative recursion~\cref{eqn:proof_strongly_convex_composition} (obtained by composition theorem). We use the original recursion~\cref{eqn:proof_strongly_convex_composition} for recursively bounding $\eps_k^{j+1}(\alpha)$ during the first half of one epoch, i.e., for $j = 0, \cdots, \lfloor \frac{n}{2b}\rfloor - 1$, and then we use the new alternative recursive bound~\cref{eqn:proof_strongly_convex} for the second half of one epoch, i.e., for $j = \lfloor \frac{n}{2b}\rfloor, \cdots, n/b - 1$. \footnote{This seemingly artificial way of separating one epoch into two halves, and using the recursive bounds separately in each half, is for obtaining a small privacy bound at convergence. 
        } Via this combination, we obtain a new combined recursion for privacy bound as follows. For any $k=0, \cdots, \K - 1$,
        \begin{align}
            \frac{\eps_k^{j + 1}(\alpha)}{\alpha} \leq \begin{cases}
                \frac{\eps_{k}^{0}(\alpha) + \eps_0^{j + 1}(\alpha)}{\alpha}& \text{ if $j=0, \cdots, \lfloor \frac{n}{2b}\rfloor - 1$}\\
                \frac{\eps_{k}^{j}(\alpha')}{\alpha'}\cdot (1-\eta\lambda)^2 \cdot \frac{\sum_{s=0}^{k\cdot n/b + j - 1} (1-\eta\lambda)^{2s}}{\sum_{s=0}^{k\cdot n/b + j } (1-\eta\lambda)^{2s}} & \text{ if $j = \lfloor \frac{n}{2b}\rfloor, \cdots, n/b - 1$}\\
            \end{cases}
            \label{eqn:proof_strongly_convex_new_recursive}
        \end{align}
        where $\alpha'=(\alpha - 1)\cdot\left(1 + \frac{c\cdot 2 \eta \sig^2}{L^2}\right)^{-1} + 1$.

        We now solve this new recursion~\cref{eqn:proof_strongly_convex_new_recursive} for one epoch, by accumulating $j=0,\cdots, n/b - 1$. We prove that for any $k=0,\cdots,\K-1$,
        \begin{align}
            \frac{\eps_{k+1}^{0}(\alpha)}{\alpha} & = \frac{\eps_{k}^{n/b}(\alpha)}{\alpha} \\
            & \leq \frac{\eps_k^0(\tilde{\alpha}) + \eps_0^{\lfloor \frac{n}{2b}\rfloor}(\tilde{\alpha})}{\tilde{\alpha}} \cdot \prod_{j' = \lfloor \frac{n}{2b}\rfloor}^{n/b - 1} (1-\eta\lam)^2 \cdot \frac{\sum_{s=0}^{k\cdot n/b + j - 1} (1-\eta\lambda)^{2s}}{\sum_{s=0}^{k\cdot n/b + j } (1-\eta\lambda)^{2s}}\\
            & \leq \frac{\left(\eps_k^0(\tilde{\alpha}) + \eps_0^{\lfloor \frac{n}{2b}\rfloor}(\tilde{\alpha}) \right)}{\tilde{\alpha}} \cdot  (1-\eta\lam)^{2\cdot (n/b - \lfloor \frac{n}{2b}\rfloor)}\label{eqn:proof_strongly_convex_k_recursive}
        \end{align}

        where the RDP order $\tilde{\alpha}>1$ is the $n/b-\lfloor \frac{n}{2b}\rfloor$ fold mapped value of $\alpha$ under repeated mappings $\alpha\leftarrow(\alpha - 1)\cdot\left(1 + \frac{c\cdot 2 \eta \sig^2}{L^2}\right)^{-1} + 1$.

        We now further solve this new recursion~~\cref{eqn:proof_strongly_convex_new_recursive} for multiple epochs, by accumulating \cref{eqn:proof_strongly_convex_k_recursive} for $k = 0, 1,\cdots, \K-1$. We prove that for any $k=0, 1,\cdots, \K-1$, 
        \begin{align}
            \frac{\eps_{k+1}^{0}(\alpha)}{\alpha}& \leq \frac{\eps_0^0(\tilde{\alpha})}{\tilde{\alpha}} \cdot (1-\eta\lambda)^{2\cdot (k + 1) \cdot (n/b - \lfloor \frac{n}{2b}\rfloor)} + \frac{\eps_0^{\lfloor \frac{n}{2b}\rfloor}(\tilde{\alpha})}{\tilde{\alpha}} \cdot \sum_{k'=1}^{k + 1} (1-\eta\lam)^{2\cdot k' \cdot (n/b - \lfloor \frac{n}{2b}\rfloor)}
        \end{align}
        for some RDP order $\tilde{\alpha}>1$ that is the $(k+1)\cdot (n/b-\lfloor \frac{n}{2b}\rfloor)$ fold mapped value of $\alpha$ under repeated mapping $\alpha\leftarrow(\alpha - 1)\cdot\left(1 + \frac{c\cdot 2 \eta \sig^2}{L^2}\right)^{-1} + 1$.

        By further substituting $\eps_0^0(\tilde{\alpha}) = 0$ at the initialization point $\theta_0$ for any $\tilde{\alpha}>1$ in the above equation, we prove that for any $k = 0, \cdots, \K - 1$
        \begin{align}
            \frac{\eps_{k+1}^{0}(\alpha)}{\alpha} & = \frac{\eps_0^{\lfloor \frac{n}{2b}\rfloor}(\tilde{\alpha})}{\tilde{\alpha}} \cdot (1-\eta\lambda)^{2 \cdot (n/b - \lfloor \frac{n}{2b}\rfloor)} \cdot \frac{1 - (1-\eta\lambda)^{2\cdot (k + 1) \cdot (n/b - \lfloor \frac{n}{2b}\rfloor)}}{1 - (1-\eta\lambda)^{2\cdot (n/b - \lfloor \frac{n}{2b}\rfloor)} }
            \label{eqn:proof_strongly_convex_above}
        \end{align}
        for some RDP order $\tilde{\alpha}>1$.

        By setting $k = \K - 2$ in \cref{eqn:proof_strongly_convex_above}, we prove that for $\K \geq 2$
        \begin{align}
            \frac{\eps_{\K - 1}^0 (\alpha)}{\alpha} \leq & \frac{\eps_0^{\lfloor \frac{n}{2b}\rfloor}(\tilde{\alpha})}{\tilde{\alpha}} \cdot \frac{1 - (1-\eta\lambda)^{2\cdot (\K - 1) \cdot (n/b - \lfloor \frac{n}{2b}\rfloor)}}{1 - (1-\eta\lambda)^{2\cdot (n/b - \lfloor \frac{n}{2b}\rfloor)} }
            \label{eqn:proofstrconvexsmooth_K_minus_1_0_aux}
        \end{align}
        for some RDP order $\tilde{\alpha}>1$. By the format of $\eps_0^j(\alpha)$ for any $\alpha>1$ in our proof~\cref{eqn:proof_strongly_convex_first_epoch} (i.e., $\frac{\eps_0^j(\alpha)}{\alpha}$ is bounded by a constant for all $\alpha>1$), we prove that \cref{eqn:proofstrconvexsmooth_K_minus_1_0_aux} is equivalent to the following equation
        \begin{align}
            \eps_{\K - 1}^0 (\alpha) \leq & \eps_0^{\lfloor \frac{n}{2b}\rfloor}(\alpha) \cdot \frac{1 - (1-\eta\lambda)^{2\cdot (\K - 1) \cdot (n/b - \lfloor \frac{n}{2b}\rfloor)}}{1 - (1-\eta\lambda)^{2\cdot (n/b - \lfloor \frac{n}{2b}\rfloor)} }
            \label{eqn:proofstrconvexsmooth_K_minus_1_0}
        \end{align}

        For $\K = 1$, by definition, we compute that $\eps_{\K - 1}^0(\alpha) = 0 = $Right hand side of~\cref{eqn:proofstrconvexsmooth_K_minus_1_0}, therefore \cref{eqn:proofstrconvexsmooth_K_minus_1_0} also holds for $K=1$.
    \end{enumerate}
    
    We now plug our bound \cref{eqn:proofstrconvexsmooth_K_minus_1_0} into \cref{eqn:proofstrconvexsmooth_translate_j_0}, and prove a bound for $\eps$ as follows. 
    \begin{align}
        \eps \leq & \eps_{\K - 1}^0 (\alpha) + \eps_{0}^{n/b - j_0}(\alpha) \\
        \leq & \eps_0^{\lfloor \frac{n}{2b}\rfloor}(\alpha) \cdot \frac{1 - (1-\eta\lambda)^{2\cdot (\K - 1) \cdot (n/b - \lfloor \frac{n}{2b}\rfloor)}}{1 - (1-\eta\lambda)^{2\cdot (n/b - \lfloor \frac{n}{2b}\rfloor)} } + \eps_{0}^{n/b - j_0}(\alpha)
    \end{align}

    where the term $\eps_0^j(\alpha)$ is upper bounded by \cref{eqn:proof_strongly_convex_first_epoch}. This gives the privacy bound \cref{eqn:strconvexsmooth} in the theorem statement.
\end{proof}

\subsection{Explanations for the privacy bound derived from~\citet{balle2019privacy} in~\cref{fig:improved_dynamics}}

\label{append:figure1}

For convenience, we first translate~\cite[Theorem 5]{balle2019privacy} into the symbols used in this paper, as well as under mini-batch with size $b>1$, as follows.

\begin{theorem}[\citet{balle2019privacy}]
    Let $\ell(\theta,\x)$ be an $L$-Lipschitz, $\beta$-smooth, $\lambda$-strongly convex loss function. If $\eta\leq \frac{2}{\beta +\lambda}$, then conditioned on a fixed sequence of partitioned mini-batches $B^0,\cdots,B^{n/b-1}$ in~\cref{step:batch_generation_shuffle}, ~\cref{alg:noisymBGD} satisfies $(\alpha,\eps)$-R\'enyi DP for data points in the mini-batch $j_0$, with
    \begin{align}
        \eps = \begin{cases}
            \alpha\cdot \frac{2(L/b)^2}{2\eta\sig^2}& j_0=n/b-1\\
            \alpha\cdot \frac{2(L/b)^2}{(n/b-j_0-1)\cdot 2\eta\sig^2}\left(1-\frac{2\eta\beta\lambda}{\beta+\lambda}\right)^{\frac{n/b-j_0}{2}}& j_0=0,\cdots,n/b-2
        \end{cases} 
    \end{align}
\end{theorem}

The variables from~\citet{balle2019privacy} that we replaced are as follows: $L\rightarrow L/b$, this is because, in~\cref{alg:noisymBGD}, we additionally average each per-example gradient with the mini-batch size $b$; $i\rightarrow j_0+1$, this is because the index used in our paper starts from $j_0=0$, while the index used in~\citet{balle2019privacy} starts from $i=1$; $n\rightarrow n/b$, this is because under mini-batch size $b>1$, the number of iterations in one epoch in our paper is $n/b$, while the number of iteration in~\citet{balle2019privacy} equals the size of the datasets $n$; $\sig^2\rightarrow2\eta\sig^2$, this is because in our paper, we follow the noise scaling with variance $2\eta\sig^2$ that is related to stepsize $\eta$, as in SGLD, while~\citet{balle2019privacy} use the standard Gaussian noise with variance $\sig^2$.

By further substituting the Lipschitz requirement with its equivalent sensitivity assumption, i.e. $L\rightarrow \eta S_g/2$, and by using R\'enyi DP composition~\cite{mironov2017renyi} over the epochs, we obtain that the privacy bound in~\citet[Theorem 5]{balle2019privacy} is $\frac{\alpha \cdot \eta S_g^2}{4 \cdot (n/b - 1)\cdot b^2\sig^2}\cdot \left(1 - \frac{2\eta \beta\lambda}{\beta + \lambda}\right)^{\frac{n}{2b}}\cdot \K$ when the differing data point is in the first batch, and the bound is $\frac{\alpha \cdot \eta S_g^2}{4 \cdot (n/b - 1)\cdot b^2\sig^2}\cdot \left(1 - \frac{2\eta \beta\lambda}{\beta + \lambda}\right)^{\frac{n}{2b}}\cdot (\K - 1) + \frac{\alpha \cdot \eta S_g^2}{4  b^2\sig^2}$ when the differing data point is in the last batch. 

\subsection{Revisiting noisy GD: A tighter bound than~\cite[Corollary 1]{chourasia2021differential}}

\label{append:simple_proof_gd}

In this section, we prove a new converging hidden-state privacy bound for the noisy GD algorithm, which is a special case of Algorithm~\ref{alg:noisymBGD} under $b=n$. We then show that our bound is slightly tighter than~\cite[Corollary 1]{chourasia2021differential} and also admits a conceptually simpler proof (our proof relies on the privacy amplification by randomized post-processing results in~\cref{sec:fix_sgd}). We first state our privacy bound and its proof as follows.

\begin{theorem}
    \label{thm:GD_tighter}
    Let $\ell(\theta;\x)$ be a $\lambda$-strongly convex, and $\beta$-smoooth loss function, with a finite total gradient sensitivity $S_g$, then the noisy gradient descent algorithm with step-size $\eta<\frac{1}{\beta}$, satisfies $(\alpha,\varepsilon)$ R\'enyi Differential Privacy with
    \begin{align}
        \varepsilon \leq \frac{\alpha \eta S_g^2}{2\sigma^2n^2} \cdot  \sum_{k=1}^{\K}\left(1 - \frac{\eta\lambda}{2}\right)^{k}.
    \end{align}
\end{theorem}

\begin{proof}
    We first offer a new perspective of viewing a noisy GD update as follows. Recall that one noisy GD update in~\cref{alg:noisymBGD} is written as

\begin{align}
    \theta_{k+1}^0 = \theta_k^0 - \eta \cdot g(\theta_k^0;D) + \sqrt{2\eta\sigma^2}\cdot \mathcal{N}(0,\mathbb{I}_d)\ \text{ where }\ g(\theta_k^j;D) = \frac{1}{n}\sum_{i=1}^n\nabla\ell(\theta_k^j;\x_i)
\end{align}

Therefore, this update is equivalent to the following two steps:
\begin{align}
    \theta_{k}^{\frac{1}{2}} & = \theta_k^0 - \eta \cdot g(\theta_k^0;D) + \sqrt{\eta\sigma^2}\cdot \mathcal{N}(0,\mathbb{I}_d)\ \text{ where }\ g(\theta_k^0;D) = \frac{1}{n}\sum_{i=1}^n\nabla\ell(\theta_k^0;\x_i)\\
    \theta_{k+1}^0 & = \theta_k^{\frac{1}{2}} + \sqrt{\eta\sigma^2}\cdot \mathcal{N}(0,\mathbb{I}_d)
\end{align}

That is, we view each noisy GD update as another noisy GD update with smaller noise scale followed by pure additive Gaussian noise. 

Moreover, by~\cref{lem:LSIseq_stronglyconvex}, we prove that distributions of $\theta_k^0$ satisfies $c_k^0$-log Sobolev inequality with 

\begin{align}
    c_k^0= \frac{1}{2\eta\sigma^2}\cdot \frac{1}{\sum_{s=0}^{k-1}(1-\eta\lambda)^{2s}}
\end{align}

Therefore, by further using LSI under Lipschitz mapping (\cite[Lemma 16]{vempala2019rapid}) and under Gaussian convolution (\cite[Lemma 17]{vempala2019rapid}), we prove that the distribution of $\theta_k^{\frac{1}{2}}$ satisfies $c_k^{\frac{1}{2}}$-log Sobolev inequality with
\begin{align}
    \label{eqn:lsi_one_half}
    \frac{1}{c_k^{\frac{1}{2}}} & = \frac{(1-\eta \lambda)^2}{c_k^0} + \eta \sigma^2\\
    & = 2\eta \sigma^2 \cdot \sum_{s=1}^{k}(1-\eta\lambda)^{2s} + \eta \sigma^2
\end{align}

Therefore, by applying composition on the conceptual step from $\theta_k^0\rightarrow \theta_k^{\frac{1}{2}}$, we prove that
\begin{align}
    \label{eqn:step1}
    \frac{R_\alpha(\theta_k^{\frac{1}{2}}\lVert {\theta'}_k^{\frac{1}{2}})}{\alpha} \leq \frac{R_{\alpha}(\theta_k^0\lVert {\theta'}_k^{0})}{\alpha} + \frac{\eta S_g^2}{2\sigma^2n^2}
\end{align}

For the remaining conceptual step from $\theta_k^{\frac{1}{2}}\rightarrow\theta_{k+1}^0$, we note that pure additive Gaussian noise is equivalent to identity mapping convolved with Gaussian noise. Therefore, we apply $i_0\notin B_k^j$ case of~\cref{lem:recursive} with $L=1$ (for identity mapping) and noise standard deviation $\frac{\sigma}{\sqrt{2}}$, and prove that
\begin{align}
    \label{eqn:step2}
    \frac{R_\alpha(\theta_{k+1}^{0}\lVert {\theta'}_{k+1}^{0})}{\alpha} \leq \frac{R_{\alpha'}(\theta_k^{\frac{1}{2}}\lVert {\theta'}_k^{\frac{1}{2}})}{\alpha'} \cdot \left(1 + c_k^{\frac{1}{2}}\cdot \eta \sigma^2\right)^{-1}\ \text{ with }\ \alpha' = \frac{\alpha - 1}{1 + c_k^{\frac{1}{2}}\cdot \eta \sigma^2} + 1
\end{align}

Therefore, by combining~\eqref{eqn:step1} and \eqref{eqn:step2}, we prove the following recursive R\'enyi DP bound.

\begin{align}
    \label{eqn:GD_recursive_complex}
    \frac{R_\alpha(\theta_{k+1}^{0}\lVert {\theta'}_{k+1}^{0})}{\alpha} \leq \left(\frac{R_{\alpha'}(\theta_k^{0}\lVert {\theta'}_k^{0})}{\alpha'} + \frac{\eta S_g^2}{2\sigma^2n^2}\right) \cdot \left(1 + c_k^{\frac{1}{2}}\cdot \eta \sigma^2\right)^{-1}\ \text{ with }\ \alpha' = \frac{\alpha - 1}{1 + c_k^{\frac{1}{2}}\cdot \eta \sigma^2} + 1
\end{align}

By further plugging in the LSI constant sequence~\eqref{eqn:lsi_one_half}, we simplify \cref{eqn:GD_recursive_complex} into the following inequality.

\begin{align}
    \frac{R_\alpha(\theta_{k+1}^{0}\lVert {\theta'}_{k+1}^{0})}{\alpha} & \leq \left(\frac{R_{\alpha'}(\theta_k^{0}\lVert {\theta'}_k^{0})}{\alpha'} + \frac{\eta S_g^2}{2\sigma^2n^2}\right) \cdot \left(1 - \frac{1}{2\sum_{s=0}^{k}(1-\eta\lambda)^{2s}}\right)\\
    & \leq \left(\frac{R_{\alpha'}(\theta_k^{0}\lVert {\theta'}_k^{0})}{\alpha'} + \frac{\eta S_g^2}{2\sigma^2n^2}\right) \cdot \left(1 - \frac{1}{2\sum_{s=0}^{+\infty}(1-\eta\lambda)^{2s}}\right)\\
    & = \left(\frac{R_{\alpha'}(\theta_k^{0}\lVert {\theta'}_k^{0})}{\alpha'} + \frac{\eta S_g^2}{2\sigma^2n^2}\right) \cdot \left(1 - \frac{1-(1-\eta\lambda)^2}{2}\right)\\
    & \leq \left(\frac{R_{\alpha'}(\theta_k^{0}\lVert {\theta'}_k^{0})}{\alpha'} + \frac{\eta S_g^2}{2\sigma^2n^2}\right) \cdot \left(1 - \frac{\eta \lambda}{2}\right),\label{eqn:lam_eta_leq1}
\end{align}

where $\alpha' = (\alpha - 1)\cdot \left(1 - \frac{1}{2\sum_{s=0}^{k}(1-\eta\lambda)^{2s}}\right) + 1$, and the last inequality \cref{eqn:lam_eta_leq1} is because of the inequality $\eta \lambda<\frac{\lambda}{\beta}\leq 1$ that is ensured by the condition $\eta <\frac{1}{\beta}$.

Therefore, by solving the recursion~\cref{eqn:lam_eta_leq1} from $k=0,\cdots,\K-1$, and by using $R_\alpha(\theta_0^0\lVert {\theta'}_0^0) = 0$ for any $\alpha>1$, we prove the theorem statement in \cref{thm:GD_tighter}.
\begin{align}
    \label{eqn:GD_tighter}
    \frac{R_\alpha(\theta_{\K}^{0}\lVert {\theta'}_{\K}^{0})}{\alpha} & \leq \frac{\eta S_g^2}{2\sigma^2n^2} \cdot  \sum_{k=1}^{\K}\left(1 - \frac{\eta\lambda}{2}\right)^{k}
\end{align}
\end{proof}

\paragraph{Comparison with~\cite[Corollary 1]{chourasia2021differential}}

We now compare our privacy bound \cref{thm:GD_tighter} with~\cite[Corollary 1]{chourasia2021differential} and show that our bound is tighter. For convenience, we now repeat the privacy bound for noisy gradient descent (i.e. when $b=n$ in \cref{alg:noisymBGD}) that is proved in~\cite[Corollary 1]{chourasia2021differential} below.

\begin{theorem}[Corollary 1 in \citet{chourasia2021differential}]
    \label{thm:GD_baseline}
    Let $\ell(\theta;\x)$ be a $\lambda$-strongly convex, and $\beta$-smoooth loss function, with a finite total gradient sensitivity $S_g$, then the noisy gradient descent algorithm with start parameter $\theta_0\sim \mathcal{N}(0,\frac{2\sigma^2}{\lambda}\mathcal{I}_d)$, and step-size $\eta<\frac{1}{\beta}$, satisfies $(\alpha,\varepsilon)$ R\'enyi Differential Privacy with
    \begin{align}
        \varepsilon \leq \frac{\alpha S_g^2}{\lambda\sigma^2n^2}(1 - e^{-\lambda\eta K/2}).
    \end{align}
\end{theorem}

Because $1-x< e^{-x}$ for $x\neq 0$, we are able to further relax the bound in~\cref{thm:GD_tighter} as follows.

\begin{align}
    \varepsilon & \leq \frac{\alpha \eta S_g^2}{2\sigma^2n^2} \cdot  \sum_{k=1}^{\K}\left(1 - \frac{\eta\lambda}{2}\right)^{k}\\
    & < \frac{\alpha \eta S_g^2}{2\sigma^2n^2} \cdot \left(1 - \frac{\eta\lambda}{2}\right)\cdot  \sum_{k=0}^{\K - 1}e^{- \frac{\eta\lambda k}{2}}\\
    & = \frac{\alpha \eta S_g^2}{2\sigma^2n^2} \cdot \left(1 - \frac{\eta\lambda}{2}\right)\cdot \frac{1 - e^{-\lambda\eta K}}{1 - e^{-\eta \lambda/2}} \label{eqn:GD_tighter_relax_ex}
\end{align}

Because $1-e^{-x}> x\cdot e^{-x}$ for $x>0$, we prove that $1 - e^{-\eta \lambda/2} \geq \frac{\lambda\eta}{2}\cdot e^{-\frac{\lambda\eta}{2}}$. By plugging this inequality to \cref{eqn:GD_tighter_relax_ex}, we prove that
\begin{align}
    \varepsilon & < \frac{\alpha S_g^2}{\lambda\sigma^2n^2} \cdot \frac{1 - \frac{\eta\lambda}{2}}{e^{-\frac{\lambda\eta}{2}}}\cdot (1 - e^{-\lambda\eta K}) \label{eqn:GD_tighter_relax_eta}
\end{align}

By again using $1-x<e^{-x}$ for $x\neq 0$ into \cref{eqn:GD_tighter_relax_eta}, we prove that 

\begin{align}
    \varepsilon & < \frac{\alpha S_g^2}{\lambda\sigma^2n^2} \cdot (1 - e^{-\lambda\eta K}) = \text{RHS of \cite[Corollary 1]{chourasia2021differential}}
\end{align}

This shows that the privacy bound~\cref{thm:GD_tighter} enabled by our analysis is strictly tighter than~\cite[Corollary 1]{chourasia2021differential}.

\section{Proof for \cref{sec:sgd_subsampled}}

\label{append:rootconvexity}

\begin{replemma}{lem:rootconvexity1}
    Let $\mu_1,\cdots,\mu_m$ and $\nu_1, \cdots, \nu_m$ be measures over $R^d$. Then for any $\alpha\geq 1$, and any $p_1,\cdots,p_m\geq 0$ that satisfies $p_1+\cdots + p_m=1$,
    \begin{align}
        e^{(\alpha-1) \cdot R_\alpha(\sum_{j=1}^mp_j\mu_j\lVert \sum_{j=1}^mp_j\nu_j)} & \leq \sum_{j=1}^mp_j\cdot e^{(\alpha-1)\cdot R_\alpha(\mu_j\lVert \nu_j)}
    \end{align}
\end{replemma}

\begin{proof}
    By definition of R\'enyi divergence, 
    \begin{align}
        e^{(\alpha - 1)\cdot R_\alpha(\mu\lVert \nu)} = \int \left(\frac{\mu(\theta)}{\nu(\theta)}\right)^\alpha \cdot \mu(\theta)d\theta
    \end{align}
    Therefore by definition of $f$-divergence, $e^{(\alpha - 1)\cdot R_\alpha(\mu\lVert \nu)}$ is $f$-divergence with $f(x) = x^\alpha$. By the convexity of $f$ when $\alpha\geq 1$, and by applying Theorem 3.1 in \citet{taneja2004relative}, we prove that the f-divergence $e^{(\alpha - 1)\cdot R_\alpha(\mu\lVert \nu)}$ is jointly convex in arguments $\mu,\nu$.
\end{proof}

\subsection{Proof for \cref{thm:shuffle}}

\label{append:proofshuffle}

\begin{reptheorem}{thm:shuffle}[Privacy dynamics for "shuffle and partition" mini-batch gradient descent]
    If the loss function $\ell(\theta;x)$ is $\lambda$-strongly convex, $\beta$-smooth, and if its gradient has finite $\ell_2$-sensitivity $S_g$, then for $\K\geq 1$ and $\frac{n}{b}\geq 2$, \cref{alg:noisymBGD} with stepsize $\eta<\frac{2}{\lambda + \eta}$ satisfies $(\alpha,\eps)$-R\'enyi DP for all data points with
    \begin{align}
    \eps & \leq \eps_0^{\lfloor \frac{n}{2b}\rfloor}(\alpha) \cdot \frac{1 - (1-\eta\lambda)^{2\cdot (\K - 1) \cdot (n/b - \lfloor \frac{n}{2b}\rfloor)}}{1 - (1-\eta\lambda)^{2\cdot (n/b - \lfloor \frac{n}{2b}\rfloor)} } +  \frac{1}{\alpha-1}\cdot \log\left(\underset{0\leq j_0<n/b}{Avg} e^{(\alpha - 1)\eps_{0}^{n/b - j_0}(\alpha)}\right)
    \end{align}
    where the terms $\eps_0^{j}(\alpha)$ is upper-bounded for any $j = 1, \cdots, n/b$ as follows.
    \begin{align}
        \eps_0^{j}(\alpha) \leq \frac{\alpha \eta S_g^2}{4\sig^2b^2}\cdot \left(1 - \eta \lambda \right)^{2\cdot (j - 1)} \cdot \frac{1}{\sum_{s=0}^{ j - 1 } (1-\eta\lambda)^{2s}}
    \end{align}
\end{reptheorem}

\begin{proof}
    Our proof relies on the joint convexity of the scaled exponentiated R\'enyi divergence $e^{(\alpha-1)R_\alpha(\mu\lVert \nu)}$ in its arguments $\mu$ and $\nu)$. By further using the batch decomposition in~\cref{sec:overview}, we prove
    \begin{align}
        \eps & = \frac{1}{\alpha-1}\log e^{(\alpha-1)\cdot R_{\alpha}(\theta_K^0\lVert {\theta'}_K^0)}\\
        &\leq \frac{1}{\alpha-1}\log\left(\sum_{B^0,\cdots,B^{n/b-1}\text{in~\cref{step:batch_generation_shuffle}}} p(B^0,\cdots,B^{n/b-1})\cdot  e^{(\alpha-1)\cdot R_{\alpha}\left(p(\theta_K^0|{B^0,\cdots,B^{n/b-1}})\lVert p({\theta'}_K^0|{B^0,\cdots,B^{n/b-1}})\right)}\right)\label{eqn:proofshuffle_above}
    \end{align}

    Therefore, by plugging the privacy bound of \cref{alg:noisymBGD} proved in \cref{thm:strconvexsmooth} into~\cref{eqn:proofshuffle_above}, depending on which mini-batch contains the differing data point with any index $i_0$, we further prove that
    \begin{align}
        \eps & \leq \frac{1}{\alpha-1}\log\left(\sum_{j_0=0}^{n/b-1} p(i_0\in B^{j_0})\cdot  e^{(\alpha-1)\cdot \left(\eps_0^{\lfloor \frac{n}{2b}\rfloor}(\alpha) \cdot \frac{1 - (1-\eta\lambda)^{2\cdot (\K - 1) \cdot (n/b - \lfloor \frac{n}{2b}\rfloor)}}{1 - (1-\eta\lambda)^{2\cdot (n/b - \lfloor \frac{n}{2b}\rfloor)} } + \eps_{0}^{n/b - j_0}(\alpha)\right)}\right)\\
        & = \frac{1}{\alpha-1}\log\left(\sum_{j_0=0}^{n/b-1} \frac{b}{n}\cdot  e^{(\alpha-1)\cdot \left(\eps_0^{\lfloor \frac{n}{2b}\rfloor}(\alpha) \cdot \frac{1 - (1-\eta\lambda)^{2\cdot (\K - 1) \cdot (n/b - \lfloor \frac{n}{2b}\rfloor)}}{1 - (1-\eta\lambda)^{2\cdot (n/b - \lfloor \frac{n}{2b}\rfloor)} } + \eps_{0}^{n/b - j_0}(\alpha)\right)}\right)
    \end{align}

\end{proof}
\paragraphbe{Explanations for the terms in \cref{thm:shuffle}} The privacy bound \cref{thm:shuffle} is strictly smaller than the composition-based privacy bound (derived from DP-SGD\cite{abadi2016deep} and SGM~\cite{mironov2019r} analysis) after $\frac{1}{\lambda\eta} + \frac{n}{b}$ epochs. More specifically, the first term in \cref{eqn:shuffle_logistic} is strictly smaller than the composition-based privacy bound after $\frac{1}{\lambda \eta}$ epochs, and the second term in \cref{eqn:shuffle_logistic} is strictly smaller than the composition-based privacy bound after $n/b$ epochs (where $n/b\geq 2$). We explain the terms more specifically as follows.

\begin{enumerate}
    \item The sensitivity for one mini-batch gradient update in \cref{alg:noisymBGD} is $\eta \cdot \frac{S_g}{b} $, the standard deviation of Gaussian noise added in one update is $\sqrt{2\eta \sig^2}$, and the sampling probability for each data point is $\frac{b}{n}$. By \citet{abadi2016deep} and \citet{mironov2019r}, the composition-based privacy bound for one iteration is larger than $\frac{b^2}{n^2}\frac{\alpha \cdot \eta^2 \cdot S_g^2/b^2}{ 2 \cdot 2\eta \sig^2} = \frac{b^2}{n^2} \frac{\alpha \cdot \eta S_g^2}{ 4 \sig^2 b^2}$. By R\'enyi DP composition over $n/b$ iterations, the composition-based privacy bound for one epoch is larger than $ \frac{b}{n}\cdot\frac{\alpha \cdot \eta S_g^2}{ 4 \sig^2 b^2}$.
    \item The first term is upper bounded as follows, which is smaller than the composition-based privacy bound after $\frac{1}{\lambda\eta}$ epochs.
    
    \begin{align}
        &\eps_0^{\lfloor \frac{n}{2b}\rfloor}(\alpha) \cdot \frac{1 - (1-\eta\lambda)^{2\cdot (\K - 1) \cdot (n/b - \lfloor \frac{n}{2b}\rfloor)}}{1 - (1-\eta\lambda)^{2\cdot (n/b - \lfloor \frac{n}{2b}\rfloor)} } \\
        & \leq \frac{\alpha \eta S_g^2}{4\sig^2b^2}\cdot \left(1 - \eta \lambda \right)^{2\cdot (\lfloor \frac{n}{2b}\rfloor - 1)} \cdot \frac{1}{\sum_{s=0}^{ \lfloor \frac{n}{2b}\rfloor - 1 } (1-\eta\lambda)^{2s}} \cdot \frac{1 - (1-\eta\lambda)^{2\cdot (\K - 1) \cdot (n/b - \lfloor \frac{n}{2b}\rfloor)}}{1 - (1-\eta\lambda)^{2\cdot (n/b - \lfloor \frac{n}{2b}\rfloor)} }\\
        & = \frac{\alpha \eta S_g^2}{4\sig^2b^2}\cdot \left(1 - \eta \lambda \right)^{2\cdot (\lfloor \frac{n}{2b}\rfloor - 1)} \cdot \frac{1 - (1-\eta\lambda)^2}{1 - (1-\eta\lambda)^{2\cdot \lfloor \frac{n}{2b}\rfloor}} \cdot \frac{1 - (1-\eta\lambda)^{2\cdot (\K - 1) \cdot (n/b - \lfloor \frac{n}{2b}\rfloor)}}{1 - (1-\eta\lambda)^{2\cdot (n/b - \lfloor \frac{n}{2b}\rfloor)} }\\
        & \leq \frac{\alpha \eta S_g^2}{4\sig^2b^2}\cdot \frac{1}{1 - (1-\eta\lambda)^{2\cdot (n/b - \lfloor \frac{n}{2b}\rfloor)} } \quad \text{(by $(1-\eta\lambda)^x$ is monotonically decreasing for $x\in \mathbb{R}$)}\\
        & \leq \frac{\alpha \eta S_g^2}{4\sig^2b^2}\cdot \frac{1}{2\cdot (n/b - \lfloor \frac{n}{2b}\rfloor)\cdot\lambda \eta} \quad \text{(by $(1-x)^{a}\leq 1 - a x$ for $x>0$ and $a\geq 1$)}\\
        & \leq \frac{\alpha \eta S_g^2}{4\sig^2b^2}\cdot \frac{1}{n/b\cdot\lambda \eta} \quad \text{(by $\frac{n}{b}\geq 2$)}
    \end{align}

    \item The second term is upper bounded by $\eps_0^1(\alpha) = \frac{\alpha \eta S_g^2}{4\sig^2b^2}$, which is smaller than the composition-based privacy bound after $\cdot \frac{n}{b}$ epochs.
\end{enumerate}

\subsection{Proof for \cref{thm:amp_samp_wo_replacement}}

\begin{reptheorem}{thm:amp_samp_wo_replacement}
    If the loss function $\ell(\theta;x)$ is $\lambda$-strongly convex, $\beta$-smooth, and if its gradient has finite $\ell_2$-sensitivity $S_g$, then \cref{alg:noisymBGD} with stepsize $\eta<\frac{2}{\lambda + \beta}$ satisfies $(\alpha, \eps)$-R\'enyi DP guarantee with
    \begin{align}
        \eps \leq \frac{1}{\alpha - 1} \log\left(S_\K^0(\alpha)\right)
    \end{align}
    where the term $S_k^j(\alpha)$ is recursively defined by
    \begin{align}
        S_0^0(\alpha) &= 1\\
        S_k^{j+1}(\alpha) & = \frac{b}{n}\cdot e^{\frac{(\alpha - 1) \alpha \eta S_g^2}{4\sig^2 b^2}}\cdot S_k^j(\alpha) + (1-\frac{b}{n})\cdot S_k^j(\alpha)^{(1-\eta\lambda)^2}, \text{ for $k = 0,\cdots,\K-1$ and $j=0, \cdots, n/b - 1$}\\
        S_{k+1}^0(\alpha) & = S_k^{n/b}(\alpha) \text{ for $k = 0,\cdots,\K-1$}\\
    \end{align}
\end{reptheorem}

\begin{proof}
    By definition,

    \begin{align}
        e^{(\alpha - 1)\eps_k^{j+1}(\alpha)} = & \int \left(\frac{p_k^{j+1}(\theta)}{{p'}_k^{j+1}({\theta})}\right)^{\alpha}{p'}_k^{j+1}(\theta)d\theta\\
        = & \int \left(\frac{\mathbb{E}_{B_0^0,\cdots,B_k^j}\left[p_k^{j+1}(\theta|B_0^0,\cdots, B_k^{j})\right]}{\mathbb{E}_{B_0^0,\cdots,B_k^j}\left[{p'}_k^{j+1}({\theta}|B_0^0,\cdots,B_k^j)\right]}\right)^{\alpha} \cdot \mathbb{E}_{B_0^0,\cdots,B_k^j}\left[{p'}_k^{j+1}({\theta}|B_0^0,\cdots,B_k^j)\right]d\theta
    \end{align}

    By the joint convexity of the function $\frac{x^\alpha}{y^{\alpha - 1}}$ on $x, y>0$ (Lemma 20~\citet{balle2019privacy}), and by the mini-batch decomposition for the distribution of the last iterate parameters in~\cref{sec:overview}, we prove
    \begin{align}
        e^{(\alpha - 1)\eps_k^{j+1}(\alpha)} 
        \leq & \mathbb{E}_{B_0^0,\cdots,B_k^j}\left[\int \left(\frac{p_k^{j+1}(\theta|B_0^0,\cdots, B_k^{j})}{{p'}_k^{j+1}({\theta}|B_0^0,\cdots,B_k^j)}\right)^{\alpha} \cdot {p'}_k^{j+1}({\theta}|B_0^0,\cdots,B_k^j)d\theta\right]
    \end{align}

    We now derive a recursive scheme for computing the right hand side term denoted as $S_k^{j+1}(\alpha)$. By definition,

    \begin{align}
        S_k^{j+1}(\alpha) & = \mathbb{E}_{B_0^0,\cdots,B_k^j}\left[\int \left(\frac{p_k^{j+1}(\theta|B_0^0,\cdots, B_k^{j})}{{p'}_k^{j+1}({\theta}|B_0^0,\cdots,B_k^j)}\right)^{\alpha} \cdot {p'}_k^{j+1}({\theta}|B_0^0,\cdots,B_k^j)d\theta\right]  
    \end{align}

    The term inside expectation corresponds to R\'enyi privacy loss under fixed mini-batches $B_0^0, \cdots, B_k^j$. Therefore, by \cref{lem:recursive}, 

    \begin{align}
        \int & \left(\frac{p_k^{j+1}(\theta|B_0^0,\cdots, B_k^{j})}{{p'}_k^{j+1}({\theta}|B_0^0,\cdots,B_k^j)}\right)^{\alpha} \cdot {p'}_k^{j+1}({\theta}|B_0^0,\cdots,B_k^j)d\theta \\
        & \leq \begin{cases}
            \int \left(\frac{p_k^{j}(\theta|B_0^0,\cdots, B_k^{j})}{{p'}_k^{j}({\theta}|B_0^0,\cdots,B_k^j)}\right)^{\alpha} \cdot {p'}_k^{j}({\theta}|B_0^0,\cdots,B_k^j)d\theta \cdot e^{ \frac{(\alpha - 1) \alpha \eta S_g^2}{4\sig^2 b^2}} & \text{ if $i_0\in B_k^j$}\\
            \left(\int \left(\frac{p_k^{j}(\theta|B_0^0,\cdots, B_k^{j})}{{p'}_k^{j}({\theta}|B_0^0,\cdots,B_k^j)}\right)^{\alpha} \cdot {p'}_k^{j}({\theta}|B_0^0,\cdots,B_k^j)d\theta \right)^{(1-\eta\lambda)^2} & \text{ if $i_0\notin B_k^j$}\\
        \end{cases} \label{eqn:proof_rec_samp_wo}
    \end{align}

    By plugging \cref{eqn:proof_rec_samp_wo}, and by the definition of $S_k^{j+1}(\alpha)$, we prove
    \begin{align}
        S_k^{j+1}(\alpha) & = \mathbb{E}_{B_0^0,\cdots,B_k^j}\left[\int \left(\frac{p_k^{j+1}(\theta|B_0^0,\cdots, B_k^{j})}{{p'}_k^{j+1}({\theta}|B_0^0,\cdots,B_k^j)}\right)^{\alpha} \cdot {p'}_k^{j+1}({\theta}|B_0^0,\cdots,B_k^j)d\theta\right] \\
        & = P(i_0\in B_k^j) \cdot S_k^j\cdot e^{\frac{(\alpha - 1) \alpha \eta S_g^2}{4\sig^2 b^2}} + P(i_0\notin B_k) \cdot \\
        & \quad  \mathbb{E}_{B_0^0,\cdots,B_k^{j-1}}\left[ \left(\int \left(\frac{p_k^{j}(\theta|B_0^0,\cdots, B_k^{j})}{{p'}_k^{j}({\theta}|B_0^0,\cdots,B_k^j)}\right)^{\alpha} \cdot {p'}_k^{j}({\theta}|B_0^0,\cdots,B_k^j)d\theta \right)^{(1-\eta\lambda)^2} \right]
    \end{align}

    By further using the concavity of $x^{(1-\eta\lambda)^2}$, and the definition of $S_k^{j}(\alpha)$, we prove

    \begin{align}
        S_k^{j+1}(\alpha) & \leq P(i_0\in B_k^j) \cdot S_k^j\cdot e^{\frac{(\alpha - 1) \alpha \eta S_g^2}{4\sig^2 b^2}} + 
        P(i_0\notin B_k) \cdot \\
        & \quad \left(
            \mathbb{E}_{B_0^0,\cdots,B_k^{j-1}}\left[
                \int \left(\frac{p_k^{j}(\theta|B_0^0,\cdots, B_k^{j})}{{p'}_k^{j}({\theta}|B_0^0,\cdots,B_k^j)}\right)^{\alpha} 
                \cdot {p'}_k^{j}({\theta}|B_0^0,\cdots,B_k^j)d\theta 
                \right]
            \right)^{(1-\eta\lambda)^2}\\
            & = \frac{b}{n}\cdot e^{\frac{(\alpha - 1) \alpha \eta S_g^2}{4\sig^2 b^2}}\cdot S_k^j + \left(1-\frac{b}{n}\right)\cdot (S_k^j)^{(1-\eta\lambda)^2}
    \end{align}

    Therefore, the full recursive scheme for computing $S_k^j(\alpha)$ is as follows.

    \begin{enumerate}
        \item $S_0^0(\alpha) = 1$
        \item $S_k^{j+1}(\alpha) = \frac{b}{n}\cdot e^{\frac{(\alpha - 1) \alpha \eta S_g^2}{4\sig^2 b^2}}\cdot S_k^j(\alpha) + (1-\frac{b}{n})\cdot S_k^j(\alpha)^{(1-\eta\lambda)^2}$, for $k = 0,\cdots,\K-1$ and $j=0,\cdots,n/b-1$.
        \item $S_{k+1}^0(\alpha) = S_k^{n/b}(\alpha)$ for $k = 0, \cdots, \K - 1$.
    \end{enumerate}

    And \cref{alg:noisymBGD} satisfies $(\alpha, \eps)$-R\'enyi DP guarantee with
    \begin{align}
        \eps \leq \frac{1}{\alpha - 1} \log\left(S_\K^0(\alpha)\right)
    \end{align}
\end{proof}
\section{Proofs and Explanations for \cref{sec:trade-off}}

\label{append:experiment}

\subsection{Pseudocode for DP-SGD under notations in this paper}

\label{app:pseudocode}

In \cref{alg:noisymBGD_logistic}, we provide the pseudocode of an equivalent of DP-SGD algorithm under notations in our paper. For the convenience of implementation in existing privacy libaries, we introduce the noise multiplier $\sig_{mul}$ to substitute $\sig$ in \cref{alg:noisymBGD}. We comment that the noisy gradient update in \cref{alg:noisymBGD_logistic} under noise multiplier $\sig_{mul}$, is equivalent to a noisy gradient update in \cref{alg:noisymBGD} with noise standard deviation $\sigma = \sqrt{\frac{\eta}{2}}\cdot\frac{1}{b}\cdot \sig_{mul}\cdot \frac{S_g}{2}$. By plugging this equivalent $\sigma$ into \cref{thm:shuffle}, we prove the privacy dynamics bound in \cref{cor:logistic} for regularized logistic regression.

\begin{algorithm}[t!]
	\caption{$\algo_{\text{implementation}}$: Noisy mini-batch Gradient Descent on regularized logistic regression loss function}
	\label{alg:noisymBGD_logistic}
	\begin{algorithmic}
		\State {\bfseries Input:} Data domain $\X$. Dataset $\D = ((\x_1, \y_1), (\x_2, \y_2), \cdots, (\x_\size, \y_\size))$, where each data point consists of the feature vector $\x_i\in \mathbb{R}^d$ and the label vector $\y_i\in \{0,1\}^c$. The logistic regression loss function $\loss_0(\vtheta;\x, \y)$ defined as \cref{eqn:loss_logistic} with parameter space $\mathbb{R}^{(d+1)\cdot c}$. Stepsize $\step$, noise multiplier $\sig_{mul}$, a (data-independent) parameter initialization distribution $p_0(\theta)$, mini-batch size $b$, feature clipping norm $L$, and (unregularized) gradient clipping norm $\frac{S_g}{2}$.
		\State{\bfseries Feature Normalization:} $\x_1,\cdots,\x_n\leftarrow \text{normalize}(\x_1,\cdots,\x_n)$, where $\text{normalize}()$ is an $(\alpha, \eps_{norm})$-R\'enyi differentially private batch normalization or group normalization scheme described in \citet{tramer2020differentially} (Section 2.3 and Appendix B).
        \State{\bfseries Feature Clipping:} $\x_i\leftarrow \frac{\x_i}{\lVert \x_i\rVert_2}\cdot \min\{\lVert\x_i\rVert_2, L\}$.
        \State{\bfseries Initialization:} Sample $\theta_{0}^{0}$ from the initialization distribution $p_0(\theta)$.
		\State{\bfseries Batch Generation:} shuffle the indices set $\{1,\cdots,n\}$, and partition them into $n/b$ sequential mini-batches $B^0, \cdots, B^{n/b-1}$ that are subsets of $\{1,\cdots,n\}$, each with size $b$.\label{step:batch_shuffle}
		\For {$\kk = 0, 1, \cdots, \K - 1$}
		\For {$j = 0, 1, \cdots, n/b - 1$}
        \State { {\bfseries Gradient Clipping:} $g_0(\vtheta_{\kk}^{j}, \B^{j})= \frac{1}{b}\sum_{x_\ii\in \B^{j}} \frac{\grad\loss(\vtheta_{\kk}^{j};\x_\ii)}{\lVert \grad\loss_0(\vtheta_{\kk}^{j};\x_\ii)\rVert_2}\cdot \min\{\lVert \grad\loss_0(\vtheta_{\kk}^{j};\x_\ii)\rVert_2, \frac{S_g}{2}\}$}
		\State { {\bfseries Regularization:} $\g{\vtheta_{\kk}^{j}}{\B^{j}}= g_0(\theta_k^j;B^j) + \lambda \cdot \theta_k^j$}
		\State {$\vtheta_{\kk}^{j + 1} = \vtheta_{\kk}^{j} - \step\cdot \g{\vtheta_{\kk}^{j}}{\B^j} + \eta \cdot \frac{1}{b}\cdot \sig_{mul}\cdot \frac{S_g}{2}\cdot \Gauss{0}{\Id}$} \label{alg:ngd:updatestep}
		\EndFor
		\State {$\vtheta_{\kk + 1}^{0} = \vtheta_{\kk}^{n/b}$}
		\EndFor
		\State {Output $\vtheta_{\K}^{0}$}
	\end{algorithmic}
\end{algorithm}

\subsection{Proof for ensuring strong convexity}

\label{app:app_proof_str_cvx}

\paragraphbe{Regularized Logistic regression (for strong convexity)} The loss function for regularized logistic regression in the multi-class setting (with per-class bias) is as follows.
\begin{align}
    \label{eqn:app_reg_loss_logistic}
    \ell_\lambda(\theta;\x,\y)&= \ell_0(\theta;\x,\y) + \frac{\lambda}{2}\lVert \theta \rVert_2^2
\end{align}
where $\ell_0(\theta;\x,\y)$ is the following logistic regression loss function.
\begin{align}
    \ell_0(\theta;\x,\y) = - \y^1 \log\left(\frac{e^{\bar{x}^T\cdot \theta_1}}{e^{\bar{x}^T\cdot \theta_1} + \cdots + e^{\bar{x}^T\cdot \theta_c}}\right) - \cdots - \y^c \log\left(\frac{e^{\bar{x}^T\cdot \theta_c}}{e^{\bar{x}^T\cdot \theta_1} + \cdots + e^{\bar{x}^T\cdot \theta_c}}\right)
    \label{eqn:app_loss_logistic}
\end{align}
where $\bar{\x} = (\x, 1) \in \mathbb{R}^{d+1}$ denotes the concatenation of the data feature vector $\x$ and $1$, and $\y = (\y^1,\cdots,\y^c)$ is the label vector. The parameter vector is $\theta = (\theta_1,\cdots,\theta_c)\in \mathbb{R}^{(d+1)\cdot c}$ that represents the weight and the per-class bias of the linear model. The logistic regression loss function is convex, and therefore the regularized logistic regression loss function is $\lambda$-strongly convex.

\subsection{Proof for ensuring smoothness}
\label{app:proof_smooth}
\paragraphbe{Feature Normalization (for faster convergence)} For a better convergence of the learning task, we follow \citet{tramer2020differentially}, and use feature normalization (including both batch normalization and group normalization). Batch normalization first computes the per-channel mean and variance of the training dataset, in a differentially private way, and then normalizes each channel of each data point in the training dataset. Group normalization~\cite{wu2018group} separates the channels in a data feature vector into a number of different groups, and normalizes each channel of each data point with the per-point per-group mean and variance (which does not incur additional privacy cost). We refer the reader to~\cite{tramer2020differentially} (their Section 2.2 and Appendix B) for more details.

\paragraphbe{Feature Clipping (for bounding the smoothness constant)} To ensure that the condition of finite gradient sensitivity $S_g$ is satisified in our experiments, we follow \citet{feldman2018privacy} and normalize the data feature vector in $\ell_2$ norm, such that $\lVert\x\rVert_2\leq \Lip$. Under this data feature clipping, we prove that the logistic regression loss function \eqref{eqn:loss_logistic} is $(\frac{\Lip^2+1}{2})$-smooth in the following \cref{prop:smooth}.
\begin{proposition}
    \label{app_prop:smooth}
    If the data feature vector $\x$ has bounded $\ell_2$ norm, such that $\lVert \x\rVert_2\leq L$, then the unregularized logistic regression loss function $\ell_0(\theta;\x,\y)$ \cref{eqn:loss_logistic} is convex
    , $L$-Lipschitz 
    and $\beta$-smooth with regard to parameters $\theta$, for
    \begin{align}
        L & = \sqrt{2 (L^2 + 1)} \label{eqn:app_sensitivity}\\
        \beta & = \frac{L^2 + 1}{2} \label{eqn:app_smoothness}
    \end{align}
\end{proposition}

\begin{proof}
    We compute the gradient and Hessian matrix for the logistic regression loss function as follows.

    The gradient of the logistic regression loss function \cref{eqn:loss_logistic} with respect to $\theta$ is:
    \begin{align}
        \nabla_\theta\ell_0(\theta; \x,\y) & = - \y^{1} \cdot \nabla_\theta \log\left(\frac{e^{\bar{\x}^T\cdot \theta_{1}}}{e^{\bar{\x}^T\cdot \theta_1} + \cdots + e^{\bar{\x}^T\cdot \theta_c}}\right) - \cdots - \y^{c} \cdot \nabla_\theta \log\left(\frac{e^{\bar{\x}^T\cdot \theta_{c}}}{e^{\bar{\x}^T\cdot \theta_c} + \cdots + e^{\bar{\x}^T\cdot \theta_c}}\right) \\
        & =  \left( \left(\frac{e^{\bar{\x}^T\cdot \theta_1}}{e^{\bar{\x}^T\cdot \theta_1} + \cdots + e^{\bar{\x}^T\cdot \theta_c}}-\y^1\right) \bar{\x}^T, \cdots, \left(\frac{e^{\bar{\x}^T\cdot \theta_1}}{e^{\bar{\x}^T\cdot \theta_c} + \cdots + e^{\bar{\x}^T\cdot \theta_c}}- \y^c\right) \bar{\x}^T\right) \\
    \end{align}

    The Hessian of loss function with respect to $\theta$ is:
    \begin{align}
        \nabla^2_\theta\ell_0(\theta; \x,\y)& =
        \begin{pmatrix}
            H_{11}&\cdots&H_{1c}\\
            \vdots&\vdots&\vdots\\
            H_{c1}&\cdots&H_{cc}
        \end{pmatrix}
    \end{align}

    where the submatrices are as follows.
    \begin{align}
        \text{For $i=j$,}\quad H_{ii} & = \frac{e^{\bar{\x}^T\cdot \theta_i}}{e^{\bar{\x}^T\cdot \theta_1} + \cdots + e^{\bar{\x}^T\cdot \theta_c}}\cdot\left(1 - \frac{e^{\bar{\x}^T\cdot \theta_i}}{e^{\bar{\x}^T\cdot \theta_1} + \cdots + e^{\bar{\x}^T\cdot \theta_c}}\right)\bar{\x}\cdot \bar{\x}^T \\
        \text{For $i\neq j$,}\quad H_{ij} & = - \frac{e^{\bar{\x}^T\cdot \theta_i}}{e^{\bar{\x}^T\cdot \theta_1} + \cdots + e^{\bar{\x}^T\cdot \theta_c}}\cdot \frac{e^{\bar{\x}^T\cdot \theta_j}}{e^{\bar{\x}^T\cdot \theta_1} + \cdots + e^{\bar{\x}^T\cdot \theta_c}} \cdot \bar{\x}\cdot \bar{\x}^T \\
    \end{align}

    Therefore, the Hessian matrix of the loss function with respect to $\theta$ equals the following Kronecker product.

    \begin{align}
        \nabla^2_\theta\ell_0(\theta; \x,\y)& =
        T \otimes (\bar{\x}\cdot \bar{\x}^T)\\
        \text{where }T_{ij} &= \begin{cases}
            \frac{e^{\bar{\x}^T\cdot \theta_i}}{e^{\bar{\x}^T\cdot \theta_1} + \cdots + e^{\bar{\x}^T\cdot \theta_c}}\cdot\left(1 - \frac{e^{\bar{\x}^T\cdot \theta_i}}{e^{\bar{\x}^T\cdot \theta_1} + \cdots + e^{\bar{\x}^T\cdot \theta_c}}\right) & i=j\\
            - \frac{e^{\bar{\x}^T\cdot \theta_i}}{e^{\bar{\x}^T\cdot \theta_1} + \cdots + e^{\bar{\x}^T\cdot \theta_c}}\cdot\frac{e^{\bar{\x}^T\cdot \theta_j}}{e^{\bar{\x}^T\cdot \theta_1} + \cdots + e^{\bar{\x}^T\cdot \theta_c}}&i\neq j\\
        \end{cases}
    \end{align}

    By \citet{bohning1992multinomial}, $T$ is a positively semi-definite matrix that satisfies
    \begin{align}
        T \preceq \frac{1}{2}\cdot \left(I_{c} - \frac{1}{c}\cdot \mathbf{1}_c \cdot \mathbf{1}_c^T\right)\text{, where }\mathbf{1}_c = (1,\cdots,1)^T\in \mathbb{R}^{c}.
    \end{align}

    Therefore, the eigenvalues of $T$ fall in the range $[0, \frac{1}{2}]$. Therefore, because the eigenvalues for the Kronecker product matrix $\nabla_\theta^2\ell_0(\theta;\x,\y) = T\otimes \left(\bar{\x}\cdot \bar{\x}^T\right)$ consists of the product of eigenvalues for $T$ and $\bar{\x}\cdot \bar{\x}^T$, we prove that the eigenvalues for $\nabla_\theta^2\ell_0(\theta;\x,\y)$ fall in the range of $[0, \frac{1}{2}\lVert \bar{\x}\rVert_2^2 ]$.

    By $\lVert\x\rVert_2\leq L$, we prove that $\lVert \bar{\x}\rVert_2\leq \sqrt{L^2+1}$, therefore \cref{eqn:sensitivity} and \cref{eqn:smoothness} (in the proposition statement) hold. 
\end{proof}

These computations of these Lipschitz smoothness constants could also be cross-checked in the tensorflow privacy tutorial for logistic regression. \url{https://github.com/tensorflow/privacy/blob/master/tutorials/mnist_lr_tutorial.py}. The feature clipping technique is different from the DP-SGD~\cite{abadi2016deep} algorithm that only requires per-example gradient clipping. The major reason that we use data feature clipping (besides per-example gradient clipping), is for ensuring smoothness of the logistic regression loss function (by \cref{prop:smooth}), which is a necessary condition for applying our privacy bound \cref{thm:strconvexsmooth}.

\subsection{Proof for ensuring finite gradient sensitivity}

\label{app:proof_gradient_sensitivity}

\paragraphbe{Per-example Clipping on Unregularized Gradient (for reducing gradient sensitivity without harming smoothness or strong convexity)} Although feature clipping already bounds the gradient sensitivity by $2\sqrt{2(L^2+1)}$ (by \cref{prop:smooth}), this bound grows with the feature clipping norm $L$. This in turn restricts the signal to noise ratio, and does not give good empirical privacy-utility trade-off in our experiments. Therefore, we additionally perform per-example $\ell_2$-clipping on the unregularized gradient (detailed pseudocode in \cref{app:pseudocode}). Under per-example clipping on unregularized gradient, we prove in the following \cref{prop:smooth_strcvx_clip}, that each gradient update in \textit{regularized logistic regression} has finite gradient sensitivity, and preserves strong convexity and smoothness.

\begin{proposition}
    \label{prop:app_smooth_strcvx_clip}
    Let $\ell_0(\theta;\x,\y)$ be the logistic regression loss function defined in~\cref{eqn:loss_logistic}. Let $g_0(\theta;\x, \y) = \frac{\nabla \ell_0(\theta;\x,\y)}{\lVert \nabla \ell_0(\theta;\x,\y)\rVert_2}\cdot \min\{\lVert \nabla \ell_0(\theta;\x,\y)\rVert_2, \frac{S_g}{2}\}$ be the clipped gradient of (unregularized) loss function $\ell_0(\theta;\x,\y)$, under $\ell_2$ clipping norm $\frac{S_g}{2}$. If $g(\theta;\x,\y) = g_0(\theta;\x,\y) + \lambda\theta$, and if the data vector $\x$ has bounded $\ell_2$ norm, such that $\lVert \x\rVert_2\leq L$, then $g(\theta;\x,y)$ has finite $\ell_2$-sensitivity $S_g$,  is continuous, and is almost everywhere differentiable with
    \begin{align}
        \label{eqn:app_prop_smooth_strcvx_clip}
        \lambda\cdot \mathbb{I}_{(d+1)\cdot c}\preceq \nabla_\theta g(\theta;\x,\y)\preceq (\beta + \lambda) \cdot \mathbb{I}_{(d+1)\cdot c}
    \end{align}
    for any $\theta,\theta'\in\mathbb{R}^{(d+1)\cdot c}$ and $\beta = \frac{L^2+1}{2}$.
\end{proposition}

We provide complete proof for this proposition below. This construction of clipped unregularized gradient facilitates us to enjoy the benefits of gradient clipping (such as for speeding up convergence~\cite{zhang2019gradient,chen2020understanding}) while satsifying the necessary smoothness and strong convexity conditions for applying our privacy dynamics bound.

\begin{proof}
    By definition \cref{eqn:loss_logistic} for the logistic regression loss function, $\nabla\ell_0(\theta;\x,\y)$ is twice continuously differentiable and convex. By \cref{prop:smooth}, 
    $\ell_{\lambda}(\theta;\x,\y)$ is $\beta$-smooth with $\beta = \frac{L^2+1}{2}$. Therefore, the Hessian matrix of $\ell_0(\theta;\x,\y)$ satisfies the following inequality.
    \begin{align}
        \label{eqn:l0smoothness}
        0\cdot \mathbb{I}_{(d+1)\cdot c}\preceq\nabla_\theta^2 \ell_0(\theta;\x,\y) \preceq \beta \cdot \mathbb{I}_{(d+1)\cdot c},
    \end{align}
    where $d$ is the dimension of the input data feature vector $\x$, and $c$ is the number of classes in the label vector $\y$. Moreover, the clipped (unregularized) gradient $g_0(\theta;\x,\y)$ (under $\ell_2$ clipping norm $\frac{S_g}{2}$) is continuous, and is almost everywhere differentiable as follows.
    \begin{align}
        &\nabla_\theta g_0(\theta;\x,\y) = \nonumber\\
        &\begin{cases}
            \nabla_\theta^2 \ell_0(\theta;\x,\y) & \text{if }\lVert\nabla_\theta\ell_0(\theta;\x,\y)\rVert_2<\frac{S_g}{2}\\
            \frac{S_g}{2}\cdot \frac{1}{\lVert\nabla_\theta\ell_0(\theta;\x,\y)\rVert_2}\cdot M \cdot \nabla_\theta^2\ell_0(\theta;\x,\y)& \text{if }\lVert\nabla_\theta\ell_0(\theta;\x,\y)\rVert_2>\frac{S_g}{2}
        \end{cases}
    \end{align}
    where $M = \mathbb{I}_{(d+1)\cdot c} - \frac{\nabla\ell_0(\theta;\x,\y)}{\lVert\nabla\ell_0(\theta;\x,\y)\rVert_2}\cdot \left(\frac{\nabla\ell_0(\theta;\x,\y)}{\lVert\nabla\ell_0(\theta;\x,\y)\rVert_2}\right)^T$ is a symmetric matrix. Because $\frac{\nabla \ell_0 (\theta;\x, \y)}{\lVert \nabla \ell_0(\theta;\x, \y) \rVert_2}$ is a unit vector, we prove that $M$ is positive semi-definite and satisfies $0\cdot \mathbb{I}_{(d+1)\cdot c}\preceq M\preceq \mathbb{I}_{(d+1)\cdot c}$.  Moreover, for the case where  $\lVert\nabla_\theta\ell_0(\theta;\x,\y)\rVert_2>\frac{S_g}{2}$, we have that $\frac{S_g}{2}\cdot \frac{1}{\lVert\nabla_\theta\ell_0(\theta;\x,\y)\rVert_2}\leq 1$. By combining this ineuality with property of $M$ and \cref{eqn:l0smoothness}, we prove that for any $\theta$ such that $\lVert\nabla_\theta\ell(\theta;\x,\y)\rVert_2\neq \frac{S_g}{2}$,
    \begin{align}
        0\cdot \mathbb{I}_{(d+1)\cdot c}\preceq \nabla_\theta g_0(\theta;\x,\y)\preceq \beta \cdot \mathbb{I}_{(d+1)\cdot c}
    \end{align}

    Therefore, by plugging this into the definition of $g(\theta;\x,\y) = g_0(\theta;\x,\y) + \lambda\theta$, we prove that $g(\theta;\x,\y)$ is continous, and is almost everywhere differentiable with
    \begin{align}
        \label{eqn:clip_strcvx_smooth}
        \lambda\cdot \mathbb{I}_{(d+1)\cdot c}\preceq \nabla_\theta g(\theta;\x,\y)\preceq (\beta + \lambda) \cdot \mathbb{I}_{(d+1)\cdot c}
    \end{align}
\end{proof}

\end{document}